\documentclass{article}
\usepackage[dvipsnames,table]{xcolor}
\usepackage{iclr2024_conference,times}
\usepackage{note}
\usepackage{fedsarsa}
\usepackage{subfiles}
\usepackage{minitoc}

\usepackage{hyperref}       
\usepackage{url}   

\newcommand{\hw}[1]{{\color{brown}#1}}

\newcommand{\am}[1]{{\color{red}#1}}

\title{Finite-Time Analysis of On-Policy Heterogeneous Federated Reinforcement Learning}

\author{%
  Chenyu Zhang\\
  Data Science Institute\\
  Columbia University\\
  New York, NY 10025, USA \\
  \texttt{cz2736@columbia.edu} \\
  \And
  Han Wang\\
  Department of Electrical Engineering \\
  Columbia University\\
  New York, NY 10025, USA \\
  \texttt{hw2786@columbia.edu} \\
  \And
  Aritra Mitra\\
  Department of Electrical and Computer Engineering \\
  NC State University\\
  Raleigh, NC 27695, USA \\
  \texttt{amitra2@ncsu.edu} \\
  \And
  James Anderson \\
  Department of Electrical Engineering \\
  Columbia University\\
  New York, NY 10025, USA \\
  \texttt{james.anderson@columbia.edu}
}
\iclrfinalcopy
\begin{document}
\maketitle

\begin{abstract}
  Federated reinforcement learning (FRL) has emerged as a promising paradigm for reducing the sample complexity of reinforcement learning tasks by exploiting information from different agents. However, when each agent interacts with a potentially different environment, little to nothing is known theoretically about the non-asymptotic performance of FRL algorithms. The lack of such results can be attributed to various technical challenges and their intricate interplay: Markovian sampling, linear function approximation, multiple local updates to save communication, heterogeneity in the reward functions and transition kernels of the agents' MDPs, and continuous state-action spaces. Moreover, in the on-policy setting, the behavior policies vary with time, further complicating the analysis. In response, we introduce \fedsarsa, a novel federated on-policy reinforcement learning scheme, equipped with linear function approximation, to address these challenges and provide a comprehensive finite-time error analysis. Notably, we establish that \fedsarsa converges to a policy that is near-optimal for all agents, with the extent of near-optimality proportional to the level of heterogeneity. Furthermore, we prove that \fedsarsa leverages agent collaboration to enable linear speedups as the number of agents increases, which holds for both fixed and adaptive step-size configurations.
\end{abstract}

\section{Introduction}

Federated reinforcement learning (FRL) \citep{qi2021FederatedReinforcement,nadiger2019Federatedreinforcement,zhuo2019Federateddeep}, a distributed learning framework that unites the principles of reinforcement learning (RL) \citep{sutton2018Reinforcementlearninga} and federated learning (FL) \citep{mcmahan2023CommunicationEfficientLearning}, is rapidly gaining prominence for its wide range of real-world applications, spanning areas such as
edge computing \citep{wang2019Inedgeai},
% autonomous driving \citep{liang2022Federatedtransfer},
robot autonomous navigation \citep{liu2019Lifelongfederated},
and Internet of Things \citep{lim2020Federatedreinforcement}. This paper poses an FRL problem, where
multiple agents independently explore their own environments and collaborate to find a near-optimal universal policy accounting for their differing environmental models.
FRL leverages the collaborative nature of FL to address the data efficiency and exploration challenges of RL.
Specifically, we expect \emph{linear speedups} in the convergence rate and increased overall exploration ability due to federated collaboration.
We use FRL in autonomous driving \citep{liang2022Federatedtransfer} as a simple example to demonstrate our motivations and associated theoretical challenges.
In this scenario, the objective is to determine a strategy (policy) that minimizes collision probability. In contrast to the single-agent setting, where a policy is found by letting one vehicle interact with its environment, the federating setting coordinates multiple vehicles to interact with their distinct environments---comprising different cities and traffic patterns. Despite their aligned objectives, the environmental heterogeneity will produce distinct optimal strategies for each vehicle. Our goal is to find a universal robust strategy that performs well across all environments.

Tailored for such tasks, we propose a novel algorithm, \fedsarsa, integrating SARSA, a classic on-policy temporal difference (TD) control algorithm \citep{rummery1994OnlineQlearning,singh1996Reinforcementlearning}, into a federated learning framework.
On one hand, we want to leverage the power of federated collaboration to collect more comprehensive information and expedite the learning process.
On the other hand, we want to utilize the robustness and adaptability of on-policy methods.
To elaborate, within off-policy methods, such as Q-learning, agents select their actions according to a fixed \emph{behavior} policy while seeking the optimal policy. In contrast, on-policy methods, such as SARSA, employ learned policies as behavior policies and constantly update them.
By doing so, on-policy methods tend to learn safer policies, as they collect feedback through interaction following learned policies, and are more robust to environmental changes compared to off-policy methods (see \citet[Chapter 6]{sutton2018Reinforcementlearninga}).
Additionally, when equipped with different \emph{policy improvement operators}, on-policy SARSA is more versatile and can learn a broader range of goals than off-policy Q-learning (see \cref{sec:alg-local,sec:apx-exp}). 
Formally analyzing our federated learning algorithm poses several multi-faceted challenges.  We outline the most significant below.

\iffalse{
In federated training for autonomous vehicles, which may operate in different cities or encounter varying traffic conditions, cars interact with different environments and receive distinct feedback. Despite their shared objective of optimizing routes and speed control, the environmental heterogeneity results in distinct optimal strategies for each vehicle. Furthermore, factors such as safety and the reliability of stable connections play crucial roles in these scenarios.}

\fi

% Despite numerous intuitive advantages and application scenarios, theoretical guarantees on FRL algorithms are still limited. We highlight two distinctive challenges in analyzing an FRL algorithm within our setup.
\begin{itemize}
  \item \emph{Time-varying behavior policies.}
      In off-policy FRL with Markovian sampling \citep{woo2023BlessingHeterogeneity,khodadadian2022FederatedReinforcement,wang2023FederatedTemporal},  agents' observations are not i.i.d.; they are generated from a \emph{time-homogeneous} ergodic Markov chain as agents follow a \emph{fixed} behavior policy.
      Such an ergodic Markov chain converges rapidly to a steady-state distribution, enabling off-policy methods to inherit the theoretical guarantees for i.i.d. and mean-path cases \citep{bhandari2018FiniteTime,wang2023FederatedTemporal}.
      In contrast, on-policy methods update agents' behavior policies dynamically, rendering their trajectories \emph{nonstationary}.
      Therefore, previous analyses for off-policy methods, whether involving Markovian sampling or not, do not apply to our setting.
      Specifically, it remains unknown if the trajectories generated by on-policy FRL methods converge, and if they do, how this nonstationarity affects the convergence performance.
  \item \emph{Environmental heterogeneity in on-policy planning.}
        In an FRL instance, it is impractical to assume that all agents share the same environment \citep{khodadadian2022FederatedReinforcement,woo2023BlessingHeterogeneity}.
        In a planning task, this heterogeneity results in agents having distinct optimal policies.
        Thus, to affirm the advantages of federated collaboration, it is crucial to precisely characterize the disparities in optimality.
    Only two FRL papers have considered heterogeneity: \citet{jin2022Federatedreinforcement} explored heterogeneity in transition dynamics without linear speedup, and \citet{wang2023FederatedTemporal} considered heterogeneity in a prediction task (policy evaluation). 
    Beyond these studies, other research has addressed heterogeneity primarily within the domains of control design~\citep{wang2023model} and system identification~\citep{wang2023fedsysid}.
    Unfortunately, neither the characterizations nor analyses of heterogeneity from the previous work apply to on-policy FRL.
    Specifically, heterogeneity in agents' optimal policies implies heterogeneity in the behavior policies, which could lead to drastically different local updates across agents,  negating the benefits of collaboration.
  \item \emph{Multiple local updates and client drift.}
      In the federated learning framework, agents communicate with a central server periodically to reduce communication cost, and conduct local updates between communication rounds.
      However, these local updates push agents to local solutions at the expense of the overall \emph{federated} performance, a phenomenon known as \emph{client drift} \citep{scaffold}.
        Uniquely within our setting, client drift and nonstationarity amplify each other.%: agents' distinct local experiences and periodic synchronization lead to increased nonstationarity, and the discrepancies between agents become more significant as they follow different behavior policies.
  \item \textit{Continuous state-action spaces and linear function approximation.}
        To better model real-world scenarios, we consider continuous state-action spaces and employ a linear approximation for the value function.
        Unfortunately, RL methods with linear function approximation (LFA) are known to exhibit less stable convergence when compared to tabular methods \citep{sutton2018Reinforcementlearninga,gordon1996ChatteringSARSA}.
        Besides, the parameters associated with value function approximation no longer maintain an implicit magnitude bound.
        This concern is particularly relevant in on-policy FRL, where the client drift and the bias from nonstationarity both scale with the parameter magnitude.
\end{itemize}

\renewcommand\theadfont{\bfseries}
\renewcommand{\check}{\CheckmarkBold}
\newcommand{\uncheck}{\XSolidBrush}
\begin{table}[ht]
    \caption{Comparison of finite-time analysis for value-based FRL methods. LSP and LFA represent linear speedup and linear function approximation under the Markovian sampling setting; Pred and Plan represent prediction (policy evaluation) and  planning (policy optimization) tasks, respectively.
    % Behavior policies guide the online Markovian sampling; we mark any work that does not consider Markovian sampling as N/A in the Behavior Policy column.
    } 
  \label{tb:comp}%
  \centering%
  \resizebox{\textwidth}{!}{%
      \begin{tabular}{lllllll}
        \thead{Work}                                  &\thead{Hetero-\\geneity} & \thead{LSP} & \thead{LFA} & \thead{Markovian                                                             \\Sampling}& \thead{Task} & \thead{Behavior \\ Policy} \\
        \midrule                                                                                                                                                
        % \citet{bhandari2018FiniteTime}                & $1$         & $1$              & \check   & \uncheck & Prediction \& Planning & Fixed    \\
        % \citet{zou2019Finitesampleanalysis}           & $1$         & $1$              & \check   & \uncheck & Prediction \& Planning & Adaptive \\
        % \citet{qu2020Finitetimeanalysis}              & $1$         & $1$              & \check   & \uncheck & Planning               & Fixed    \\
        \citet{doan2019FiniteTimeAnalysis} & \uncheck& \uncheck & \check & \uncheck & Pred & Fixed \\
        \citet{jin2022Federatedreinforcement}         & \check     & \uncheck            & \uncheck    & \uncheck         & Plan               & Fixed               \\
        \citet{khodadadian2022FederatedReinforcement} & \uncheck   & \check      & \check      & \check           & Pred \& Plan & Fixed               \\
        \citet{shen2023towards}                       & \uncheck   & \check\footnotemark{}    & \check      & \check         & Plan               & Adaptive               \\
        \citet{wang2023FederatedTemporal}             & \check     & \check      & \check      & \check           & Pred             & Fixed               \\
        \citet{woo2023BlessingHeterogeneity}          & \uncheck   & \check      & \uncheck    & \check         & Plan               & Fixed               \\
        \rowcolor{green!30}\bf Our work               & \check     & \check      & \check      & \check           & Pred \& Plan & Adaptive
      \end{tabular}
}
\end{table}

\footnotetext{Considered  i.i.d. and Markovian sampling, but only established  linear speedup result for the  i.i.d. case.}

Given these motivations and challenges, we ask
\begin{quote}
  \textit{Can an agent expedite the process of learning its own near-optimal policy by leveraging information from other agents with potentially different environments?}
\end{quote}

We provide a complete non-asymptotic analysis of \fedsarsa, resulting in the first positive answer to the above question. % and address all aforementioned challenges.
We situate our work with respect to prior work in \cref{tb:comp}. A summary of our contributions is provided below:
\begin{itemize}
  \item \textit{Heterogeneity in FRL optimal policies.} We formulate a practical FRL planning problem in which agents operate in heterogeneous environments, leading to heterogeneity in their optimal policies as agents pursue different goals. We provide an explicit bound on this heterogeneity in optimality, validating the benefits of collaboration (\cref{thm:fix-drift}).
        % facilitating discussion of the trade-off between collaboration and competition.
    \item \textit{Federated SARSA and its finite-sample complexity.} We introduce the \fedsarsa algorithm for the proposed FRL planning problem and establish a finite-time error bound achieving a state-of-the-art sample complexity (\cref{thm}). At the time of writing, \fedsarsa is the first provably sample-efficient on-policy algorithm for FRL problems.
  \item \textit{Convergence region characterization and linear speedups via collaboration.} We demonstrate that when a constant step-size is used, federated learning enables \fedsarsa to exponentially converge to a small region containing agents' optimal policies, whose radius tightens as the number of agents grows (\cref{cor:err-con}).
        For a linearly decaying step-size, the learning process enjoys linear speedups through federated collaboration: the finite-time error reduces as the number of agents increases (\cref{cor:err-dec}).
        We validate these findings via numerical simulations.% in \cref{sec:exp,sec:apx-exp}.
\end{itemize}

% \subsection*{Organization} This paper is structured as follows. In \cref{sec:rel}, we briefly review related work on FRL and SARSA. In \cref{sec:pre}, we set up the problem and introduce the basic concepts required for our algorithm and its analysis. We present our proposed algorithm, \fedsarsa, in \cref{sec:alg}, where we discuss its implementation and design choices. In \cref{sec:anlys}, we first give a theorem on the perturbation bounds on SARSA fixed points, which illustrates the near optimality of our algorithm. We then demonstrate the finite-time analysis and linear speedup of \fedsarsa. Finally, in \cref{sec:exp}, we evaluate our algorithm through multiple simulation experiments and analyze their numerical results.

\section{Related Work}%\label{sec:rel}

\paragraph{Federated reinforcement learning.}
A comprehensive review of FRL techniques and open problems was recently provided by \cite{qi2021FederatedReinforcement}. FRL planning algorithms can be broadly categorized into two groups: policy- and value-based methods. In the first category, \citet{jin2022Federatedreinforcement,xie2023fedkl} considered tabular methods but did not demonstrate any linear speedup.
    % such as \textsf{PAvg} in~\cite{fan2021fault}, \textsf{FedPG-BR} in~\cite{jin2022Federatedreinforcement}, and \textsf{FedKL} in~\cite{xie2023fedkl}. However, the analyses in~\cite{jin2022Federatedreinforcement,xie2023fedkl} were restricted to the tabular case and do not exhibit linear speedup. Specifically, \cite{xie2023fedkl} only presented asymptotic convergence results. 
    \citet{fan2021fault} considered homogeneous environments and showed a \emph{sublinear} speedup property. 
    In the second category, \citet{khodadadian2022FederatedReinforcement,woo2023BlessingHeterogeneity} investigated federated Q-learning and demonstrated linear speedup under Markovian sampling. However, these studies did not examine the impact of environmental heterogeneity, a pivotal aspect in FRL. To bridge this gap, \cite{wang2023FederatedTemporal} presented a finite time analysis of federated TD(0) that can handle environmental heterogeneity.
    To take advantage of both policy- and value-based methods, \cite{shen2023towards}
    analyzed distributed actor-critic algorithms, but only established the linear speedup result under i.i.d. sampling.
    \cref{tb:comp} summarizes the key features of these value-based methods, including our work.
    There are also some works developed for studying the distributed version of RL algorithms: \cite{doan2019FiniteTimeAnalysis} and \cite{liu2023distributed} provided a finite-time analysis of distributed variants of TD(0); however, their analysis is limited to the i.i.d sampling model. 

% Multi-agent reinforcement learning (MARL) is another popular distributed learning paradigm closely related to FRL \citep{doan2019FiniteTimeAnalysis,chen2021Communicationefficientpolicy,chen2022Samplecommunicationefficient}. In contrast to FRL, agents in MARL algorithms explore within a shared environment, with their actions jointly determining a global state.
% Despite having different settings, FRL and MARL have different focuses. For example, FRL emphasizes more on privacy and personalization, while MARL focuses on achieving a global goal.

\paragraph{SARSA with linear function approximation.}
Single-agent SARSA is an on-policy TD control algorithm proposed by \citet{rummery1994OnlineQlearning} and \citet{singh1996Reinforcementlearning}. 
% Unlike off-policy Q-learning, SARSA is an on-policy algorithm that can learn a broader class of policies, including soft (non-deterministic) policies.
To accommodate large or even continuous state-action spaces, \citet{rummery1994OnlineQlearning} proposed function approximation.
We refer to SARSA with and without LFA as linear SARSA and tabular SARSA respectively
% SARSA with linear function approximation is referred to as linear SARSA, and SARSA without function approximation as tabular SARSA.
The asymptotic convergence result of tabular SARSA was first demonstrated by \cite{singh2000Convergenceresults}.
However, linear SARSA may suffer from chattering behavior within a region \citep{gordon1996ChatteringSARSA,gordon2000Reinforcementlearning,bertsekas1996Neurodynamicprogramming}.
With a \textit{smooth} policy improvement strategy, \citet{perkins2002convergentform} and \citet{melo2008analysisreinforcement} established the asymptotic convergence guarantee for linear SARSA. Recently, the finite-time analysis for linear SARSA was provided by \citet{zou2019Finitesampleanalysis}.

\section{Preliminaries}%\label{sec:pre}

%In this section, we formalize our setup and state some mild assumptions needed in the analysis.

\subsection{Federated Learning}
Federated Learning (FL) is a distributed machine learning framework designed to train models using data from multiple clients while preserving privacy, reducing communication costs, and accommodating data heterogeneity. We adopt the server-client model with periodic aggregation, akin to well-known algorithms like \textsf{FedAvg} \citep{mcmahan2023CommunicationEfficientLearning} and \textsf{FedProx} \citep{sahu}. Agents (clients) perform multiple \emph{local updates} (iterations of a learning algorithm) between communication rounds with the central server. 
During a communication round, agents synchronize their local parameters with those aggregated by the server.
However, this procedure introduces \emph{client-drift} issues \citep{scaffold,charles2021convergence}, which can hinder the efficacy of federated training. This problem is particularly pronounced in our on-policy FRL setting, where client drift is amplified due to the interplay with other factors.

\subsection{Markov Decision Process and Environmental Heterogeneity} \label{sec:pre-mdp}

We consider $N$ agents that explore within the same state-action space but with potentially different environment models. Specifically, agent $i$'s environment model is characterized by a Markov decision process (MDP) denoted by $\mathcal{M}\i = \left(\S, \A, r\i, P\i, \gamma \right)$.
Here, $\mathcal{S}$ denotes the state space, $\mathcal{A}$ is the action space, $r\i: \S\times \A\to[0,R]$ is a bounded reward function, $\gamma \in (0,1)$ is the discount factor, and $P\i$ is the Markov transition kernel such that $P\i_a(s,s')$ is the probability of agent $i$'s transition from state $s$ to $s'$ following action $a$.
% Note that our algorithm and analysis work for any bounded reward; here, we normalize it to range $[0,1]$ for notation simplicity.
While all agents share the same state-action space, their reward functions and state transition kernels can differ.
%* trim
%TODO mv to apx
% We denote by $S\i_t,U\i_t$ the state and action random variables of agent $i$ at time step $t$; then, the kernel $P\i$ is a family of measurable functions such that $P\i_{a}(s,s') = \operatorname{Pr}(S\i_{t+1}=s'|S\i_t=s,U\i_t=a)$ for any $(a,s,s')\in \A\times \S \times \S$, which gives the probability of agent $i$'s state being $s'$ at time step $t+1$ given that it takes action $a$ at state $s$ at time step $t$.
% We consider a continuous state-action space, without requiring it to be finite, discrete, or even of finite measure.
% We denote by $S$ and $A$ the measures of the state space and action space, respectively, which can be infinite.
Agents select actions based on their \textit{policies}. A policy $\pi$ maps a state to a distribution over actions,  $\pi(a|s)$ denotes the probability of an agent taking action $a$ at state $s$.
% We drop the superscript $\i$ when we are not referring to a specific agent's policy, but a general one in the policy space.
% When a policy is fixed, an MDP reduces to a Markov reward process considered in prediction tasks, and the transition kernel is solely determined by the states.

%Before delving into the problem formulation, we first make an assumption regarding the MDPs we consider. and discuss how we measure the environmental heterogeneity.

% \begin{assumption} \label{asmp:bound-mdp}
%     The reward function is bounded, i.e., there exists a $R>0$ such that $r: \S\times \A\to [0,R]$.
%     Also, the action space has a finite measure, i.e., there exists $A > 0$ such that
%     \[
%     |\A| = \int_{\A}\d a = A < +\infty
%     \]
% \end{assumption}

\begin{assumption}[Uniform ergodicity] \label{asmp:steady}
	For each $i \in [N]$, the Markov chain induced by any policy $\pi$ and state transition kernel $P\i$ is ergodic with a uniform mixing rate.
	In other words, for any MDP $\mathcal{M}\i$ and candidate policy $\pi$, there exists a steady-state distribution $\eta\i_{\pi}$, as well as constants $m_i \ge 1$ and $\rho_i\in(0,1)$, such that
	% Furthermore, we assume there exists a uniform mixing rate for each transition kernel, i.e., there exist $m_i \ge 0$ and $\rho_i\in(0,1)$ such that for any $t\ge 0$, it holds that
	\[
		\sup_{s\in\S}~\sup_{\pi} \left\|P_{\pi}\left(S_t^{(i)}=\cdot \given S_0^{(i)}=s\right) -  \eta\i_{\pi}\right\|_{\mathrm{TV}} \leq m_i \rho_i^t,
	\]
    where $\|\cdot \|_{\mathrm{TV}}$ is the total variation distance.\footnotemark{}
	\footnotetext{We use the functional-analytic definition of the total variation, which is twice the quantity $\sup_{A\in \mathcal{F}}|p(A)|$ for any signed measure $p$ on $\mathcal{F}$.% that is often used by probabilists.
 }
\end{assumption}

\cref{asmp:steady} is a standard assumption in the RL literature needed to provide finite-time bounds under Markovian sampling~\citep{bhandari2018FiniteTime,zou2019Finitesampleanalysis,srikant2019Finitetimeerror}.
% Given these MDPs, we can define a virtual \textit{central} MDP $\bar{\mathcal{M}} = (\S,\A,\bar{r},\bar{P},\gamma)$, where $\bar{r} = \frac{1}{N}\sum_{i=1}^{N}r\i, \bar{P} = \frac{1}{N}\sum_{i=1}^{N}P\i$.
% By (Wang et al., 2023, Proposition 1), the virtual MDP is also ergodic under any policy~$\pi$.

% \begin{proposition}
% 	Let $\left\{P^{(i)}\right\}_{i=1}^N$ be a set of Markov transition kernels associated with Markov chains that share the same states, and are each aperiodic and irreducible with any policy $\pi$. Then, for any set of weights $\left\{w_i\right\}_{i=1}^N$ satisfying $w_i \geq 0$ for any $i \in[N]$ and $\sum_{i \in[N]} w_i=1$, the Markov chain corresponding to the kernel $\sum_{i \in[N]} w_i P^{(i)}$ is also aperiodic and irreducible with any policy $\pi$.
% \end{proposition}

% Therefore, we know the virtual MDP is also aperiodic and irreducible with any policy $\pi$.

Agents operate in their own environments and may have their own goals.  We collectively refer to the differences in the transition kernels and rewards as environmental heterogeneity.
Intuitively, collaboration among agents is advantageous when the heterogeneity is small, but can become counterproductive when the heterogeneity is large.
We now provide two natural definitions for measuring environmental heterogeneity. %These will be used to formally characterize the above intuition in \cref{sec:anlys}. \am{I think you mean some other Section here.}

\begin{definition}[Transition kernel heterogeneity] \label{asmp:ker-het}
	% There exists an $\epsilon_p > 0$ such that for any $(a,s,s')\in \A\times \S^{2}$ and any $i,j\in[N]$, we have
	% \[
	% \left|P\i_{a}(s,s') - P\j_{a}(s,s')\right| \le \epsilon_p P\i_{a}(s,s'),
	% \]
	% There exists an $\epsilon_p > 0$ such that for any $i,j\in[N]$, we have
	% \[
	% 	\left\| P\i - P\j \right\|_{\mathrm{TV}} \le \epsilon_p,
	% \]
	We capture the transition kernel heterogeneity using the total variation induced norm:
	\[
		\epsilon_{p} \coloneqq \max_{i,j\in[N]}\bigl\| P\i - P\j \bigr\|_{\mathrm{TV}},
	\]
	where with a slight abuse of notation, we define
	\[
		\|P\|_{\mathrm{TV}} \coloneqq \sup_{\substack{q\in \mathcal{P}(\mathcal{S}\times \mathcal{A})\\\|q\|_{\mathrm{TV}}=1}}\|qP\|_{\mathrm{TV}}
		= \sup_{\substack{q\in \mathcal{P}(\mathcal{S}\times \mathcal{A})\\\|q\|_{\mathrm{TV}}=1}}\left\| \int_{\mathcal{S}\times \mathcal{A}}q(s,a)P_{a}(s,\cdot )\d s\d a \right\| _{\mathrm{TV}},
	\]
	where $\mathcal{P}(\mathcal{S}\times \mathcal{A})$ is the set of probability measures on $\mathcal{S}\times \mathcal{A}$. By the triangle inequality and the uniform bound on rewards, $R$,  we have $\epsilon_{p}\le 2$.
\end{definition}

\begin{definition}[Reward heterogeneity] \label{asmp:r-het}
	% There exists an $\epsilon_r >0$ such that for any $i,j\in[N]$, we have
	% \[
	% 	\left\| r\i - r\j \right\|_\infty \le \epsilon_r,
	% \]
	We capture the reward heterogeneity using the infinity norm:
	\[
		\epsilon_{r} \coloneqq \max_{i,j\in[N]}\frac{\bigl\| r\i - r\j \bigr\|_\infty}{R},
	\]
	where $\|r\|_\infty = \sup_{s,a\in\S\times \A}|r(s,a)|$. By the triangle inequality, we have $\epsilon_{r}\le 2$.
\end{definition}

% We first note that, by definition, we have $\epsilon_{p} \le \max_{i,j\in[X]}( \| P\i \|_{\mathrm{TV}} + \| P\j \|_{\mathrm{TV}} ) \le 2$ and $\epsilon_{r} \le \max_{i,j\in[N]}( \| r\i \|_{\infty} + \| r\j \|_{\infty} ) \le 2R$.
% That is, both transition kernel heterogeneity and reward heterogeneity are bounded.
% Moreover, 
% \hw{Notably, our main results in the paper are agnostic to the level of heterogeneity.}

\subsection{Value Function and SARSA}

An RL planning task aims to maximize the expected \textit{return}, defined as the accumulated reward of a trajectory.
% A policy $\pi$ maps a state to a distribution of actions such that $\pi(a|s)$ is the probability of an agent taking action $a$ at state $s$ following policy $\pi$ for any state-action pair.
% For policy $\pi$, the expected return of state $s$ in model $\mathcal{M}\i$ 
For a given policy $\pi$, the expected return of a state-action pair~$(s,a)$
is captured by the Q-value function:
% \[
% 	% V\i_{\pi}(s) = \EE_{\pi}\i\left[ \sum_{t=0}^{\infty}\gamma^{t}r\i(S_t,A_t)|S_{0}=s \right],
% 	V_{\pi}(s) = \EE_{\pi}\left[ \sum_{t=0}^{\infty}\gamma^{t}r(S_t,A_t)|S_{0}=s \right],
% \]
\begin{equation}\label{eq:q}
	q_{\pi}(s,a) = \EE_{\pi} \left[\sum_{t=0}^{\infty}\gamma^{t}r(s_t,a_t)\Big|S_0 = s, A_0 = a\right]
	= \underbrace{r(s,a) + \gamma\EE_{\pi} \left[q_{\pi}(S_{1}, A_{1})|S_0 = s, A_0 = a\right]}_{T_\pi q_\pi (s,a) },
\end{equation}
where the expectation is taken with respect to a transition kernel that follows the policy~$\pi$ (except for the initial action, which is fixed to $a$).
For any MDP, there exists an optimal policy~$\pi_{*}$ such that $q_{\pi_{*}}(s,a) \ge q_{\pi}(s,a)$ for any other policy~$\pi$ and state-action pair~$(s,a)$.
\emph{This paper focuses on an FRL problem where all agents aim to find a universal policy that is near-optimal for all MDPs under a low-heterogeneity regime.} 
% \[
% 	% V_{\pi_{*}}\i(s) \ge V_{\pi}\i(s).
% 	% V_{\pi_{*}}(s) \ge V_{\pi}(s).
% 	\underset{A \sim \pi_{*}(s)}{\EE} q_{\pi_{*}}(s,A) \ge 	\underset{A \sim \pi(s)}{\EE} q_{\pi}(s,A).
% \]

To find such an optimal policy for a single agent, SARSA updates the estimated Q-value function based on \eqref{eq:q} by sampling and bootstrapping.
With the updated estimation of the value function, SARSA improves the policy via a policy improvement operator. By alternating policy evaluation and policy improvement, SARSA finds the optimal policy within the policy space.
The tabular SARSA for a single agent can be described by the following update rules:
\begin{equation}\label{eq:sarsa}
	\begin{cases}
		Q\left(s_t, a_t\right) & \leftarrow Q\left(s_t, a_t\right)+\alpha\left(r(s_t,a_t) +\gamma Q\left(s_{t+1}, a_{t+1}\right)-Q\left(s_t, a_t\right)\right), \\
		\pi(a_{t+1}|s_{t+1})   & \leftarrow \Gamma(Q(s_{t+1}, a_{t+1})),
	\end{cases}
\end{equation}
where $Q$ is the estimated Q-value function, $\alpha$ is the learning step-size, and $\Gamma$ is the policy improvement operator.
We provide further discussion on the policy improvement operator in \cref{sec:alg-local}.
% \fi

\subsection{Linear Function Approximation and Nonlinear Projected Bellman Equation}

When the state-action space is large or continuous, tabular methods are intractable.
Therefore, we employ a linear approximation for the Q-value function \citep{rummery1994OnlineQlearning}. For a given feature extractor $\phi: \S\times \A \to \R^{d}$, we approximate the Q-value function as $Q_{\theta}(s,a) = \phi(s,a)^T \theta$, where $\theta \in \Theta \subseteq \R^d$ is a parameter vector to be learned. Without loss of generality, we assume that $\|\phi(s,a)\|_{2} \le 1$ for every  state-action pair~$(s,a)$.
% Under an LFA, the goal is to find the parameter that minimizes the Bellman error, which is equivalent to solving the nonlinear projected Bellman equation:
Linear function approximation  translate the task of finding the optimal policy to that of identifying the optimal parameter $\theta$ that solves the nonlinear projected Bellman equation:
\begin{equation}\label{eq:bellman}
	Q_{\theta} = \Pi_{\pi} T_{\pi}Q_{\theta},
\end{equation}
where $T_{\pi}$ is the Bellman operator defined by the right-hand side of \eqref{eq:q}, and $\Pi_{\pi}$ is the orthogonal projection onto the
linear subspace spanned by the range of the $\phi$
% function family $\{ Q_{\theta} \}_{\theta\in\Theta}$
using the inner product $\left<x,y\right>_{\pi} = \EE_{S\sim \eta_{\pi}, A\sim \pi(S)}[x(S,A)^Ty(S,A)]$.
Equation~\cref{eq:bellman} reduces to the linear Bellman equation used in policy evaluation when the policy $\pi$ is fixed \citep{tsitsiklis1997analysistemporaldifference,bhandari2018FiniteTime}, and to the Bellman optimality equation used in Q-learning when the policy improvement operator is the greedy selector \citep{watkins1992Qlearning,melo2008analysisreinforcement}.

\section{Algorithm}%\label{sec:alg}

We now develop \fedsarsa; a federated version of linear SARSA.
% At its core, \fedsarsa is a simple algorithm combining FedAvg \citep{mcmahan2023CommunicationEfficientLearning} and linear SARSA.
In \fedsarsa, each agent explores its own environment and improves its policy using its observations, which we refer to as local updating. Periodically, agents send the parameter progress to the central server, where the parameters get aggregated and sent back to each agent.
We present \fedsarsa in Algorithm~\ref{alg}.

\paragraph{Local update.} \label{sec:alg-local}
% We reformulate \cref{eq:bellman} as that for any $\theta'\in\Theta$, it holds
% \[
% \begin{aligned}
% 0 =& \EE\left[ Q_{\theta'}(S,A)^T \left( T_{\pi}Q_{\theta}(S,A) - Q_{\theta}(S,A) \right) \right]\\
% =& \EE\left[ \theta'^T \phi(S,A)((r(S,A) + \gamma \phi(S',A')^T\theta) - \phi(S,A)^T\theta) \right]\\
% =& \theta' \EE\left[ \phi(S,A)r(S,A) + \phi(S,A)(\gamma\phi(S',A') - \phi(S,A))^T\theta \right].
% \end{aligned}
% \]
Locally, agent $i$ updates its parameter using the SARSA update rule. With linear function approximation, the Q-value function update in \eqref{eq:sarsa} becomes
\[
	\theta\itt = \theta\it + \alpha_{t}g\it\left(\theta\it; s\it,a\it\right),
\]
where $\alpha_t$ is the step-size%
\footnote{For ease of presentation, we assume all agents share the same step-size. Our analysis handles agents using their own step-size schedule, as long as each agent's step-size falls within the specified range.}
and $g\it$ is defined as
\begin{equation}\label{eq:g}
	g\it\left(\theta;s,a\right) = \phi(s,a) r\i(s,a) + \phi(s,a)\left(\gamma\phi(s',a')\!-\!\phi(s,a)\right)^T\!\theta,
	\quad s'\!\sim\!P\i_{a}(s,\cdot ),\ a'\!\sim\!\pi_{\theta\it}\!\left(\cdot |s'\right).
\end{equation}
We refer to $g_t$ as a \emph{semi-gradient} as it resembles a stochastic gradient but does not represent the true gradient of any static loss function \citep{barnard1993Temporaldifferencemethods}.
Also, we introduce a subscript $t$ to the semi-gradient to indicate that it depends on the policy $\pi_{\theta\it}$ at time step $t$.

% subsection Local Parameter Update (end)

% \subsubsection{Policy Improvement} \label{sec:alg-pi}

% For the policy improvement after the parameter update \cz{in \cref{eq:sarsa}},
\paragraph{Policy improvement.}
We assume that all agents use the same policy improvement operator $\Gamma$, which returns a policy $\pi$ for any Q-value function. Since we consider linearly approximated Q-value functions, we can view the policy improvement operator as acting on the parameter space: $\Gamma: \theta \mapsto \pi$. We denote the policy resulting from the parameter $\theta$ as $\pi_{\theta} = \Gamma(\theta)$.
To ensure the convergence of the algorithm, we need the following assumption on the policy improvement operator's smoothness.
% We need two mild assumptions on the policy space and the policy improvement operator.

% \begin{assumption}[Compact Policy Space] \label{asmp:compact-policy}
% 	The policy space $\Gamma$ is compact.
% \end{assumption}

% When the state-action space is finite, the set of $\epsilon$-greedy policies is finite. Therefore, this assumption is automatically satisfied. More generally, for continuous state-action space, we can consider parametrized policies $\pi(a|s,\mu)$, where $\mu$ is the parameters. Then the policy space is compact if $\pi$ is continuous in $(s,\mu)$ and the state space and the policy parameter space are compact.

\begin{assumption}[Lipschitz continuous policy improvement operator] \label{asmp:lip}
	The policy improvement operator is Lipschitz continuous in TV distance with constant $L$: 
	% \[
	% 	\sup_{a\in\A,s\in\S}|\pi_{\theta_{1}}(a|s) - \pi_{\theta_2}(a|s)| \le C \|\theta_1 - \theta_2\|.
	% \]
	% We also need the constant $C$ small enough to guarantee convergence. The bound on $C$ will be given in the analysis.
	% This directly gives
	% \[
	% 	\left\|\pi_{\theta_1}(\cdot |s) -  \pi_{\theta_2}(\cdot |s\right\|_{\mathrm{TV}}) \le \frac{1}{2}AC\|\theta_1 - \theta_2\|
	% \]
	\[
		\left\|\pi_{\theta_1}(\cdot |s) -  \pi_{\theta_2}(\cdot |s)\right\|_{\mathrm{TV}} \le L\|\theta_1 - \theta_2\|_{2},\quad \forall\theta_1,\theta_2\in\Theta, s\in\S.
	\]
	Furthermore, $L\le w/(H\sigma)$, where $H$, $\sigma$, and $w$ are problem constants to be defined in Appendix.
\end{assumption}
When the action space is of finite measure, Assumption~\ref{asmp:lip} is equivalent to that in \citet{zou2019Finitesampleanalysis}.
This assumption is standard for linear SARSA \citep{zou2019Finitesampleanalysis,perkins2002convergentform,melo2008analysisreinforcement}. As shown in \citep{defarias2000existencefixed,perkins2002existencefixed,zhang2022chatteringsarsa}, linear SARSA with noncontinuous policy improvement may diverge.

% An example policy improvement operator satisfying \cref{asmp:lip} for an action space with finite measure is to map the normalized estimated value of an action as its probability. Specifically, we define $\Gamma(\theta)$ as follows:
% \begin{equation}\label{eq:pi}
% \pi_{\theta}(a|s) \coloneqq \pi_0(a|s) + \frac{L\epsilon}{2A}\left( Q_{\theta}(a|s) - \frac{1}{A}\int_{a\in\A}Q_{\theta}(a|s)\d a \right), \quad\text{with } \epsilon \le 1,
% \end{equation}
% where $\pi_0$ is a soft policy, $\epsilon \le 1$ ensures that the resulting policy remains soft, and $A = \int_{a\in\A}\d a$ denotes the measure of $\A$. This policy improvement operator satisfies \cref{asmp:lip} because
% $$
% | \pi_{\theta_1}(a|s) - \pi_{\theta_2}(a|s) | = \frac{L\epsilon}{2A}\left| (\theta_1 - \theta_2)^T\left(\phi(s,a) - \frac{1}{A}\int_{a\in\A}\phi(s,a) \d a\right) \right| \le \frac{L}{A}\left\| \theta_1-\theta_2 \right\|_{2},
% $$
% and then, $\left\| \pi_{\theta_1}(\cdot |s) - \pi_{\theta_2}(\cdot |s) \right\|_{\mathrm{TV}} \le L\|\theta_1 - \theta_2\|$.

An example policy improvement operator satisfying \cref{asmp:lip} is the softmax function with suitable temperature parameter \citep{gao2018PropertiesSoftmax}.
In contrast, the deterministic greedy policy improvement employed in Q-learning is an illustrative case where \cref{asmp:lip} does not hold.
%* trim
% As a result, if the optimal policy set only contains deterministic policies, linear SARSA may not necessarily converge to an optimal policy.
% However, SARSA offers flexibility in designing the policy improvement operator to suit different needs.
% For instance,  one can learn parametrized distributions over actions or other soft policies to maintain the exploration capability and be more robust.
Additionally, when the policy improvement operator maps to a fixed point $\pi$, SARSA reduces to TD learning, which evaluates the policy $\pi$.
Generally, SARSA searches the \textit{optimal} policy within the policy space $\Gamma(\Theta)$ determined by the policy improvement operator and the parameter space.

% subsection Policy Improvement (end)

\paragraph{Server side aggregation.} \label{alg:agg}
% Refer to $K$ and talk about limiting cases $K=0,\infty$]
\fedsarsa adds an additional aggregation step to parallelize linear SARSA. During this step, agents communicate with a central server by sending their parameters or parameter progress over a given period. The central server then aggregates these local parameters and returns the updated parameters to the agents. Intuitively, if the agents' MDPs are similar, i.e., the level of heterogeneity is low, then exchanging information via the server should benefit each agent. This is precisely the rationale behind the server-aggregation step. In general, \( K \) is selected to strike a balance between the communication cost and the accuracy in FL.

	Besides averaging, we add a projection step to ensure stability of the parameter sequence.
	This technique is commonly used in the literature on stochastic approximation and RL  \citep{zou2019Finitesampleanalysis,bhandari2018FiniteTime,qiu2021finite,wang2023FederatedTemporal}.
	%  and removes the need for any boundedness assumption.
	In practice, it is anticipated that an \textit{implicit} bound on the parameters exists without requiring explicit projection. 
	% {\color{orange}However, we believe it is not easy to establish the finite time mean squared error without the projection.
	% We provide further remarks about this projection in Appendix.}
% \end{remark}

\begin{algorithm}[!th]
	\caption{\fedsarsa} \label{alg}
	\Input{Initial parameter $\theta_{0}\i= \bar{\theta}_{0}$}\;
	\For{$t = 0,\dots, T-1$} {
		\ForPar{each agent $i = 1, \ldots , N$} {
			$\pi\it = \Gamma(\theta\it)$ \tcp*[t]{policy improvement}
			Sample observation $(s\it, a\it, r\it, s_{t+1}\i, a\itt)$ following policy $\pi\i_{t}$\;
			% , reward function $r\i$, and model~$P\i$\;
			% $g\it = \left( r\it + \gamma \phi^T\left(s\itt,a\itt\right) \theta\it - \phi^T\left(s\it,a\it\right)\theta\it \right)\phi\left(s\it,a\it\right)$\;
			$\theta\itt = \theta\it + \alpha_t g\it$, where $g\it$ is defined in \eqref{eq:g} \tcp*[t]{local update}
			% $\theta_{t+1}^{i} = \sum_{j\in \mathcal{N}_i}W_{ij}\theta_{t+1}^{j}$ \tcp{Gossip Aggregation}
		}

		\If (\tcp*[f]{every $K$ iterations}) {$t+1 \equiv 0 \pmod{K}$} {
			$\bar{\theta}_{t+1} = \Pi_{\bar{G}}\left( \frac{1}{N}\sum_{i=1}^{N}\theta_{t+1}\i\right)$ \tcp*[t]{federated aggregation}
			% $\bar{\theta}_{t+1} = \frac{1}{N}\sum_{i=1}^{N}\theta_{t+1}\i$ \tcp*[t]{federated aggregation}
			Set $\theta\itt = \bar{\theta}_{t+1}$ for each agent $i\in [N]$ % \label{line:agg}
		}
	}
\end{algorithm}

\section{Analysis}\label{sec:anlys}

We begin our  analysis of \fedsarsa by establishing a perturbation bound on the solution to \eqref{eq:bellman}, which captures the near-optimality of the solution under reward and transition heterogeneity.
We then provide a finite-time error bound of \fedsarsa, which enjoys the linear speedup achieved by the federated collaboration.
Building on this, we discuss the parameter selection of our algorithm.
%Finally, we present a proof sketch of our finite-time result by breaking down the learning process of \fedsarsa and offer some insights into its performance.

\subsection{Near Optimality under Heterogeneity} % (fold)

We consider an FRL task where all agents collaborate to find a universal policy.
However, due to environmental heterogeneity, each agent has a potentially different optimal policy.
Therefore, it is essential to determine the convergence region of our algorithm, and how it relates to the optimal parameters of the agents.
To show that we find a near-optimal parameter for all agents, we need to characterize the difference between the optimal parameters of agents.
Given the  operator $\Gamma$, we denote by $\theta\i_{*}$ the unique solution to \eqref{eq:bellman} for MDP $\mathcal{M}\i$.
% In the next subsection, we demonstrate that our algorithm converges to a neighborhood of $\theta_{*}$ whose radius is controlled by the environmental heterogeneity.
The next theorem bounds the distance between agents' optimal parameters as a function of reward- and transition kernel heterogeneity.
% These two results together provide a near-optimality guarantee under heterogeneity.

\begin{theorem}[Perturbation bounds on SARSA fixed points] \label{thm:fix-drift}
	There exist positive problem dependent constants $w$, $H$, and $\sigma$ such that
	\[
		\max_{i,j\in[N]}\left\{\left\| \theta\i_{*} - \theta\j_{*} \right\|_{2}\right\}
		\le \frac{R\epsilon_r + H\sigma\epsilon_p}{w} \eqqcolon \frac{\Lambda(\epsilon_p,\epsilon_r)}{w},
	\]
	where $\epsilon_p$ and $\epsilon_r$ are the perturbation bounds on environmental models defined in \cref{asmp:ker-het,asmp:r-het}. 
\end{theorem}

% {\cor In the tabular setting, $w \asymp \mu_{\min}(1-\gamma)$, where $\mu_{\min}$ is the probability lower bound of visiting any state-action pair in the steady distribution corresponding to the optimal policy, represents the exploraibility of the environment}. 
We explicitly define the constants in \cref{thm:fix-drift} and show that $w\! = \!O(1\!-\!\gamma)$ in \cref{sec:use-lem}.
In the next subsection, we demonstrate that there exists a parameter $\theta_{*}$ such that $\|\theta_{*}\i - \theta_{*}\| \le \Lambda(\epsilon_{p},\epsilon_{r}) /w$, and \cref{alg} converges to a neighborhood of $\theta_{*}$ whose radius is also of $O(\Lambda(\epsilon_{p},\epsilon_{r}) /(1\!-\!\gamma))$.
Since $\Lambda(\epsilon_p,\epsilon_r) = O(\epsilon_p + \epsilon_r)$, when the environmental heterogeneity is small, these results guarantee that $\theta_{*}$ is near-optimal for all agents.

\cref{thm:fix-drift} is the first perturbation bound on nonlinear projected Bellman fixed points.
\citet{wang2023FederatedTemporal} established similar perturbation bounds for linear projected Bellman fixed points using the perturbation theory of linear equations. 
However, it is crucial to note that their approach does not extend to our setting where \cref{eq:bellman} is nonlinear.
%To tackle this inherent nonlinearity, we adopt an innovative approach:
%Characterizing disparities among different Bellman operators. Surprisingly, we show that their perturbation bounds can be managed through the perturbation bounds of the optimal parameters. This insight leads to a succinct recursion, ultimately yielding \cref{thm:fix-drift}.}

%* trim
\iffalse
To the best of our knowledge, \cref{thm:fix-drift} is the first perturbation bound on SARSA fixed points.
\citet{wang2023FederatedTemporal} established similar perturbation bounds for TD(0) fixed points. However, as TD(0) fixed points correspond to a linear projected Bellman equation, the perturbation theory of linear equations can be directly used to provide the bounds, which does not apply to the nonlinear equation~\eqref{eq:bellman}.
Despite this, we have achieved a more straightforward expression for our bound that offers more explicit dependencies on the parameters.
Moreover, our work unifies the perturbation bounds and convergence region for SARSA under a single linear function $\Lambda(\epsilon_p,\epsilon_r)$, whereas \citep{wang2023FederatedTemporal} requires two more complex functions. Since SARSA reduces to TD(0) when the policy improvement operator maps all parameters to a fixed policy, our results also apply to TD(0).
\fi

% subsection Near Optimality under Heterogeneity (end)

\subsection{Finite-Time Error and Linear Speedup} % (fold)

We now provide the main theorem of the paper, which bounds the mean squared error of \cref{alg} recursively, and directly gives several finite-time error bounds.

\begin{theorem}[One-step progress] \label{thm}
	% If $\left\| \bar{\theta}_s \right\| \le G$ holds for all $s\le t$, then
	Let $\{ \theta\it \}$ be the parameters returned by Algorithm~\ref{alg} and $\bar{\theta}_t = \frac{1}{N}\sum_{i=1}^{N}\theta\it$.
	% Suppose $\{ \bar{\theta}_t \}$ is bounded.
	Then, there exist positive problem dependent constants $w,C_1,C_2,C_3,C_4$, and a parameter $\theta_{*}$ such that $\max_{i\in[N]}\|\theta_{*}\i - \theta_{*}\| \le \Lambda(\epsilon_{p,}\epsilon_{r})/w$, and for any $t\in\mathbb{N}$, it holds that
	\begin{equation}\label{eq:thm}
		\EE\left\| \bar{\theta}_{t+1}-\theta_{*} \right\|^2
		\le (1 - \alpha _{t}w)\EE\left\| \bar{\theta}_t-\theta_{*} \right\|^2
		+ \alpha  _tC_1\Lambda^2(\epsilon_p,\epsilon_r)
		+ \alpha^2_t C_2 /N
		+ \alpha^3_tC_3
		+ \alpha^4_tC_4.\end{equation}
	Explicit definitions of the constants are provided in \cref{sec:thm-pf}.
\end{theorem}

On the right-hand side of \cref{eq:thm}, the first term is a contractive term that inherits its contractivity from the projected Bellman operator; the second term accounts for heterogeneity; the third term captures the effect of noise where the variance gets scaled down by a factor of $N$ (linear speedup) due to collaboration among agents; the last two terms represent higher-order terms, which are negligible, compared to other terms. In the following two corollaries, we study the effects of using constant and decaying step-sizes in the above bound. 
\iffalse
\hw{Subsequently, we demonstrate that \fedsarsa with a constant step-size converges geometrically to a ball of size $B$\footnote{Refer to Corollary~\ref{cor:err-con} for the formal definition of $B$.} around each agent's optimal point $\theta_i^*$. In contrast, with decaying step-sizes, \fedsarsa converges sublinearly to a smaller ball (of size less than $B$) around the optimal point. We formally present these results in the following corollaries.}
\fi
\begin{corollary}[Finite-time error bound for constant step-size]\label{cor:err-con}
  With a constant step-size
  % $\alpha_{t}= \alpha_0 \le \min \{ 1 /(8K), w /64 \}$,
  $\alpha_t\equiv \alpha_0\le w/(2120(2K+8+\ln (m/(\rho w))))$,
  for any $T\in\mathbb{N}$, we have
  \[
    \EE\left\| \bar{\theta}_T - \theta_{*}\i \right\|^2 \le 4e^{-\alpha_{0}wT}\left\| \theta_0-\theta_{*}\i \right\|^2 +
    \frac{1}{w}\left( \left(C_1 + \frac{6}{w}\right)\Lambda^2(\epsilon_p,\epsilon_r) + \alpha_0\frac{C_2}{N} + \alpha_{0}^2C_3 + \alpha_{0}^{3}C_4 \right).
\]

	% With a constant step-size
	% $\alpha _{t}\equiv \alpha_0 = O(\exp (-\max \{ K, 1 /w \}))$,
	% % $\alpha_{t}= \alpha_0 \le \min\{ 1 /(8K), w /64 \}$,
	% for any $T\in\mathbb{N}$ and $i\in [N]$, it holds that
	% \[
	% 	\EE\left\| \bar{\theta}_T - \theta\i_{*} \right\|^2 \le 4e^{-\alpha_{0}wT}\left\| \theta_0-\theta\i_{*} \right\|^2 + B,
	% \]
	% where $B$ is the squared convergence region radius satisfying the following control:
	% $$
	% 	B = O\left(\alpha_{0}^2 H^2 \tau^{4} + \frac{\alpha_{0}H^2\tau}{N(1-\gamma)} + \frac{\Lambda^2(\epsilon_{p},\epsilon_{r})}{(1-\gamma)^2}\right)
	% 	= \widetilde{O}\left(\alpha_{0}^2 H^2 + \frac{\alpha_{0}H^2}{N(1-\gamma)} + \frac{\Lambda^2(\epsilon_{p},\epsilon_{r})}{(1-\gamma)^2}\right),
	% $$
	% where $\tau \asymp \log \alpha_{T}^{-1}$ is the backtracking period (see \cref{sec:pf-sketch}) and the asymptotic notation $\widetilde{O}$ suppresses the logarithmic dependencies.
	% $$
	% 	B = \frac{1}{1-\gamma} \cdot O\left(\frac{\Lambda^2(\epsilon_{p},\epsilon_{r})+\alpha_0^2 H^2 K^2}{1-\gamma}+\frac{\alpha_0 H^2}{N}\right).
	% $$
	% That is, the parameter will converge exponentially to a region around the optimal parameter in expectation.
\end{corollary}

% \[
% 	\mathbb{E}\left\|Q_{\widetilde{\theta}_{T}}-Q_{\theta_i^*}\right\|^2
% 	\le \tilde{O}\left(\frac{1}{T^2}+\frac{1}{NT}+\Lambda^2(\epsilon_p,\epsilon_r)\right),
% \]
% And for the parameter error, we have
\begin{corollary}[Finite-time error bound for decaying step-size]\label{cor:err-dec}
	With a linearly decaying step-size $\alpha _{t} = 4/(w(1+t+a)),$ where $a>0$ is to guarantee that $\alpha_0\le\min\{ 1 /(8K), w /64 \}$, there exists a convex combination $\widetilde{\theta}_{T}$ of $\{ \bar{\theta}_t \}_{t=0}^{T}$ such that
	% \[
	% 	\begin{aligned}
	% 		\EE\bigl\| \widetilde{\theta}_T - \theta_{*}\bigr\|^2
	% 		\le & \frac{H^2}{(1-\gamma)^2} \cdot {O}\left(\frac{((1-\gamma)\tau^{4} + \tau)\log T}{(1-\gamma)^2 T^2}+\frac{1}{N T}+\frac{\Lambda^2\left(\epsilon_p, \epsilon_r\right)}{H^2}\right) \\
	% 		=   & \frac{H^2}{(1-\gamma)^2} \cdot \widetilde{O}\left(\frac{1}{(1-\gamma)^2 T^2}+\frac{1}{N T}+\frac{\Lambda^2\left(\epsilon_p, \epsilon_r\right)}{H^2}\right)
	% 		,\end{aligned}
	% \]
    \[\resizebox{1\hsize}{!}{$
			\EE\left\| \widetilde{\theta}_{T} - \theta\i_{*} \right\|^2 \!\!=\! \frac{H^2}{(1-\gamma)^2} \!\cdot\! O\!\left(\frac{K^2 + \tau^{5}}{(1-\gamma)^2T^2} \!+\! \frac{\tau}{NT} \!+\! \frac{\Lambda^2(\epsilon_{p},\epsilon_{r})}{H^2}\right) \!=\! \frac{H^2}{(1-\gamma)^2} \!\cdot\! \widetilde{O}\!\left(\frac{1}{NT} \!+\! \frac{\Lambda^2(\epsilon_{p},\epsilon_{r})}{H^2} \right).
    $}\]
\end{corollary}

% \begin{corollary}[Finite-time error bound for Tabular SARSA]\label{cor:tabular}
% 	With a linearly decaying step-size $\alpha _{t} = 4/(w(1+t+a)),$ where $a>0$ is to guarantee that $\alpha_0\le\min\{ 1 /8K, w /64 \}$, there exists a convex combination $\widetilde{\theta}_{T}$ of $\{ \bar{\theta}_t \}_{t=0}^{T}$ such that
% 	\[
% 	\EE\bigl\| \widetilde{\theta}_T - \theta_{*}\bigr\|^2 \le
% 	\frac{S^2A^2}{(1-\gamma)^4} \cdot \widetilde{O}\left(\frac{1}{(1-\gamma)^2 T^2}+\frac{1}{N T}+\frac{\Lambda^2\left(\epsilon_p, \epsilon_r\right)}{S^2A^2}\right).
% 	\]
% 	where the asymptotic notation suppresses the logarithmic factors.
% \end{corollary}

We now discuss the implications of the above theoretical guarantees.

\vspace{-0.5em}
\paragraph{Convergence region.} From \cref{cor:err-con}, with a constant step-size $\alpha$, \fedsarsa exponentially converges to a ball around the optimal parameter $\theta^*_i$ of each agent. The radius of this ball is governed by two objects: (i) the level of environmental heterogeneity; (ii) the inherent noise in our model. In the absence of heterogeneity, the above guarantee is precisely what one obtains for stochastic approximation algorithms with a constant step-size~\citep{zou2019Finitesampleanalysis, srikant2019Finitetimeerror, bhandari2018FiniteTime}. The presence of heterogeneity manifests itself in the $O(\Lambda(\epsilon_p,\epsilon_r) / (1-\gamma)) = O(\epsilon_p + \epsilon_r)$ term in the convergence region radius. Since the optimal parameters of the agents may not be identical (under heterogeneity), such a term is generally unavoidable. 

\vspace{-0.4em}
\paragraph{Linear speedup.} Turning our attention to Corollary~\ref{cor:err-dec} (where we use a decaying step-size), let us first consider the homogeneous case where $\epsilon_p=\epsilon_r=0$. When $T \geq N$, the $O(1/(NT))$ rate we obtain in this case is the best one can hope for statistically: with $T$ data samples per agent and $N$ agents, one can reduce the variance of our noise model by at most $NT$. Thus, for a homogeneous setting, our rate is optimal, and clearly demonstrates an $N$-fold linear speedup over the single-agent sample-complexity of $O(1/T)$ in~\cite{zou2019Finitesampleanalysis}. In this context, \emph{our work provides the first such bound for a federated on-policy RL algorithm}, and complements results of a similar flavor for the off-policy setting in~\cite{khodadadian2022FederatedReinforcement}. When the agents' MDPs differ, via collaboration, each agent is still able to converge at the \textit{expedited} rate of $O(1/NT)$ to a ball of radius $O(\epsilon_p + \epsilon_r)$ around the optimal parameter of each agent. The implication of this result is simple: by participating in federation, each agent can \textit{quickly} (i.e., with an $N$-fold speedup) find an $O(\epsilon_p + \epsilon_r)$-approximate solution of its optimal parameter; using such an approximate solution as an initial condition, the agent can then fine-tune (personalize) -  based on its own data - to converge to its own optimal parameter \textit{exactly} (in mean-squared sense). \emph{This is the first  result of its kind for federated planning, and complements the plethora of analogous results in federated optimization}~\citep{sahu, khaled1, li, koloskova, woodworth2020minibatch, fedsplit, FedNova, mitraNIPs, proxskip}. Arriving at the above result, however, poses significant challenges relative to prior art. We now provide insights into these challenges and our strategies to overcome them.

\subsection{Proof Sketch: Error Decomposition} %\label{sec:pf-sketch}
Our main approach of proving \cref{thm} is to leverage the contraction property of the Bellman equation \eqref{eq:bellman} to identify a primary ``descent direction.'' \cref{alg} then updates the parameters along this direction with multi-sourced stochastic bias.
We provide an informal mean squared error decomposition (formalized in \cref{sec:err-dec} ) to illustrate this idea:
\begin{align*}
	\EE\left\| \bar{\theta}_{t+1}-\theta_{*} \right\|^2 \le& \ \text{recursion} + \text{descent direction} + \text{gradient heterogeneity} + \text{client drift} \\ &+ \text{gradient progress} + \text{mixing} + \text{backtracking} + \text{gradient variance}
.\end{align*}
%We comment on each term and provide a pointer to the corresponding section in Appendix.
% Here, we highlight the trajectories introduced by \emph{backtracking}, a virtual process that plays a crucial role in the analysis of \fedsarsa.
Some of these terms commonly appears in an FRL analysis: the descent direction is given by the contraction property of the Bellman equation \cref{eq:bellman} when the policy improvement operator is sufficiently smooth (\cref{sec:des-dir}); the client drift represents the deviation of agents' local parameters from the central parameter, which is controlled by the step-size and synchronization period (\cref{sec:drift}); the mixing property (\cref{asmp:steady}) allows a stationary trajectory to rapidly reach to a steady distribution (\cref{sec:mix}). We highlight some unique terms in our analysis.

% \emph{Descent direction.}
% Mirroring gradient descent, the contraction property of the Bellman equation \cref{eq:bellman} provides a direction of sufficient descent towards $\theta_{*}$, given that the policy improvement operator is smooth enough (\cref{sec:des-dir}).

\emph{Gradient heterogeneity.}
This term accounts for the local update heterogeneity, which scales with the environmental heterogeneity. The effect of time-varying policies coupled with multiple local updates accentuates the effect of such heterogeneity. Thus, particular care is needed to ensure that the bias introduced by heterogeneity does not compound over iterations (\cref{sec:grad-hetero}).

% \emph{Client drift.}
% This term represents the deviation of agents' local parameters from the central parameter, which is controlled by the step-size $\alpha_t$ and synchronization period $K$ (\cref{sec:drift})

\emph{Backtracking.}
\fedsarsa possesses nonstationary transition kernels.
To deal with this challenge and use the mixing property of stationary MDPs, we \textit{virtually backtrack} a period $\tau$: starting at time step $t\!-\!\tau$, we fix the policy $\Gamma(\theta\ita)$ for agent $i$, and consider a subsequent virtual trajectory following this fixed policy.
The divergence between the updates computed on real and virtual observations is controlled by the step-size $\alpha_t$ and backtracking period $\tau$ (\cref{sec:stat}).

%\emph{Mixing.} 
%The virtual trajectory after backtracking rapidly converges to a steady distribution due to the mixing property in \cref{asmp:steady}, and thus a small backtracking period proves adequate (\cref{sec:mix}).

\emph{Gradient progress.} 
Note that the steady distribution in the mixing term corresponds to an \emph{old} policy. Since the backtracking period is small, the discrepancy (progress) between this old policy and the current one is small (\cref{sec:grad-prog}).

\emph{Gradient variance.} While one can directly use the projection radius to bound the semi-gradient variance, such an approach would fall short of establishing the desired linear speedup effect. To achieve the latter, we need a more refined argument that shows how one can obtain a ``variance-reduction" effect by combining data generated from non-identical time-varying Markov chains (\cref{sec:grad-norm}).

\section{Simulations} \label{sec:exp}
We create a finite state space of size $|\mathcal{S}| = 100$, an action space of $|\mathcal{A}|=100$, a feature space of dimension $d = 25$, and set $\gamma = 0.2$ and $R = 10$. The actions determine the transition matrices by shifting the columns of a reference matrix.
%We employ the softmax function
%with a temperature of $100$ 
%as the policy improvement operator.
The synchronization period is set to $K=10$, and the step-size of $\alpha_0=0.01$.
For the full experiment setup, please refer to \cref{sec:apx-exp}. In \cref{fig:sarsa}, we plot the mean squared error averaged over ten runs for different heterogeneity levels and numbers of agents.
The simulation results are consistent with \cref{cor:err-con} and demonstrate the robustness of our method towards environmental heterogeneity.
Additional simulations, including federated TD(0) and on-policy federated Q-learning covered by our algorithm, can be found in \cref{sec:apx-exp}.

\begin{figure}[ht]
  \centering
  \begin{subfigure}{0.325\textwidth}
      \centering
\includegraphics[width=\textwidth]{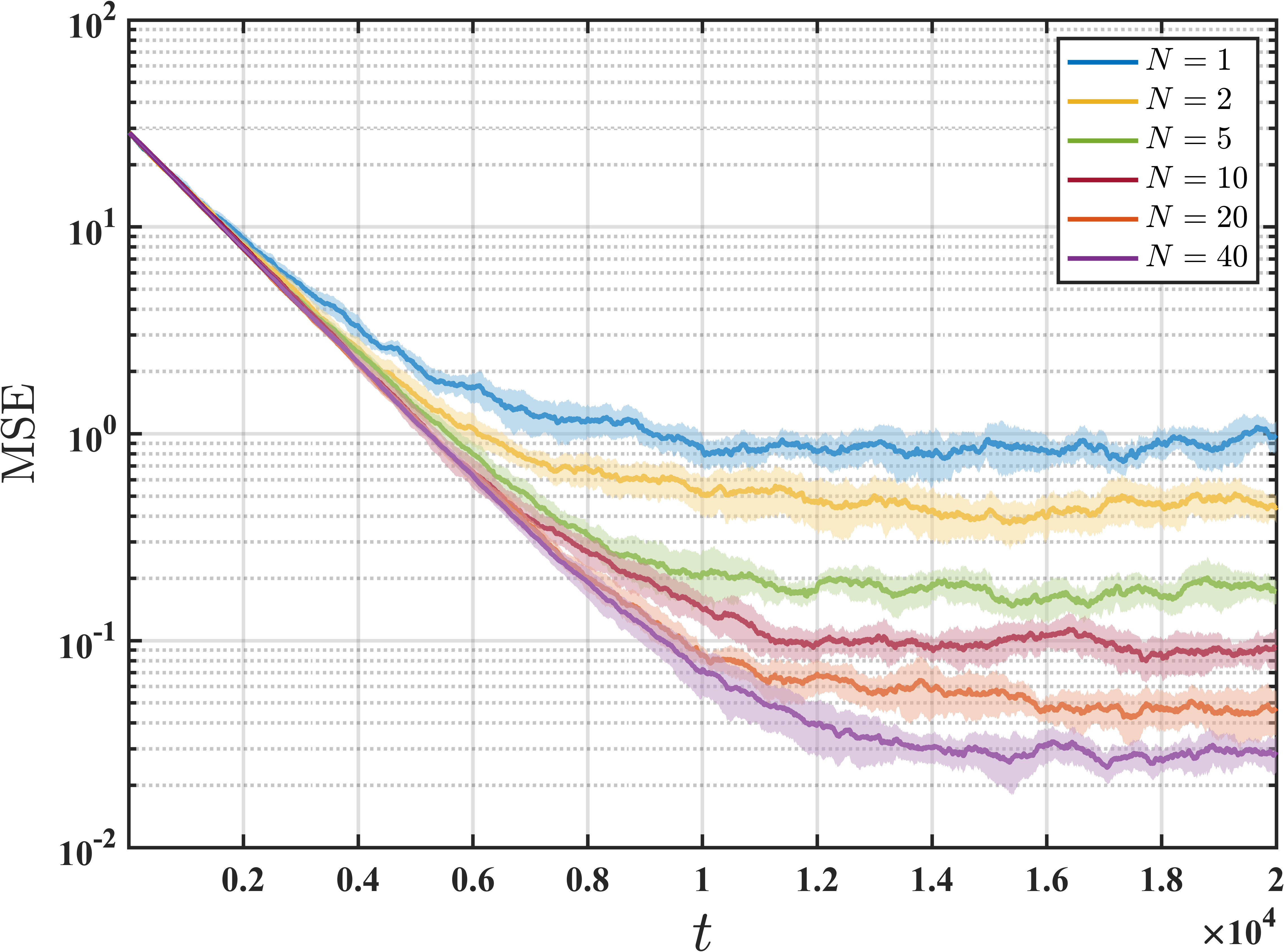}\vspace{-5pt}
      \caption{$\epsilon_p=\epsilon_r = 0$}
  \end{subfigure}
  \hfill
  % \begin{subfigure}[b]{0.325\textwidth}
  %     \centering
  %     \includegraphics[width=\textwidth]{fig/e/0.1_0.1.png}\vspace{-5pt}
  %
  %     \caption{$\epsilon_p=\epsilon_r = 0.1$}
  % \end{subfigure}
  % \hfill
  % \begin{subfigure}[b]{0.325\textwidth}
  %     \centering
  %     \includegraphics[width=\textwidth]{fig/e/0.2_0.2.png}\vspace{-5pt}
  %
  %     \caption{$\epsilon_p=\epsilon_r = 0.2$}
  % \end{subfigure}
  % \begin{subfigure}{0.325\textwidth}
  %     \centering
  %     \includegraphics[width=\textwidth]{fig/e/0.5_0.5.png}\vspace{-5pt}
  %
  %     \caption{$\epsilon_p=\epsilon_r = 0.5$}
  % \end{subfigure}
  % \hfill
  \begin{subfigure}[b]{0.325\textwidth}
      \centering
      \includegraphics[width=\textwidth]{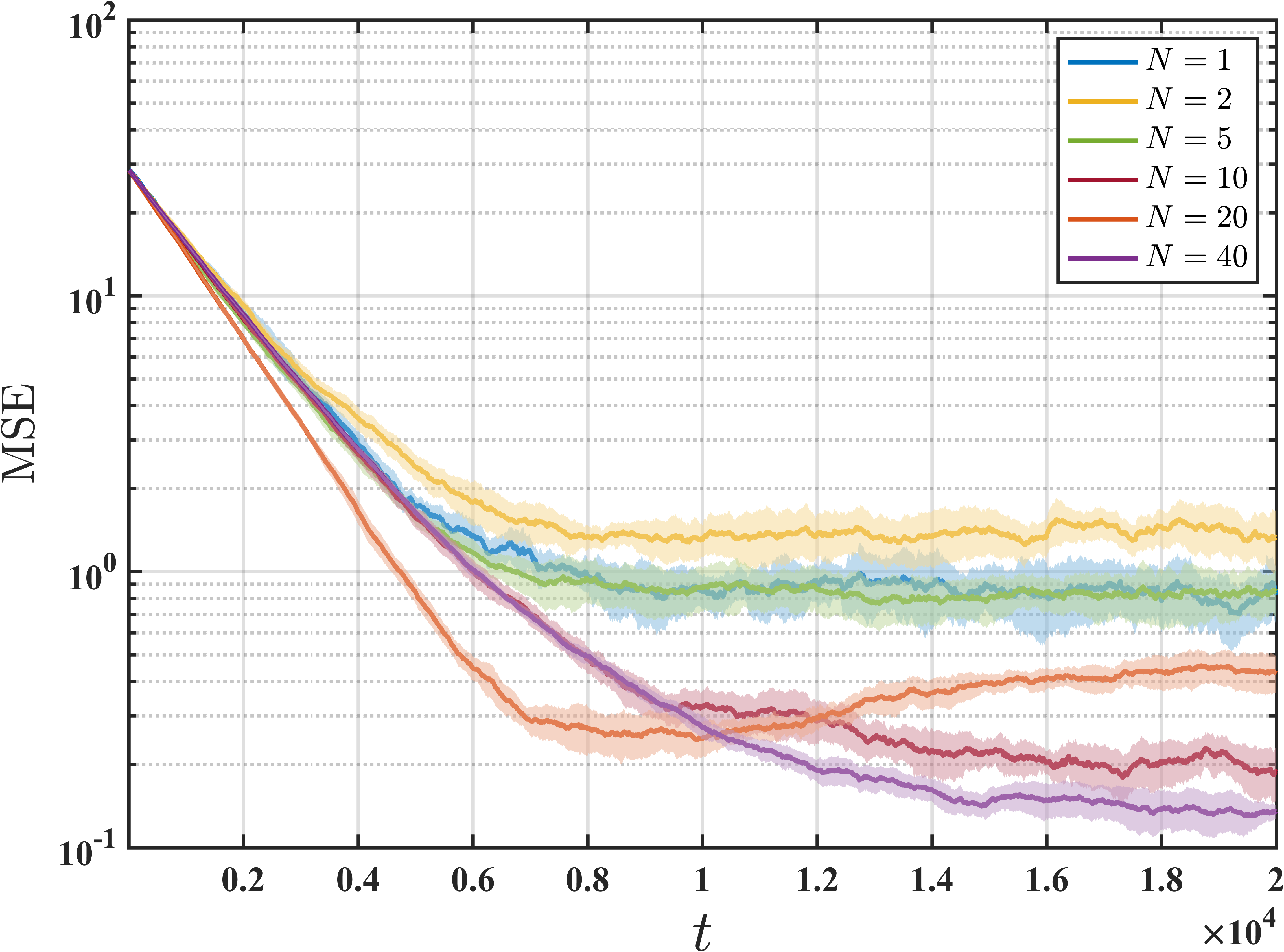}\vspace{-5pt}
\caption{$\epsilon_p=\epsilon_r = 1$}
  \end{subfigure}
  \hfill
  \begin{subfigure}[b]{0.325\textwidth}
      \centering
      \includegraphics[width=\textwidth]{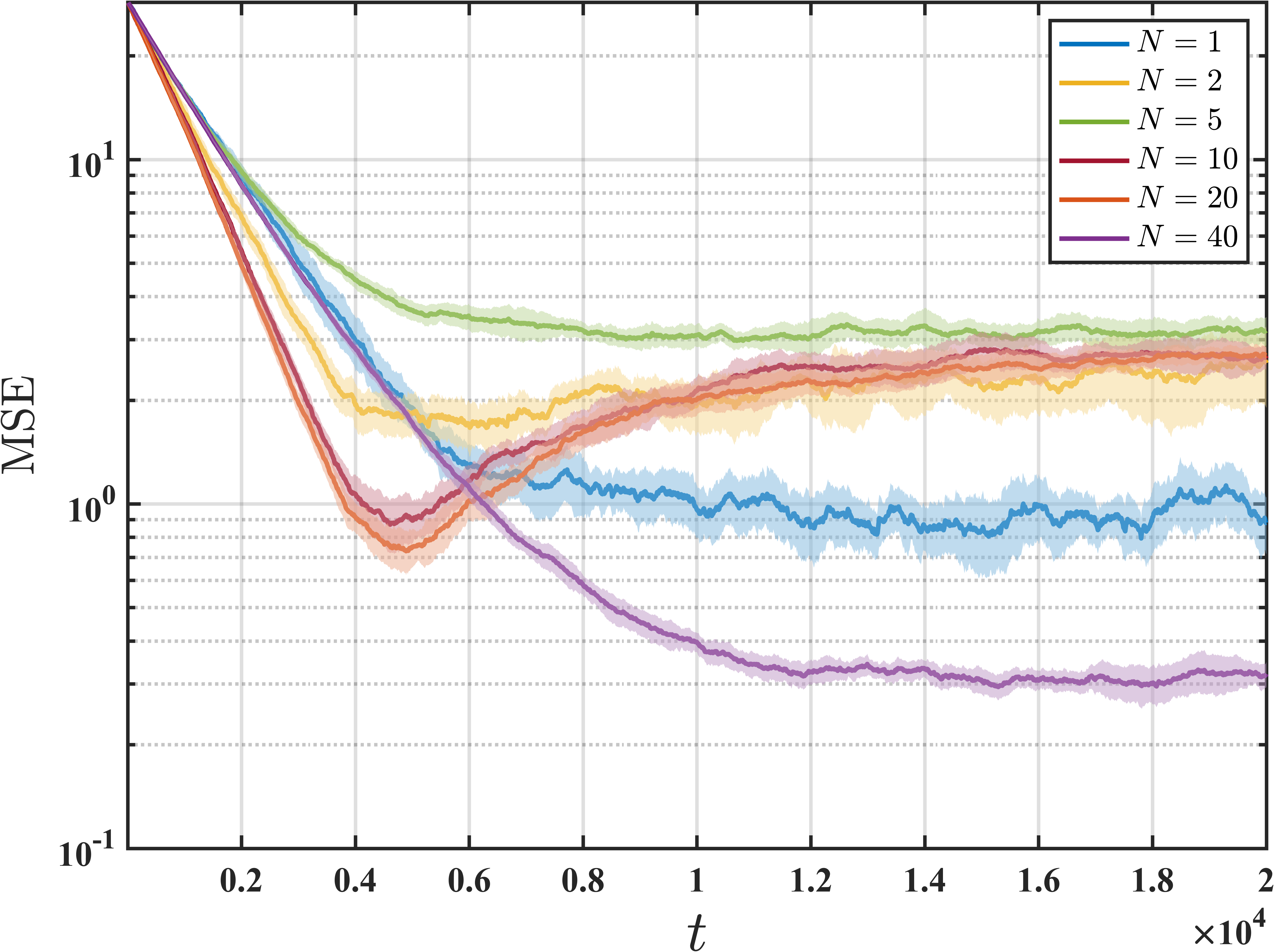}\vspace{-5pt}

      \caption{$\epsilon_p=\epsilon_r = 2$}
  \end{subfigure}
  \caption{Performance of \fedsarsa under Markovian sampling.}
  \label{fig:sarsa}
\end{figure}

% subsection \fedsarsa with Constant Step-Size (end)

\section{Conclusion}
We proposed a straightforward yet powerful on-policy federated reinforcement learning method: \fedsarsa.
Our finite-time analysis of \fedsarsa provides the first theoretically conformation of the statement:
an agent can expedite the process of learning its own near-optimal policy by leveraging information from other agents with potentially different environments.
% Future work includes allowing agents to  learn \emph{personalized} optimal policies precisely, even in scenarios with substantial environmental heterogeneity. Addressing the issue of partially observability in the federated reinforcement learning setting is an important next step.

\subsubsection*{Acknowledgments}
JA is partially supported by Columbia Data Science Institute and NSF grants ECCS 2144634 \& 2231350.

\bibliography{fl.bib,rl.bib}

%%%%%%%%%%%%%%%%%%%%%%%%%%%%%%%%%%%%%%%%
% APPENDIX
%%%%%%%%%%%%%%%%%%%%%%%%%%%%%%%%%%%%%%%%
\newpage
\doparttoc
\faketableofcontents
\appendix
\part{Appendix} %\label{apx} % Start the appendix part
\parttoc
\renewcommand{\thesection}{\Alph{section}}
\newpage

\section{Organization of Appendix}
	The appendix is organized as follows.
	First, we present an additional comparison of our results with other finite-time results in \cref{sec:comp}, and additional simulation results in \cref{sec:apx-exp}.
	% Then, in \cref{sec:proj}, we provide some remarks on the explicit projection in our algorithm.
	In \cref{sec:cmdp,sec:nota}, we introduce the concept of central MDP and some notation that will assist our analysis.
	In \cref{sec:const}, to aid readability, we list all the constants that appear in the paper for readers' convenience.
	In \cref{sec:aux}, we provide several preliminary lemmas that will be used throughout the analysis.
	Before presenting lemmas for \cref{thm}, we first prove \cref{thm:fix-drift} in \cref{sec:thm-drift-pf}, for it will be used by later lemmas.
	In \cref{sec:use-lem}, we first decompose the mean squared error and then present seven lemmas, each bounding one term in the decomposition.
	Then, we provide the proof of \cref{thm} and \cref{cor:err-con,cor:err-dec} in \cref{sec:thm-pf} and \cref{sec:cor-pf}, respectively.
	% Finally, we rigorously discuss the dependencies of constants in \cref{sec:const-dep}, providing insights into our results.
	To provide insights into our results, we discuss the dependencies of constants in \cref{sec:const-dep}.
	Finally, we reduce \fedsarsa to the tabular case in \cref{sec:tab}, demonstrating the flexibility and efficiency of our algorithm.

\section{Finite-Time Results Comparison} \label{sec:comp}

A comparison of finite-time results on temporal difference methods is provided in \cref{tb:comp-finite}.

\renewcommand*{\thefootnote}{\fnsymbol{footnote}}

\begin{table}[ht]
  \caption{Comparison of finite-time results.
    Results with green background are first provided by our work; results with blue background are covered by our work.
    ``Linear'' indicates the usage of linear function approximation, and ``Hetero'' indicates the presence of environmental heterogeneity.
    All constants are defined in \cref{sec:anlys,sec:use-lem}.
    % Our algorithm covers on-policy Q-Learning algorithms but our theoretical results do not directly apply to them.
    We show the squared $\ell_{2}$ error for linear settings and squared $\ell_{\infty}$ error for tabular settings. Asymptotic notations are omitted for simplicity.}
  \label{tb:comp-finite}
  \begin{threeparttable}
    \centering
    \centerline{\resizebox{1.2\textwidth}{!}{
        \begin{booktabs}{
          colspec = {ccccccc},
          vline{2-7},
          leftsep = 1pt,
          rightsep = 1pt,
          row{1-3} = {abovesep=0pt},
          stretch = 0,
          cell{1}{2} = {c=4}{c}, % multicolumn (column number=2, center alignment)
          cell{1}{6} = {c=2}{c},
          cell{2}{2} = {c=2}{c}, cell{2}{4} = {c=2}{c},
          cell{2}{6} = {r=2}{c}, cell{2}{7} = {r=2}{c},
          cell{4}{1-3,5-6} = {c,cyan!30},
          cell{6}{6} = {c,cyan!30},
          cell{6}{1-5,7} = {c,green!30},
          cell{4}{4,7} = {c,green!30},
            }
          \toprule
          & Federated &&&& Single-Agent & \\\cmidrule[lr]{2-5}\cmidrule[lr]{6-7}
          & Linear && Tabular && Linear & Tabular \\\cmidrule[lr]{2-3}\cmidrule[lr]{4-5}
          & Hetero & Homog & Hetero & Homog &  & \\\midrule
          TD Learning &
          \SetCell{bg=cyan!30} $\frac{H^2}{(1-\gamma)^2NT} + \frac{\Lambda^2}{(1-\gamma)}$ \tnote{\textdagger} &
          $\frac{H^2}{(1-\gamma)^2NT}$ \tnote{\textdaggerdbl} &
          $\frac{SA}{\lambda^2(1-\gamma)^4NT} + \frac{\Lambda^2}{\lambda^2(1-\gamma)^2}$ \tnote{**} &
          $\frac{S^2}{\lambda^{5}(1-\gamma)^{9}NT}$ \tnote{\textdaggerdbl} &
          $\frac{H^2}{(1-\gamma)^2 T}$ \tnote{\P} &
          $\frac{SA}{\lambda^2(1-\gamma)^4T}$ \tnote{**}\\[1ex]
          Q-Learning & -- &
          $\frac{H^2}{(1-\gamma)^2NT}$ \tnote{\textdaggerdbl} &
          $\frac{1}{(1-\gamma)^{6}T^2} + \frac{\Lambda^2}{(1-\gamma)^{4}}$ \tnote{\SSS} &
          $\frac{S^2}{\lambda^{5}(1-\gamma)^{9}NT}$ \tnote{\textdaggerdbl} &
          $\frac{H^2}{(1-\gamma)^2 T}$ \tnote{\P} &
          $\frac{SA}{\lambda(1-\gamma)^5T}$ \tnote{\textbardbl} \\[1ex]
          SARSA &
          $\frac{H^2}{(1-\gamma)^2NT} + \frac{\Lambda^2}{(1-\gamma)^2}$ \tnote{*} &
          $\frac{H^2}{(1-\gamma)^2NT}$ \tnote{*} &
          $\frac{SA}{\lambda^2(1-\gamma)^4NT} + \frac{\Lambda^2}{\lambda^2(1-\gamma)^2}$ \tnote{**} &
          $\frac{SA}{\lambda^2(1-\gamma)^4NT}$ \tnote{**}&
          $\frac{H^2}{(1-\gamma)^2 T}$ \tnote{\#} &
          $\frac{SA}{\lambda^2(1-\gamma)^4T} \tnote{**}$
          \\[1ex]
          \bottomrule
        \end{booktabs}}}
    \begin{tablenotes} \small
      \item[\textdagger] \citep{wang2023FederatedTemporal}
      \item[\textdaggerdbl] \citep{khodadadian2022FederatedReinforcement}
      \item[\SSS] \citep{jin2022Federatedreinforcement}
      \item[\P] \citep{bhandari2018FiniteTime}

      \item[\textbardbl] \citep{qu2020Finitetimeanalysis}
      \item[\#] \citep{zou2019Finitesampleanalysis}
      \item[*] \cref{cor:err-dec}
      \item[**] \cref{sec:tab}
    \end{tablenotes}
  \end{threeparttable}
\end{table}

\renewcommand*{\thefootnote}{\arabic{footnote}}
\setcounter{footnote}{0}

\section{Additional Simulations} \label{sec:apx-exp}

\subsection{Additional Simulations for \fedsarsa} \label{sec:exp-sarsa}

We first restate the simulation setup in more detail.
We index a finite state space by $\mathcal{S} = [100]$ and an action space by $\mathcal{A} = [100]$, where the actions determine the transition matrices by shifting the columns of a reference matrix \( P_0 \):
\[
  P_{a} = \texttt{circ\_shift}(P_0, \texttt{columns} = a),
\]
where $\texttt{circ\_shift}$ denotes a circular shift operator.
We construct the feature extractor as
\[
    \phi(s,a) = e_{(s\ \mathrm{mod}\ d_1) \cdot d_2 + a\ \mathrm{mod}\ d_2} \in \R^{d_1 \times d_2}
,\]
where \( e_i \) is the indicator vector with the \( i \)-th entry being \( 1 \) and the rest being \( 0 \).
We set \( d_1 = 5 \) and \( d_2 = 5 \).
For the policy improvement operator, we employ the softmax function with a temperature of $100$:
\[
    \pi_{\theta}(a|s) = \frac{\exp(\theta^{T}\phi(s,a)/100)}{\sum_{a'\in \mathcal{A}}\exp(\theta^{T}\phi(s,a')/100)}
.\] 
Other parameters are set as follows: the reward cap \( R = 10 \), the discount factor \( \gamma = 0.2 \), the synchronization period \( K = 10 \), and the step-size \( \alpha_0 = 0.01 \).

To construct heterogeneous MDPs,
we first generate a nominal MDP $\mathcal{M}_{1}$ and obtain the remaining MDPs by adding the perturbations to $\mathcal{M}_1$. 
Unlike in \textsf{FedTD} \citep{wang2023FederatedTemporal}, where the optimal parameters can be obtained by solving the linear projected Bellman equation directly, here we get a \textit{reference} parameter \( \theta^{(1)}_{\mathrm{ref}} \) by running a single-agent linear SARSA on $\mathcal{M}_1$ with decaying step-size. As suggested in \cref{cor:err-dec}, the reference parameter converges to the optimal parameter corresponding to $\mathcal{M}_1$.
Then, we calculate the mean squared error with respect to the reference parameter: $\left\|\bar{\theta}_t - \theta^{(1)}_{\mathrm{ref}}\right\|_{2}^2$.
All of our simulations are averaged over ten runs and all graphs are plotted with 95\% confidence region.

In \cref{fig:sarsa}, both kernel heterogeneity and reward heterogeneity are set at the same level.
In \cref{fig:sarsa-ee-p}, we fix the kernel heterogeneity as \( 1.0 \) and vary the reward heterogeneity.
In contrast, we fix the reward heterogeneity as zero and vary the kernel heterogeneity in \cref{fig:sarsa-ee-r}.
Again, these results affirm the robustness of our method towards environmental heterogeneity.
Furthermore, they seemingly suggest that the algorithm is more sensitive to reward heterogeneity than kernel heterogeneity.
However, it is important to note that \( \epsilon_p \) and \( \epsilon_r \) represent upper bounds and may be much larger than the actual heterogeneity level.

Further exploring the effect of heterogeneity on federated collaboration, \cref{fig:sarsa-np,fig:sarsa-nr} illustrate the effect of different reward and kernel heterogeneity levels on the performance of \fedsarsa respectively.
Generally, higher levels of heterogeneity result in larger mean squared error, which aligns with our theoretical results in \cref{sec:anlys}.

% fig:sarsa-ee-p
\begin{figure}[ht]
  \centering
  \begin{subfigure}{0.325\textwidth}
      \centering
      \includegraphics[width=\textwidth]{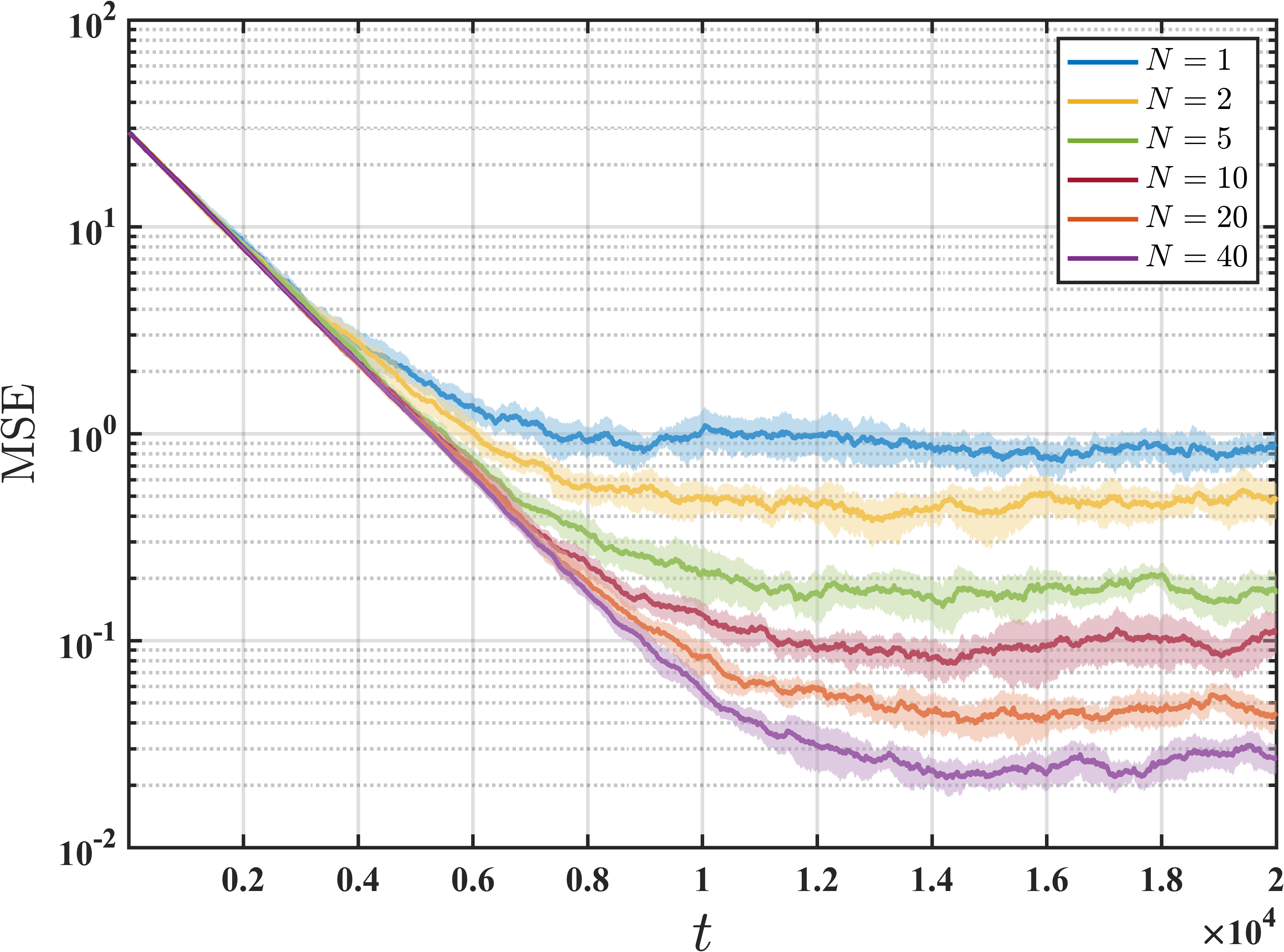}\vspace{-5pt}
      \caption{$\epsilon_r = 0$}
  \end{subfigure}
  \hfill
  \begin{subfigure}[b]{0.325\textwidth}
      \centering
      \includegraphics[width=\textwidth]{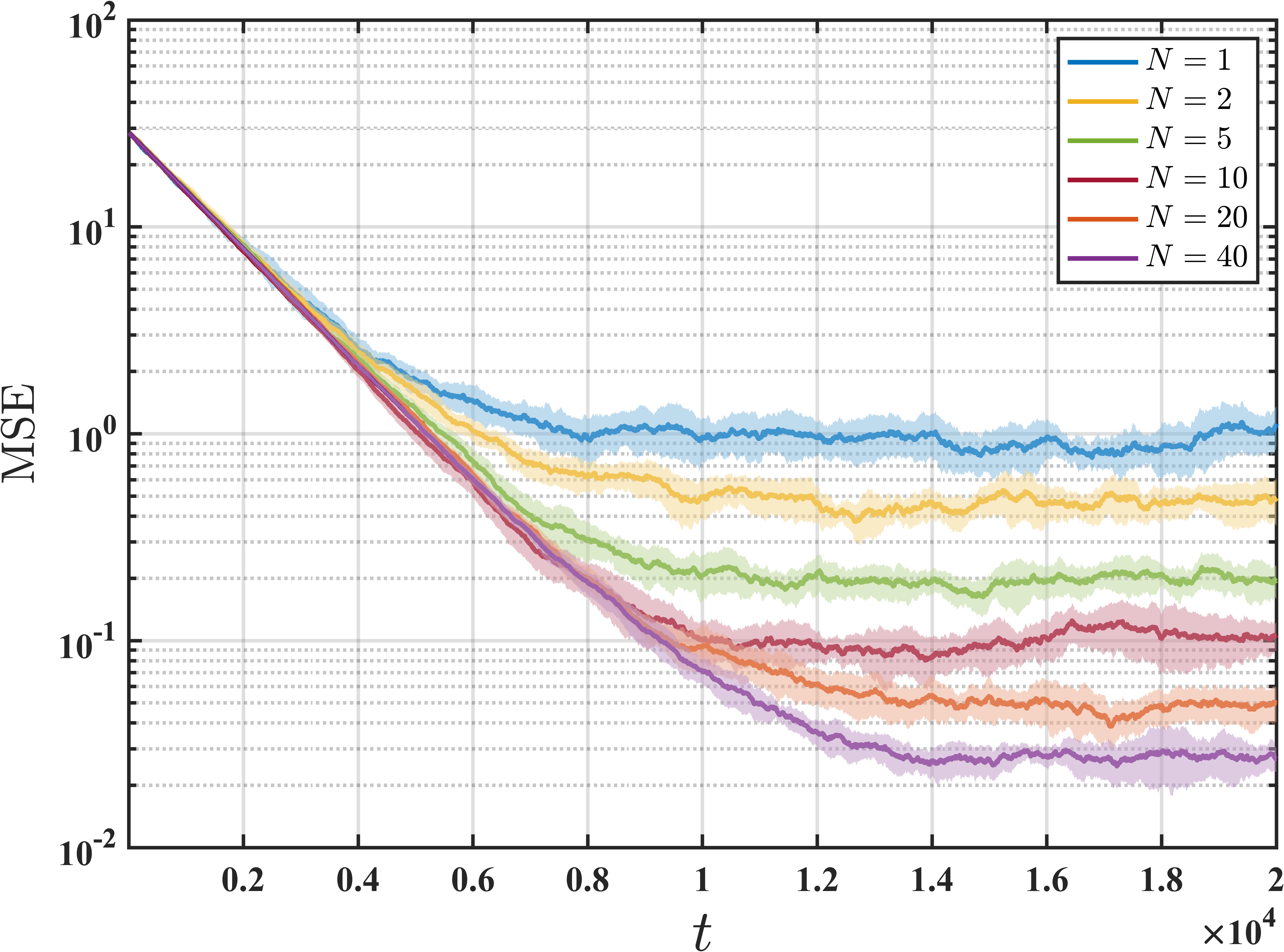}\vspace{-5pt}

      \caption{$\epsilon_r = 0.1$}
  \end{subfigure}
  \hfill
  \begin{subfigure}[b]{0.325\textwidth}
      \centering
      \includegraphics[width=\textwidth]{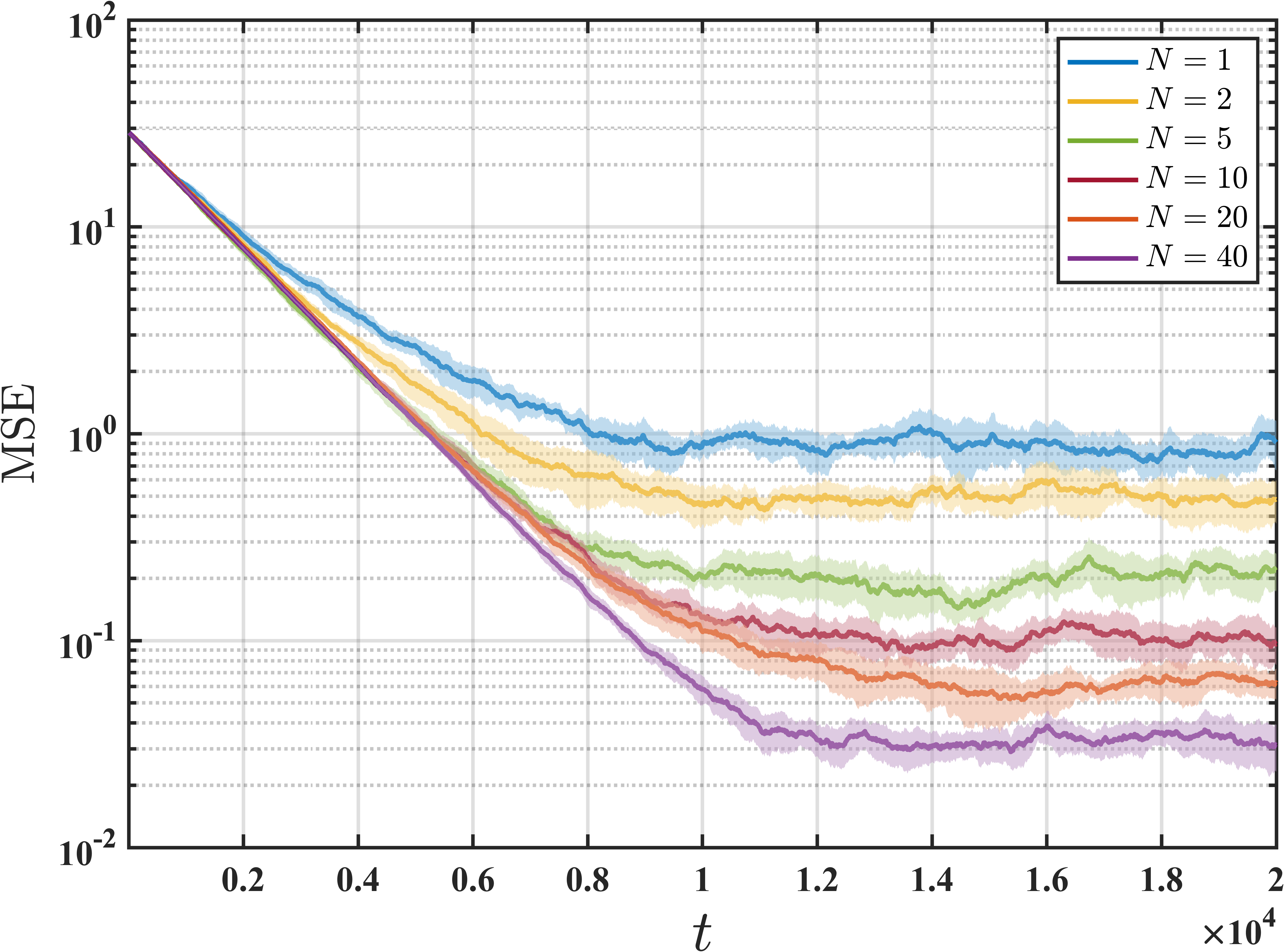}\vspace{-5pt}

      \caption{$\epsilon_r = 0.2$}
  \end{subfigure}
  \begin{subfigure}{0.325\textwidth}
      \centering
      \includegraphics[width=\textwidth]{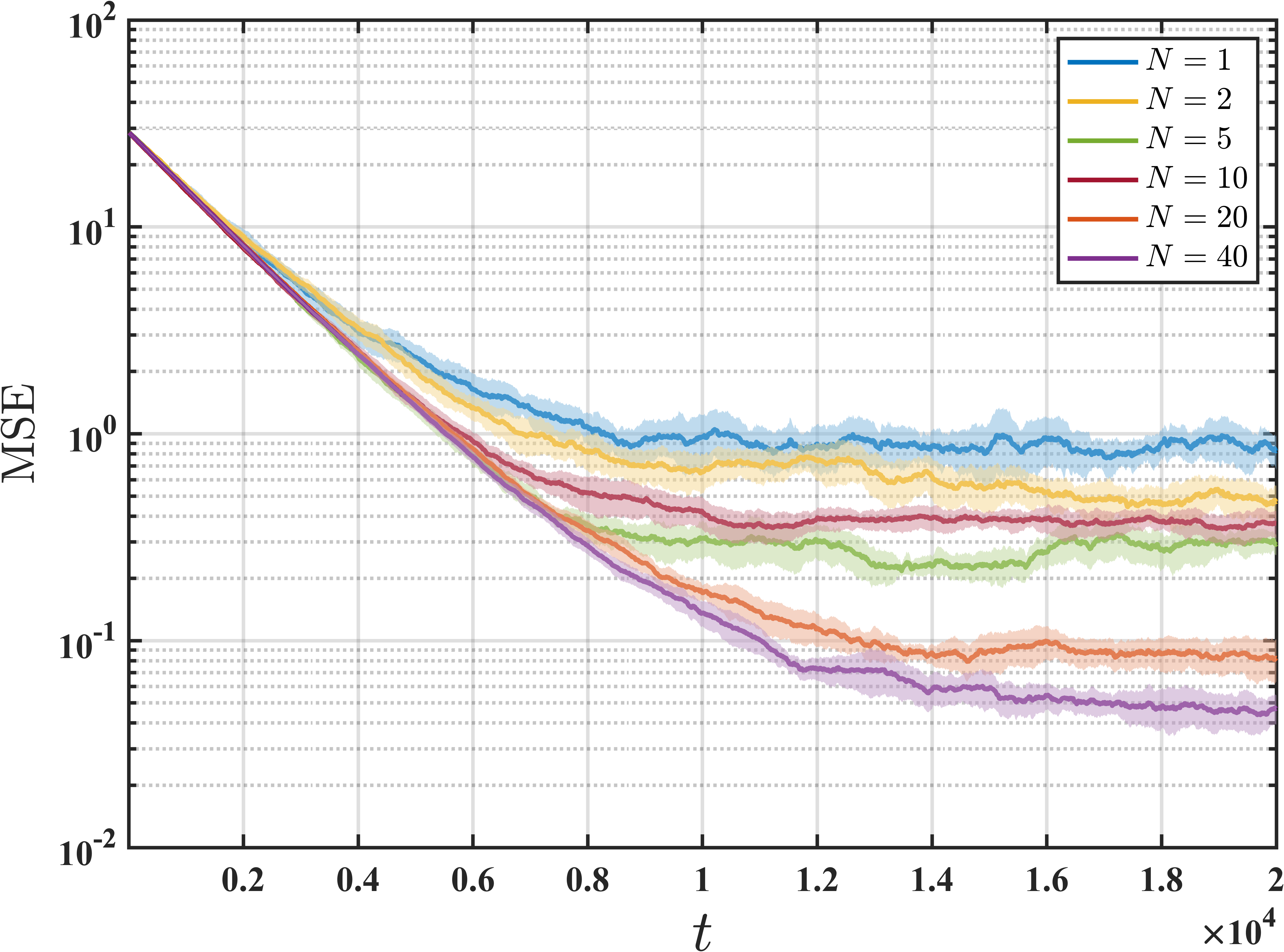}\vspace{-5pt}

      \caption{$\epsilon_r = 0.5$}
  \end{subfigure}
  \hfill
  \begin{subfigure}[b]{0.325\textwidth}
      \centering
      \includegraphics[width=\textwidth]{fig/e/1.0_1.0.png}\vspace{-5pt}

      \caption{$\epsilon_r = 1$}
  \end{subfigure}
  \hfill
  \begin{subfigure}[b]{0.325\textwidth}
      \centering
      \includegraphics[width=\textwidth]{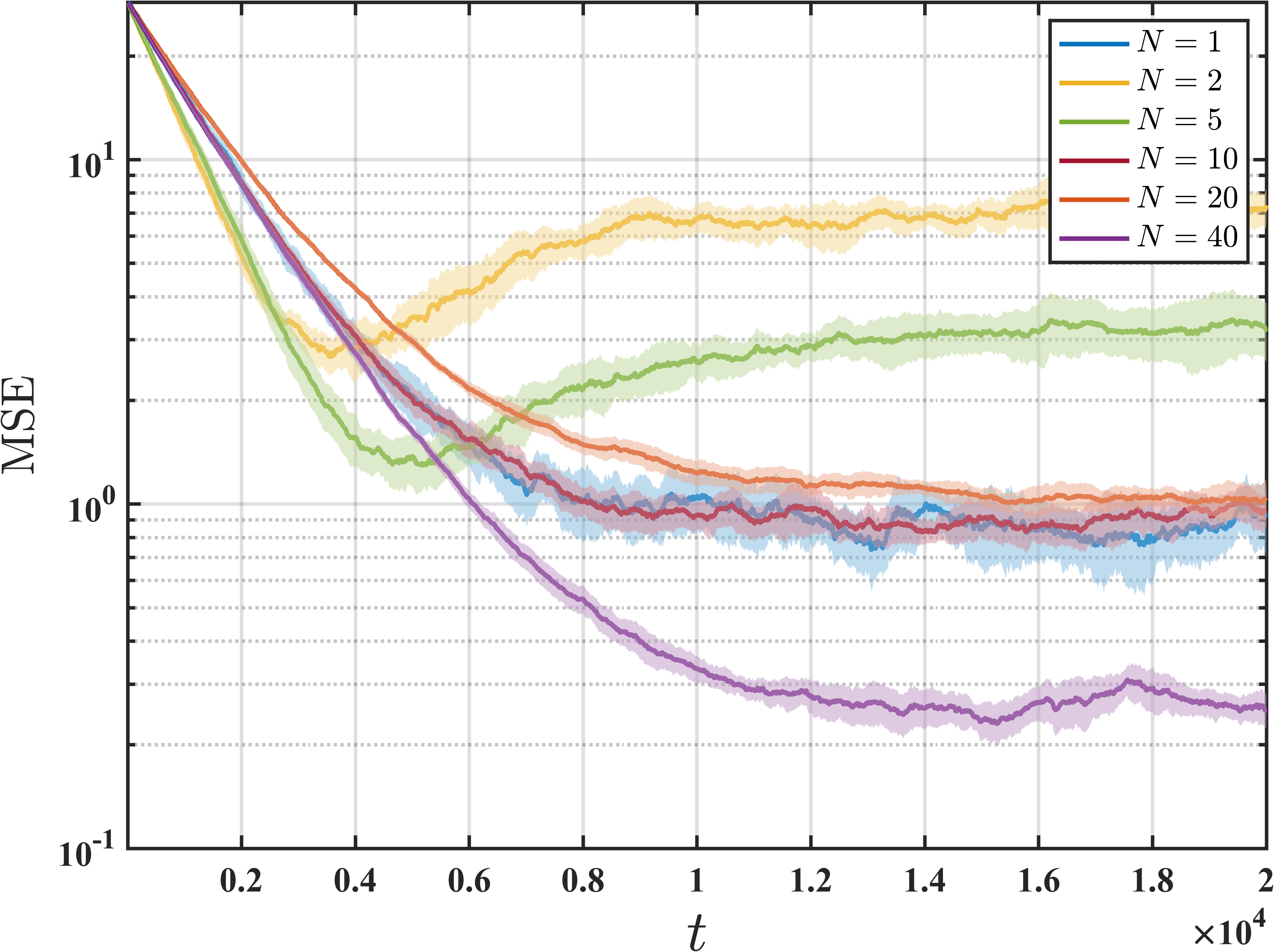}\vspace{-5pt}

      \caption{$\epsilon_r = 2$}
  \end{subfigure}
  \caption{Performance of \fedsarsa under Markovian sampling for varying reward heterogeneity and numbers of agents with fixed kernel heterogeneity (\( \epsilon_p=1 \)).}
  \label{fig:sarsa-ee-p}
\end{figure}

% fig:sarsa-ee-r
\begin{figure}[ht]
  \centering
  \begin{subfigure}{0.325\textwidth}
      \centering
      \includegraphics[width=\textwidth]{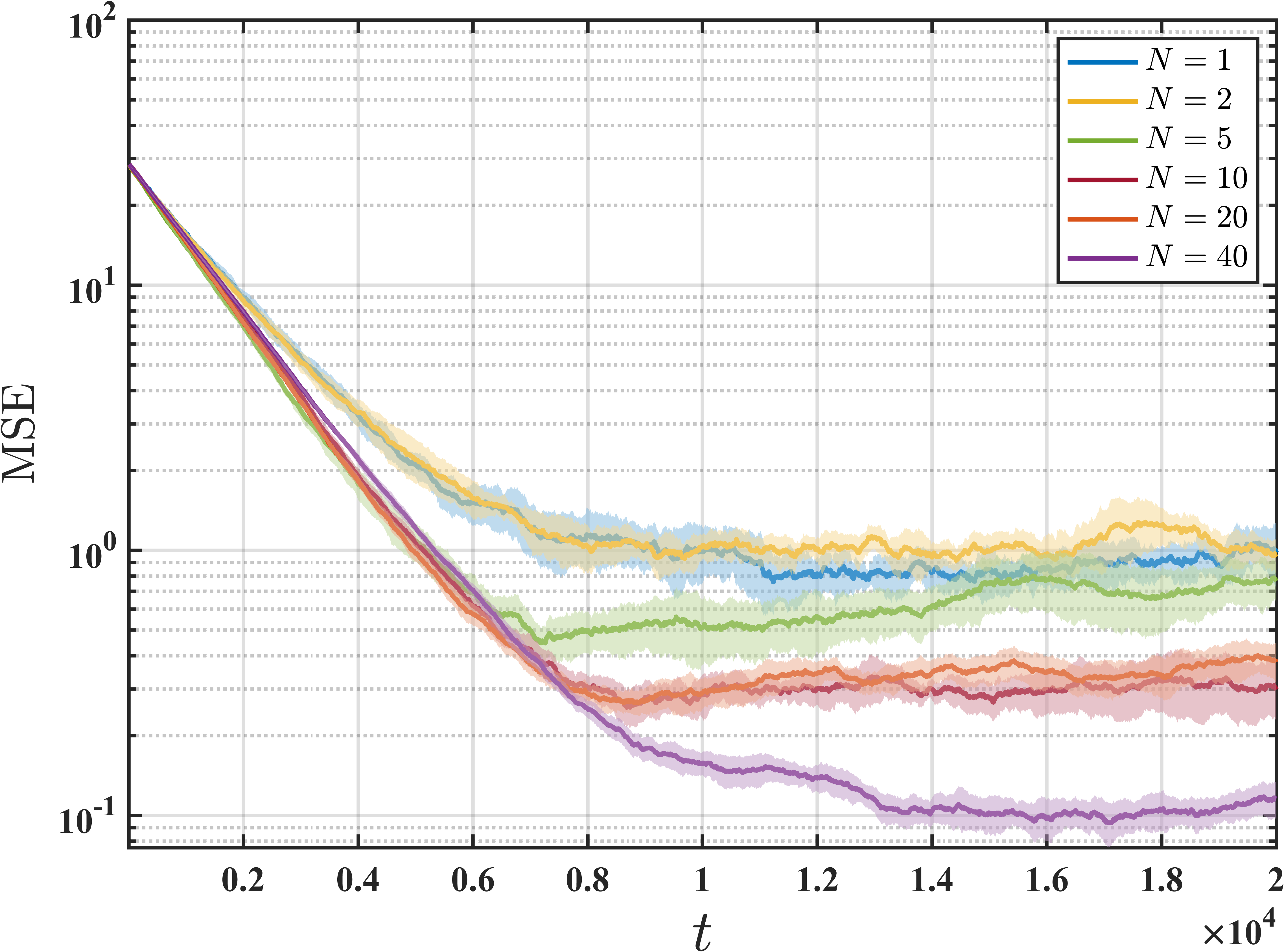}\vspace{-5pt}
      \caption{$\epsilon_p = 0$}
  \end{subfigure}
  \hfill
  \begin{subfigure}[b]{0.325\textwidth}
      \centering
      \includegraphics[width=\textwidth]{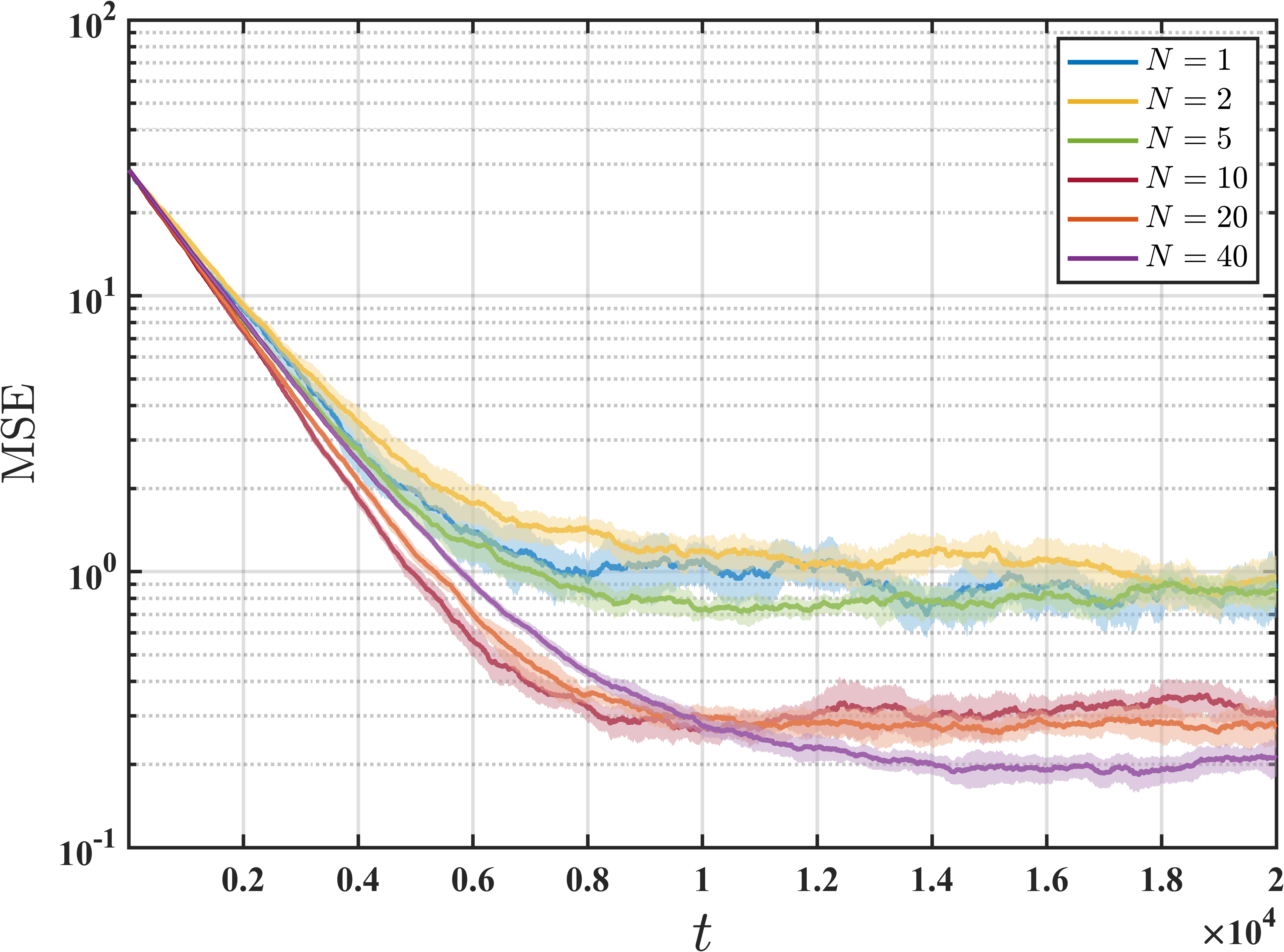}\vspace{-5pt}

      \caption{$\epsilon_p = 0.1$}
  \end{subfigure}
  \hfill
  \begin{subfigure}[b]{0.325\textwidth}
      \centering
      \includegraphics[width=\textwidth]{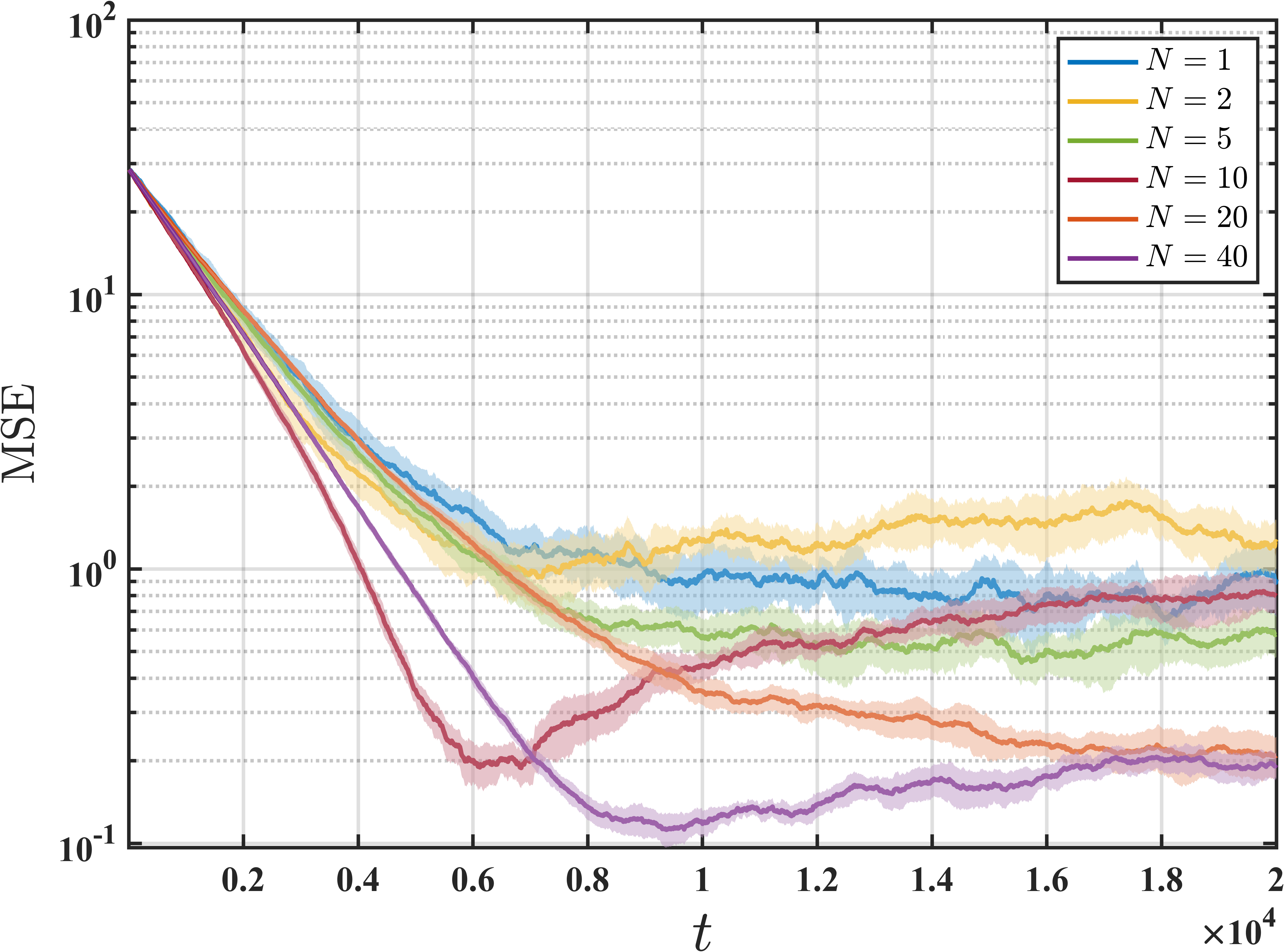}\vspace{-5pt}

      \caption{$\epsilon_p = 0.2$}
  \end{subfigure}
  \begin{subfigure}{0.325\textwidth}
      \centering
      \includegraphics[width=\textwidth]{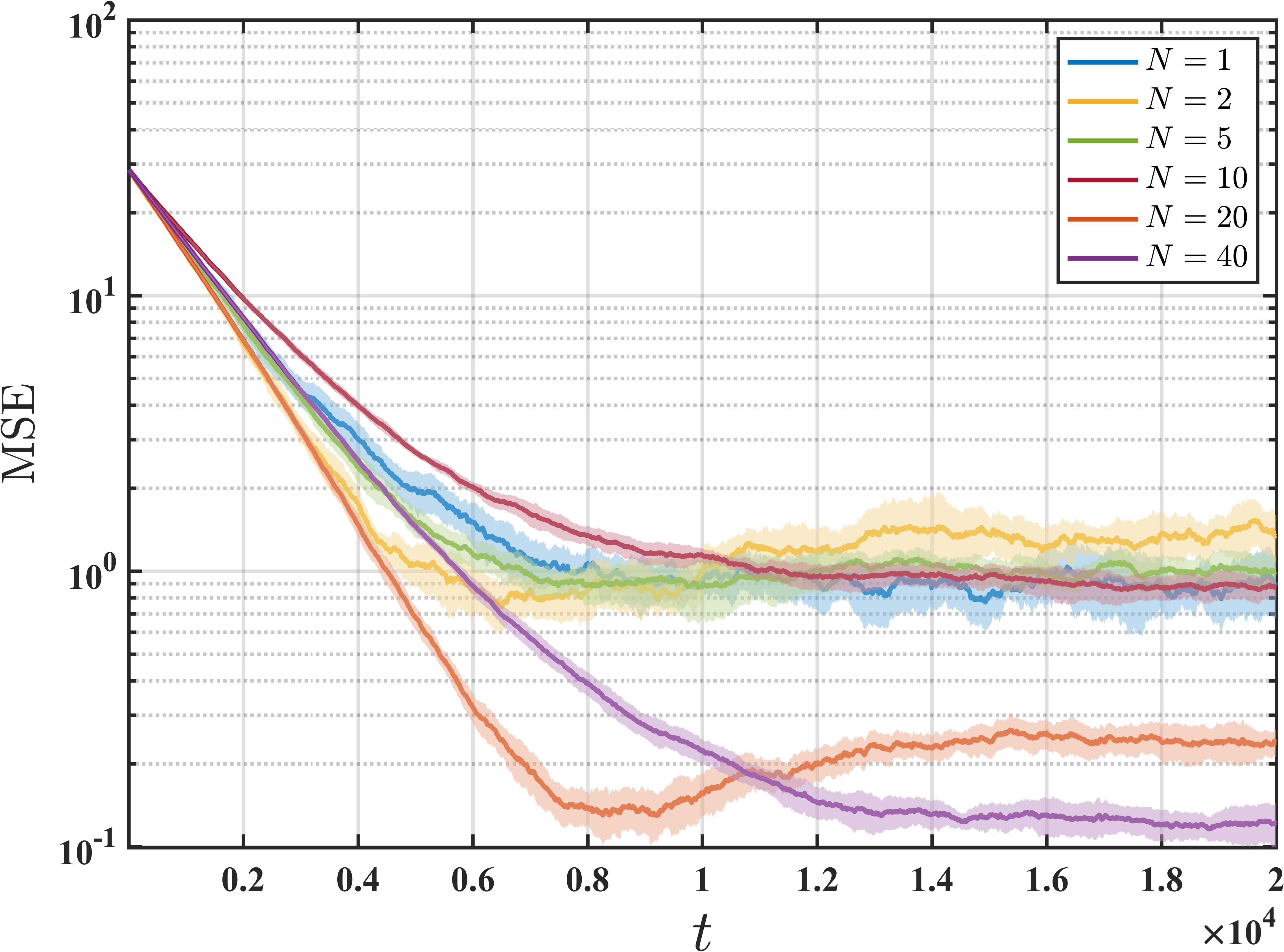}\vspace{-5pt}

      \caption{$\epsilon_p = 0.5$}
  \end{subfigure}
  \hfill
  \begin{subfigure}[b]{0.325\textwidth}
      \centering
      \includegraphics[width=\textwidth]{fig/e/1.0_1.0.png}\vspace{-5pt}

      \caption{$\epsilon_p = 1$}
  \end{subfigure}
  \hfill
  \begin{subfigure}[b]{0.325\textwidth}
      \centering
      \includegraphics[width=\textwidth]{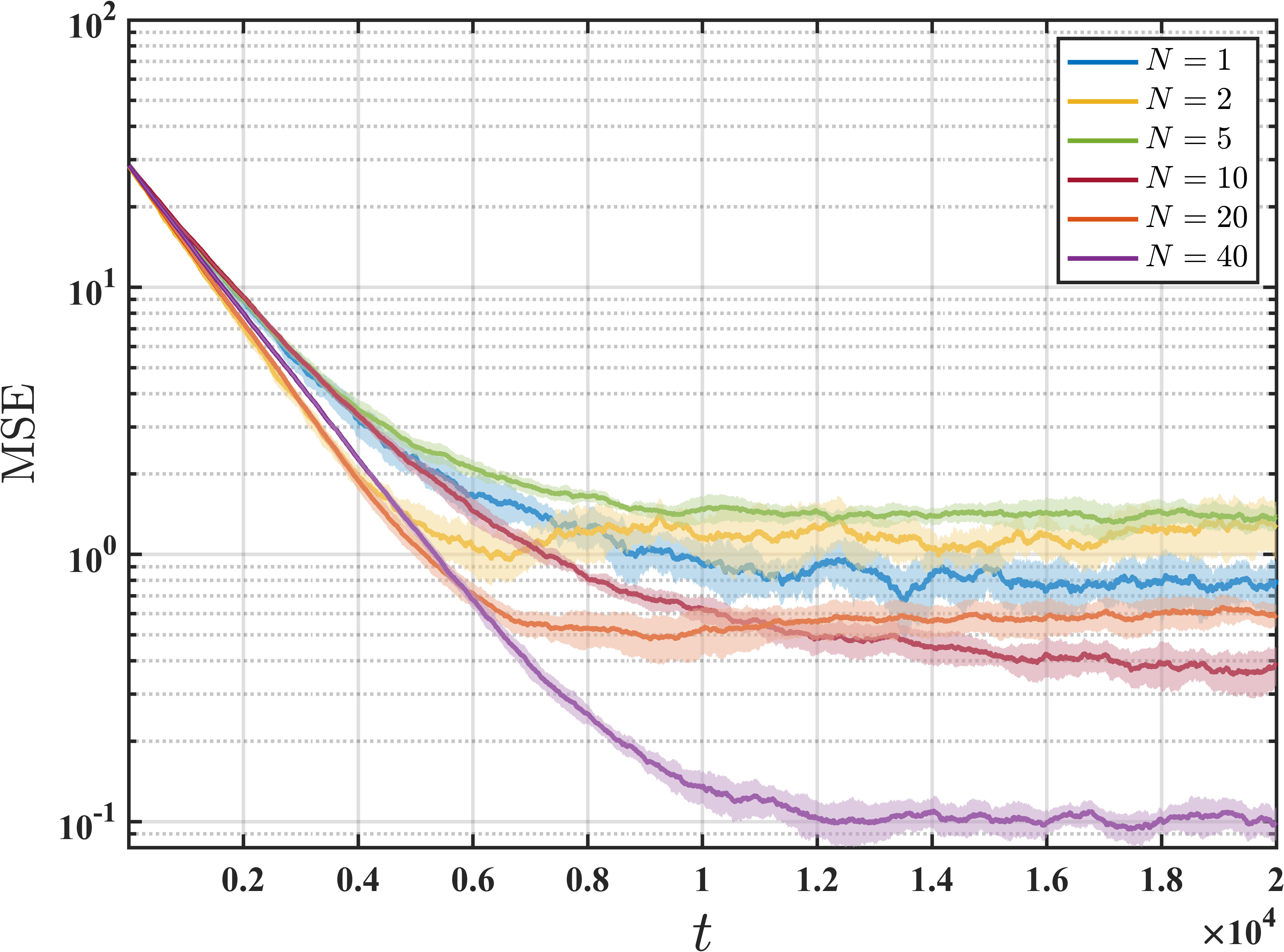}\vspace{-5pt}

      \caption{$\epsilon_p = 2$}
  \end{subfigure}
  \caption{Performance of \fedsarsa under Markovian sampling for varying kernel heterogeneity and numbers of agents with fixed reward heterogeneity (\( \epsilon_r=1 \)).}
  \label{fig:sarsa-ee-r}
\end{figure}

% fig:sarsa-np
\begin{figure}[ht]
  \centering
  \begin{subfigure}{0.325\textwidth}
      \centering
      \includegraphics[width=\textwidth]{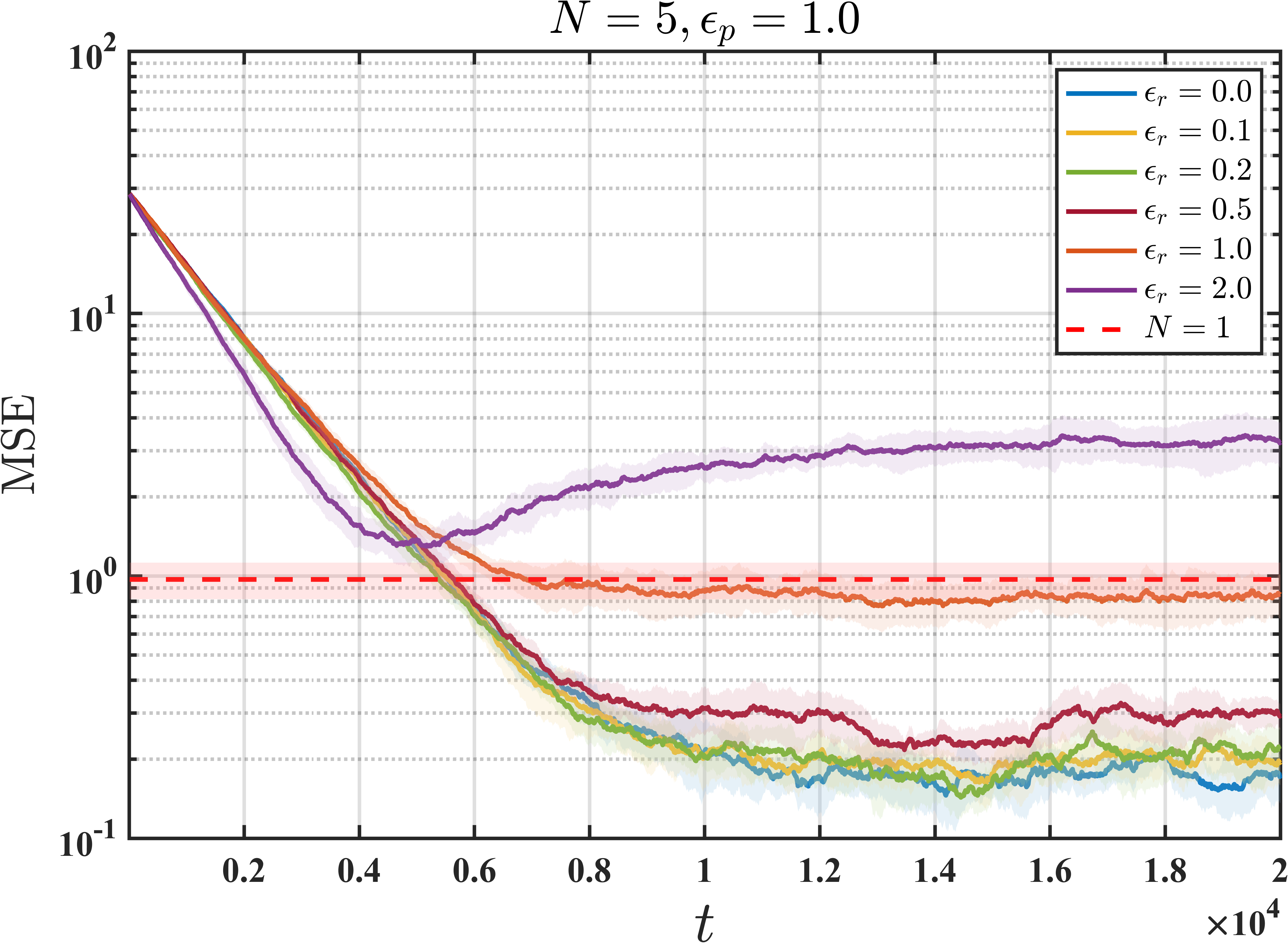}\vspace{-5pt}
      \caption{$N = 5, \epsilon_p = 1.0$}
  \end{subfigure}
  \hfill
  \begin{subfigure}[b]{0.325\textwidth}
      \centering
      \includegraphics[width=\textwidth]{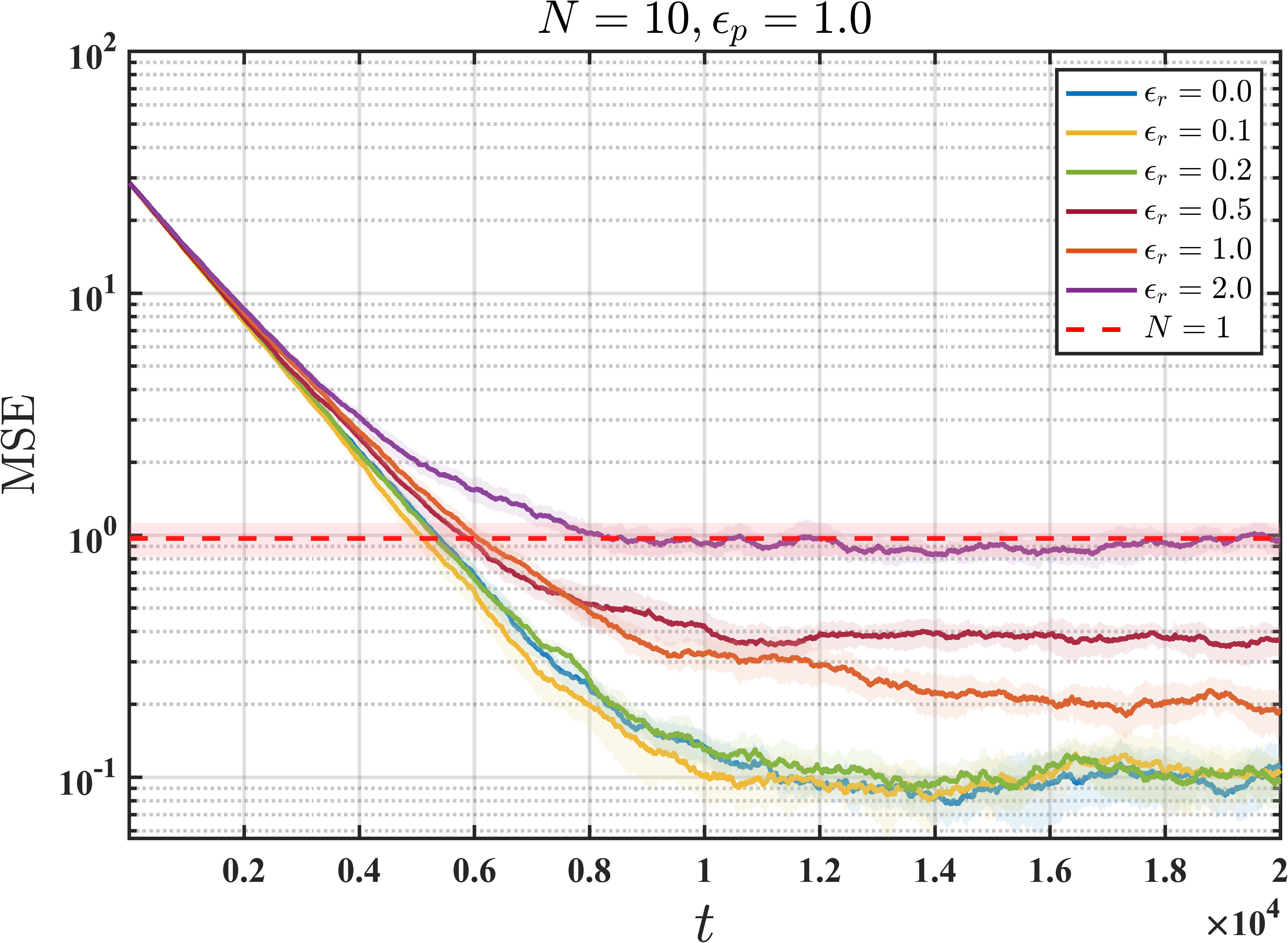}\vspace{-5pt}

      \caption{$N = 10, \epsilon_p = 1.0$}
  \end{subfigure}
  \hfill
  \begin{subfigure}[b]{0.325\textwidth}
      \centering
      \includegraphics[width=\textwidth]{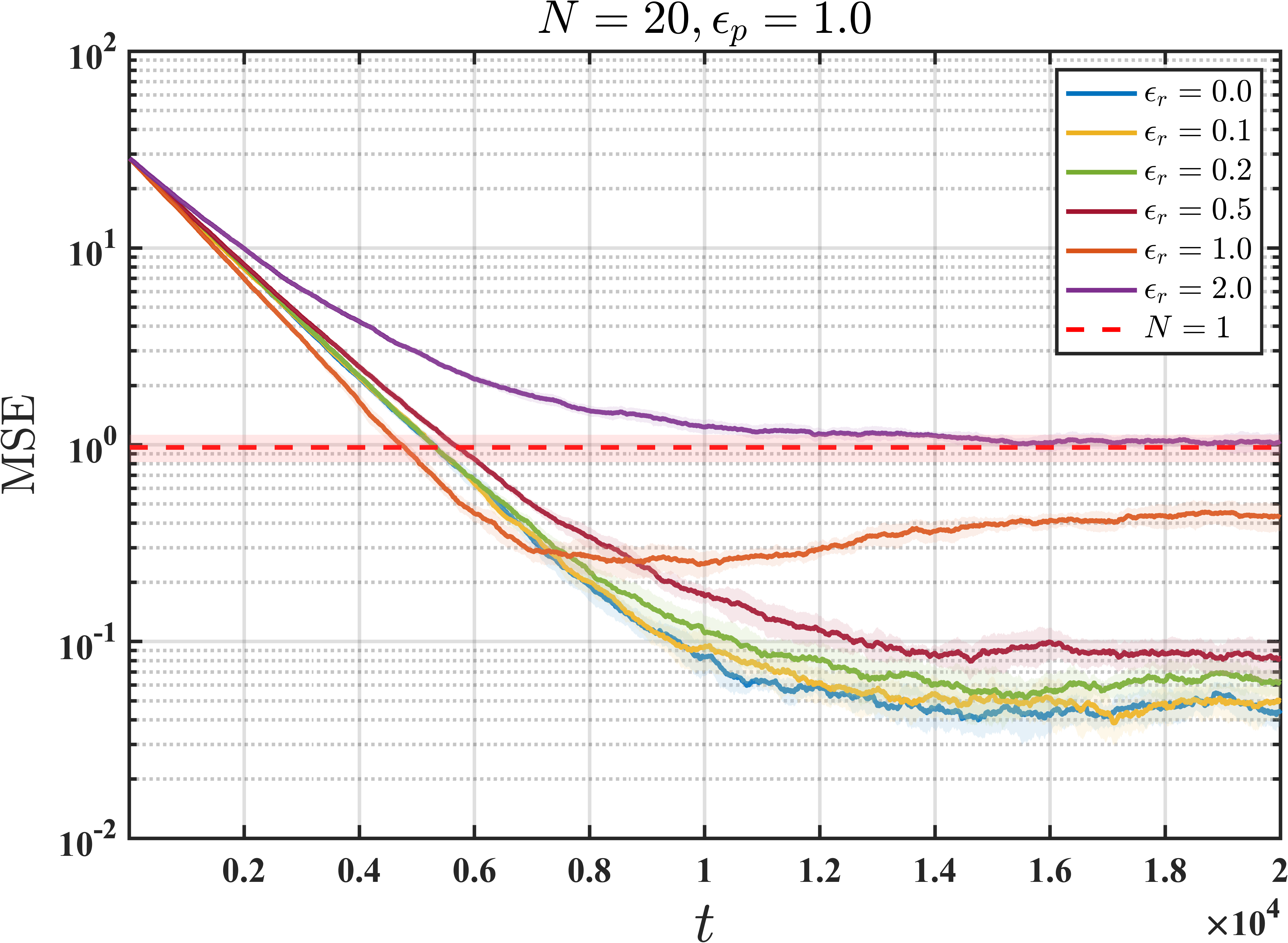}\vspace{-5pt}

      \caption{$N = 20, \epsilon_p = 1.0$}
  \end{subfigure}
  \begin{subfigure}{0.325\textwidth}
      \centering
      \includegraphics[width=\textwidth]{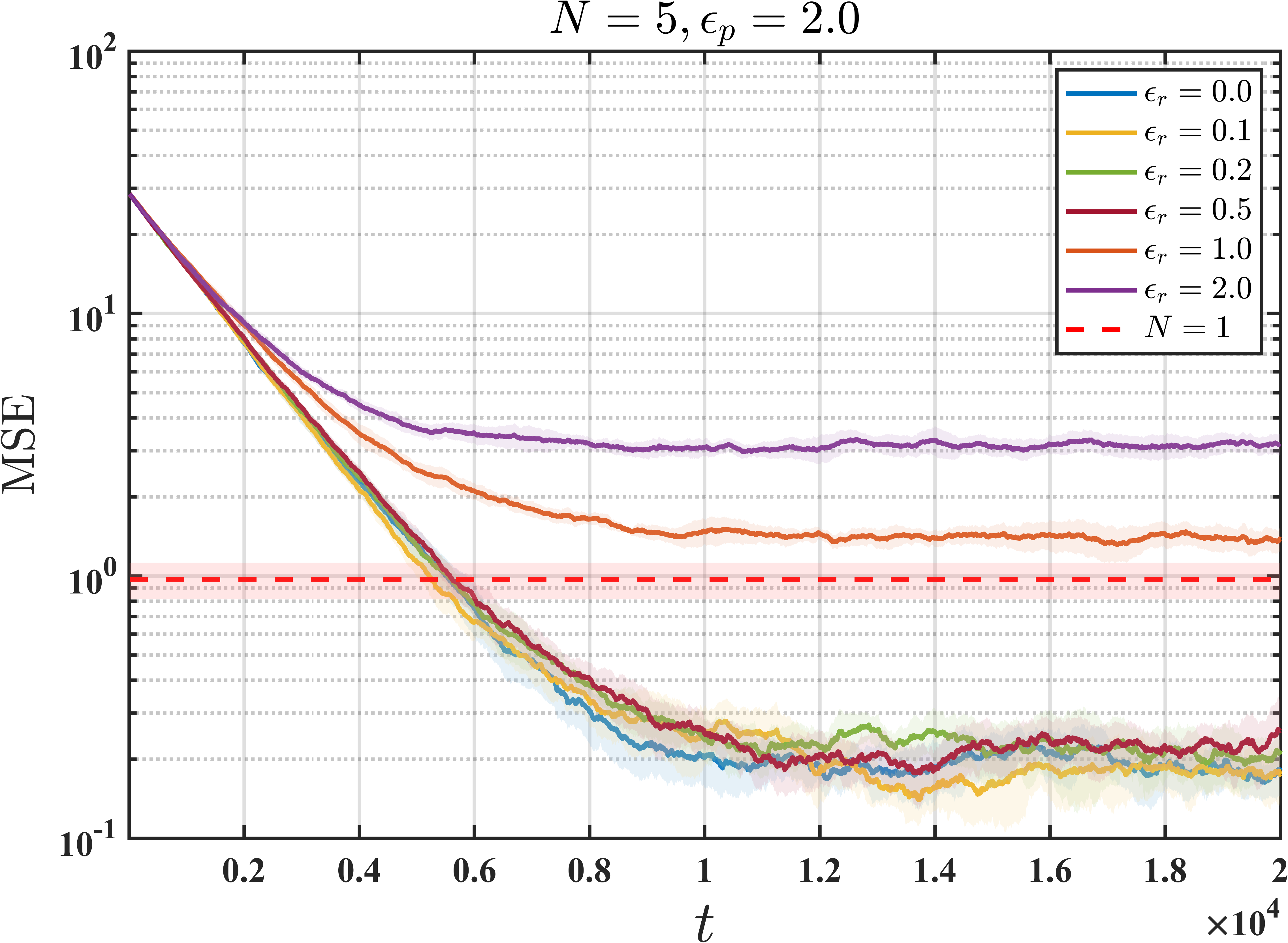}\vspace{-5pt}

      \caption{$N = 5, \epsilon_p = 2.0$}
  \end{subfigure}
  \hfill
  \begin{subfigure}[b]{0.325\textwidth}
      \centering
      \includegraphics[width=\textwidth]{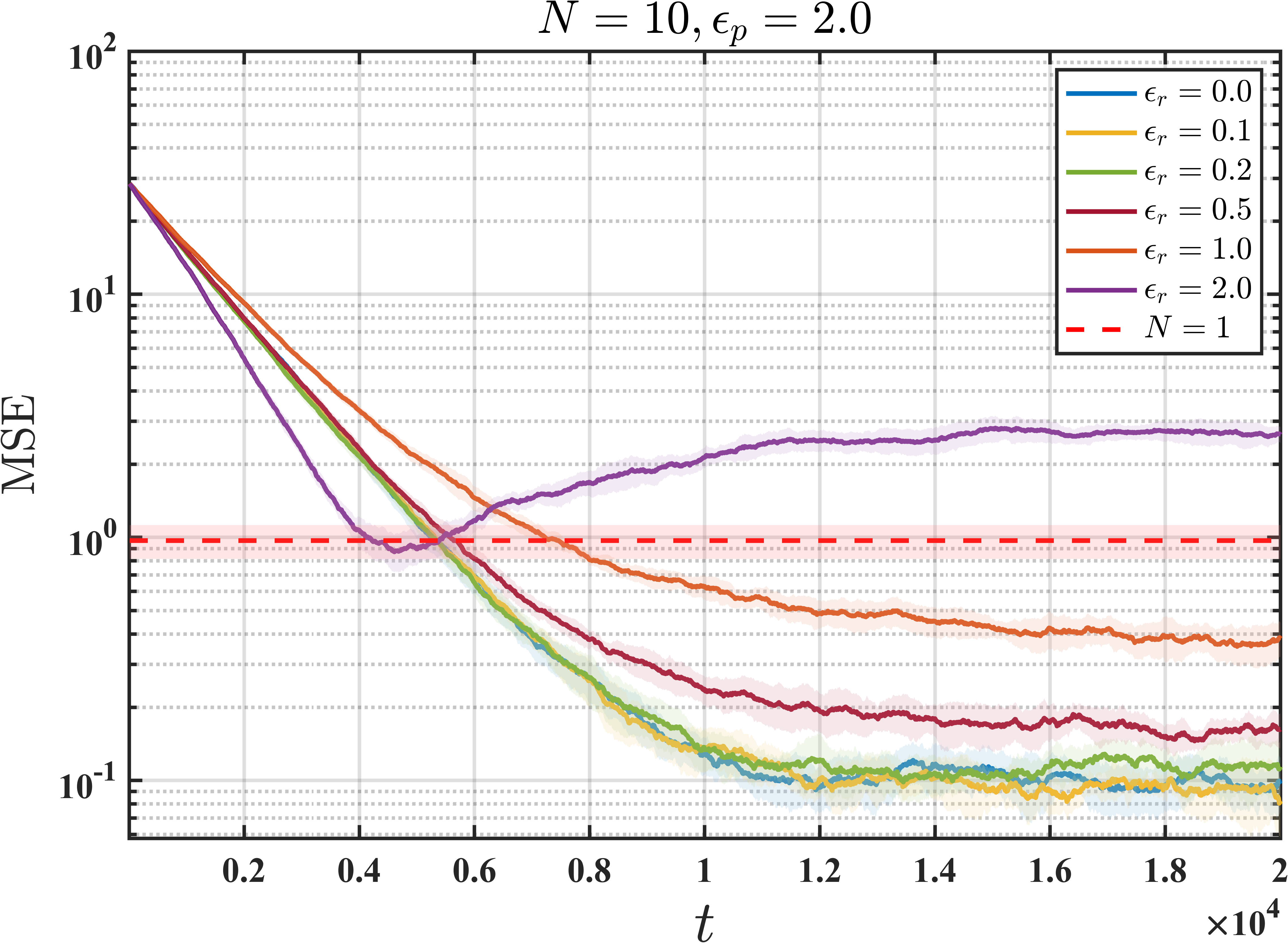}\vspace{-5pt}

      \caption{$N = 10, \epsilon_p = 2.0$}
  \end{subfigure}
  \hfill
  \begin{subfigure}[b]{0.325\textwidth}
      \centering
      \includegraphics[width=\textwidth]{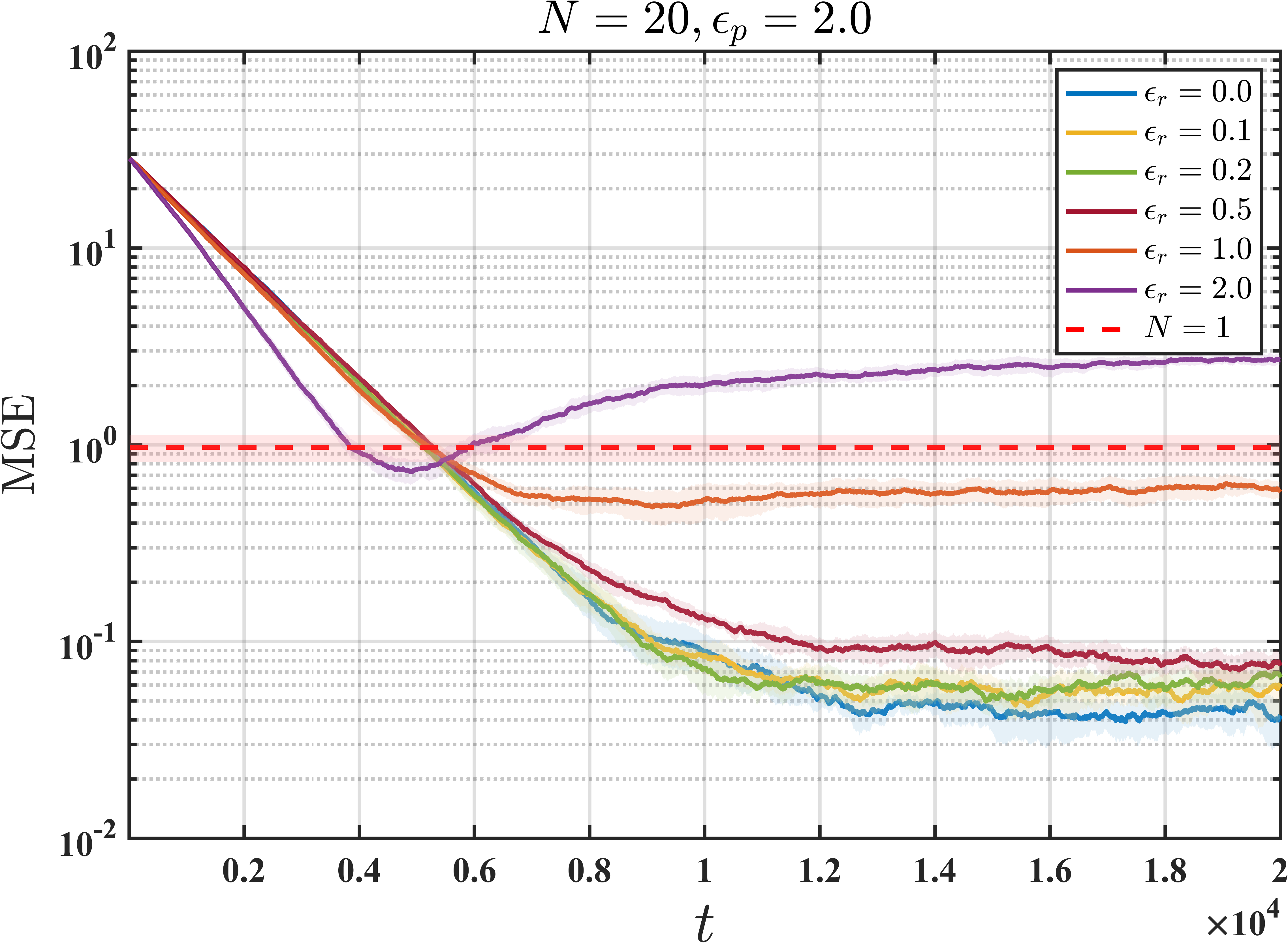}\vspace{-5pt}

      \caption{$N = 20, \epsilon_p = 2.0$}
  \end{subfigure}
  \caption{Effect of the reward heterogeneity on the performance of \fedsarsa.}
  \label{fig:sarsa-np}
\end{figure}

% fig:sarsa-nr
\begin{figure}[ht]
  \centering
  \begin{subfigure}{0.325\textwidth}
      \centering
      \includegraphics[width=\textwidth]{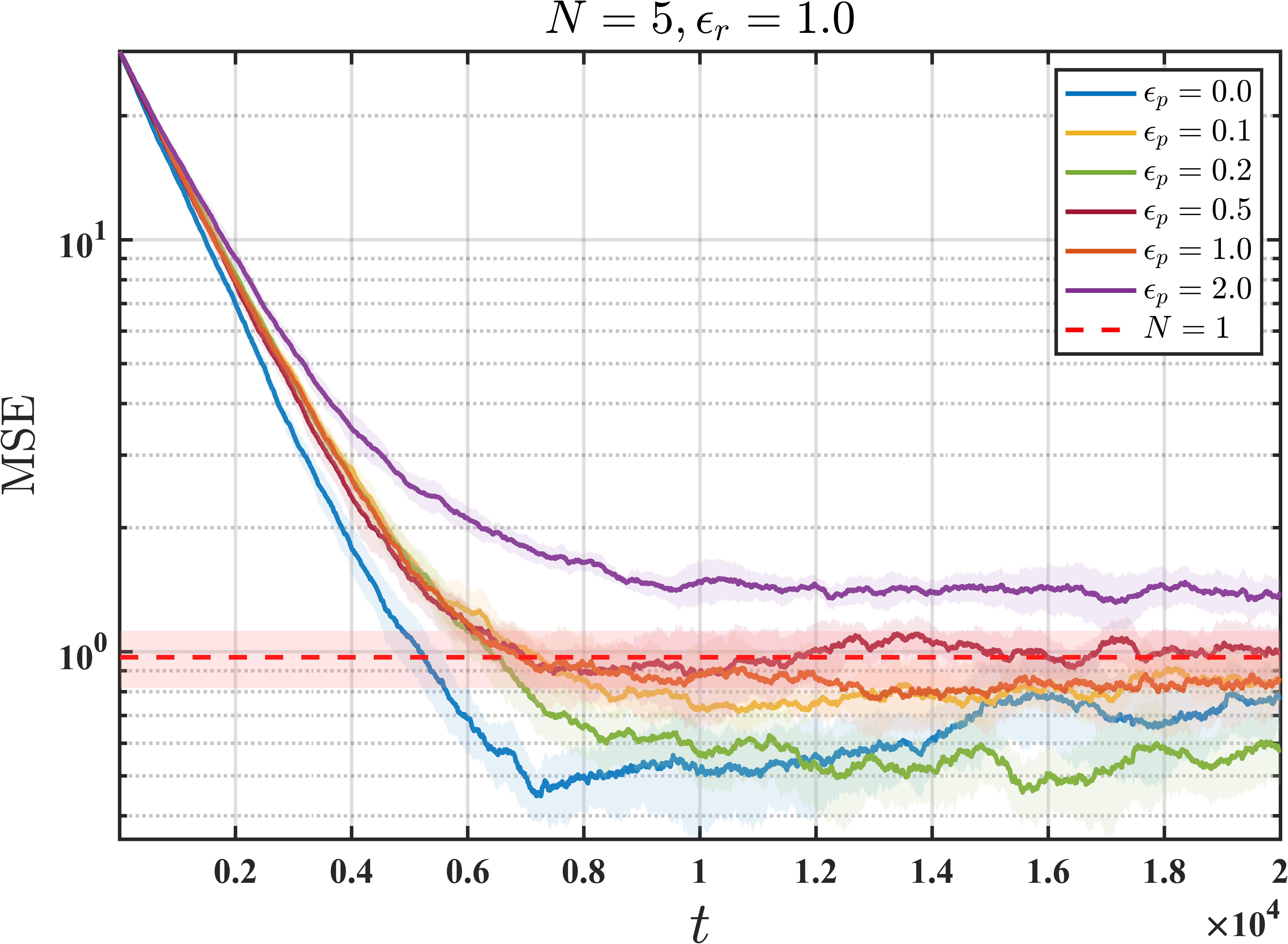}\vspace{-5pt}
      \caption{$N=5,\epsilon_r = 1.0$}
  \end{subfigure}
  \hfill
  \begin{subfigure}[b]{0.325\textwidth}
      \centering
      \includegraphics[width=\textwidth]{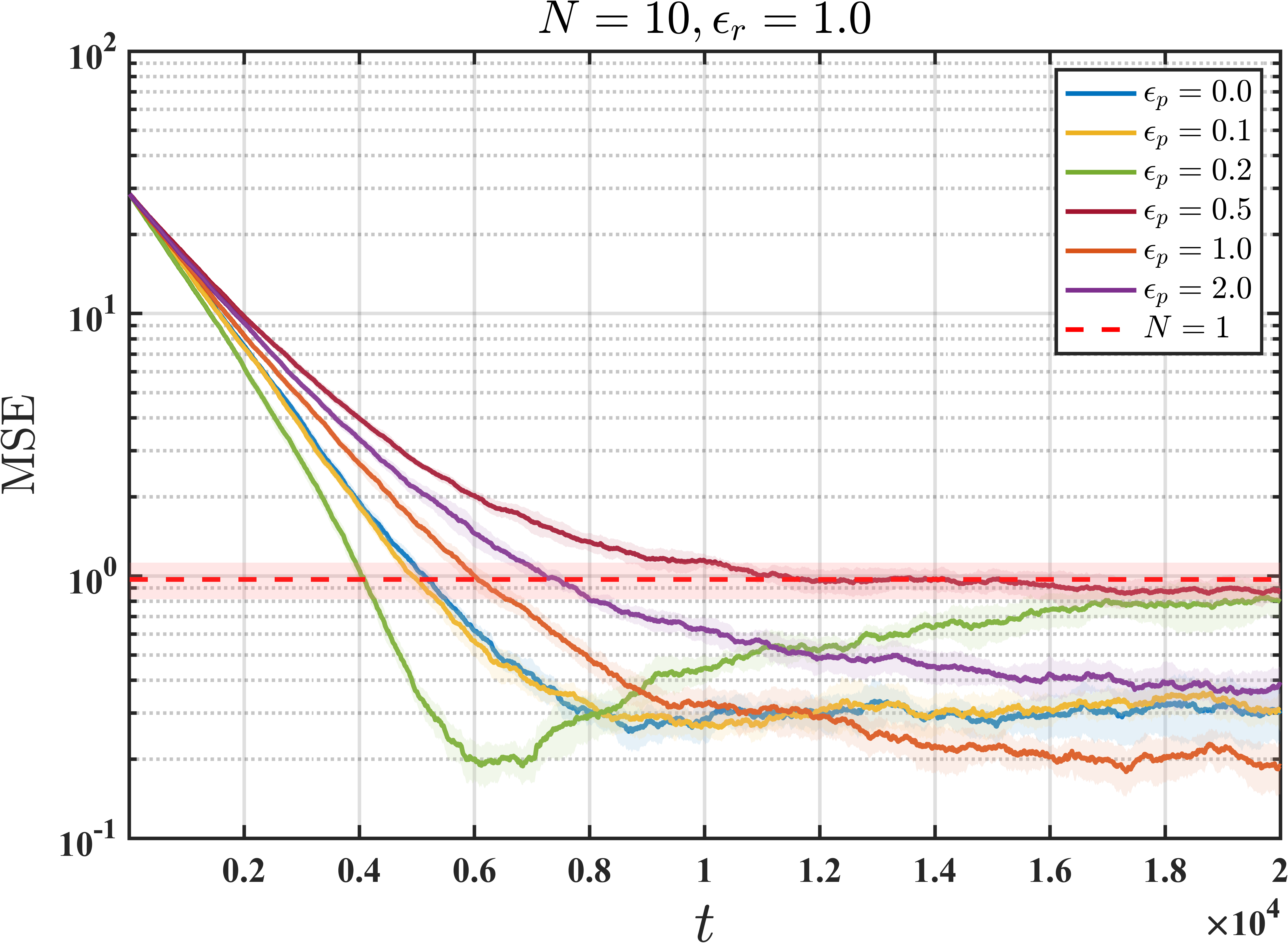}\vspace{-5pt}

      \caption{$N=10,\epsilon_p = 1.0$}
  \end{subfigure}
  \hfill
  \begin{subfigure}[b]{0.325\textwidth}
      \centering
      \includegraphics[width=\textwidth]{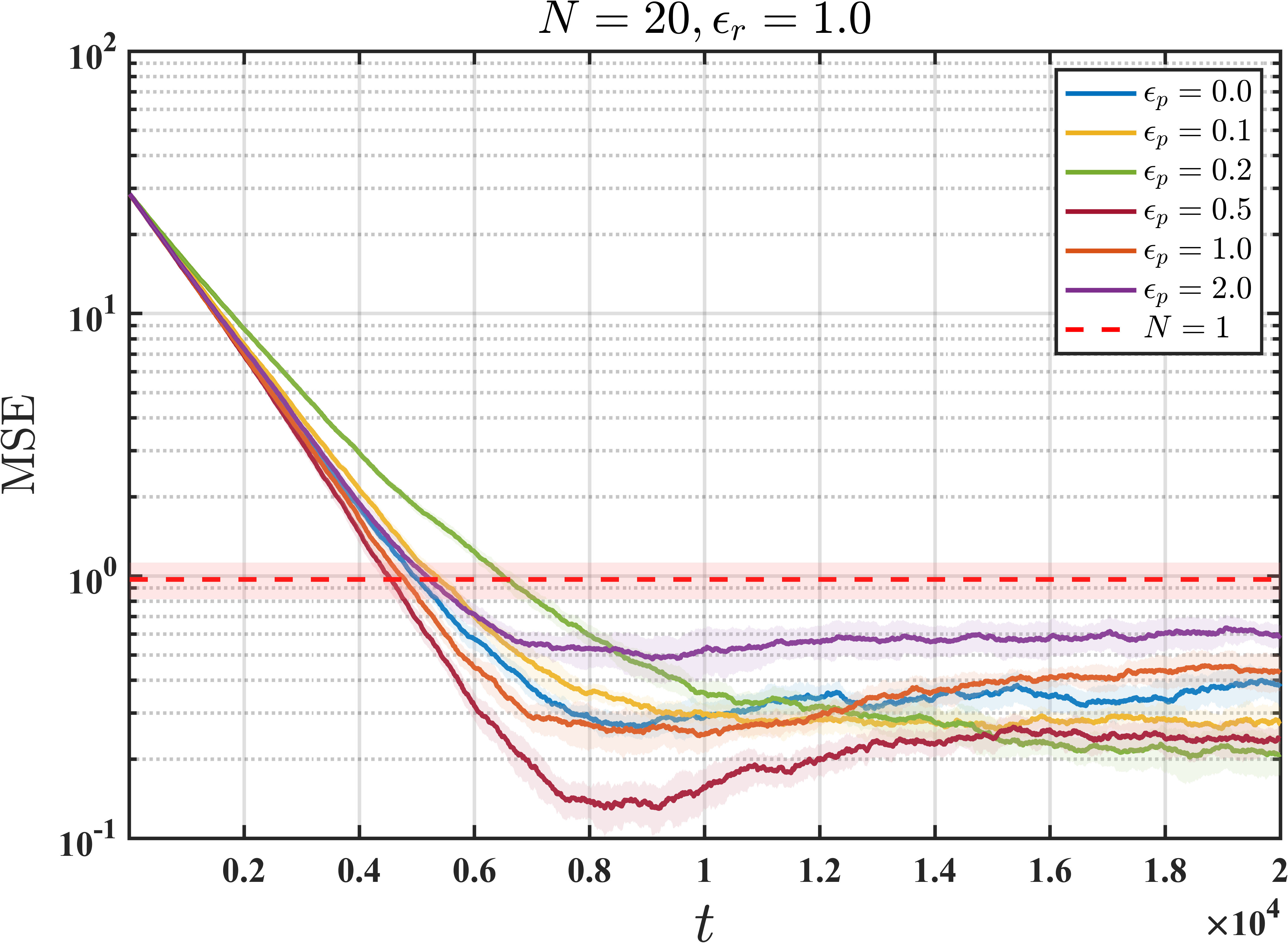}\vspace{-5pt}

      \caption{$\epsilon_p = 0.2$}
  \end{subfigure}
  \begin{subfigure}{0.325\textwidth}
      \centering
      \includegraphics[width=\textwidth]{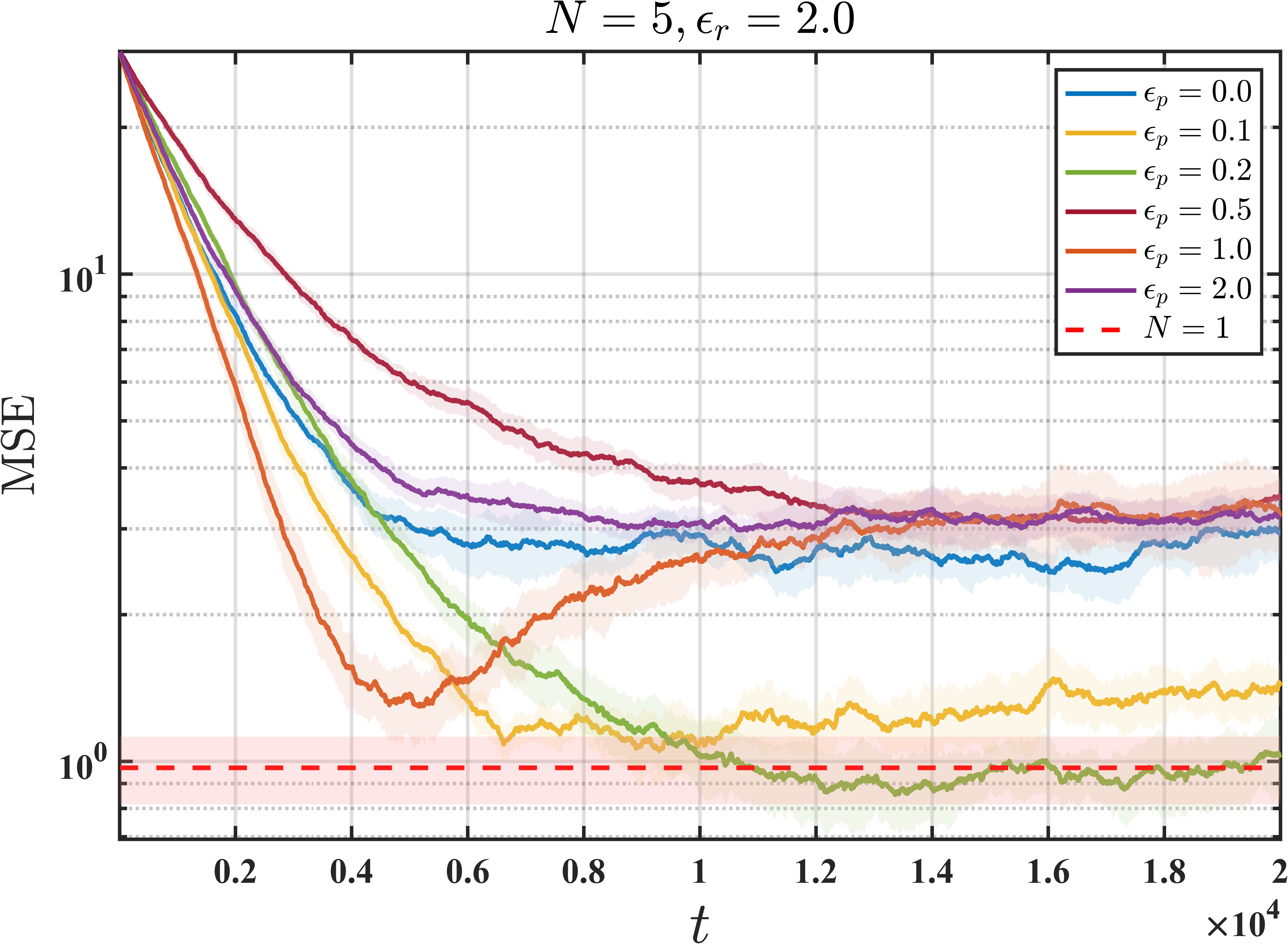}\vspace{-5pt}

      \caption{$N=20,\epsilon_r = 1.0$}
  \end{subfigure}
  \hfill
  \begin{subfigure}[b]{0.325\textwidth}
      \centering
      \includegraphics[width=\textwidth]{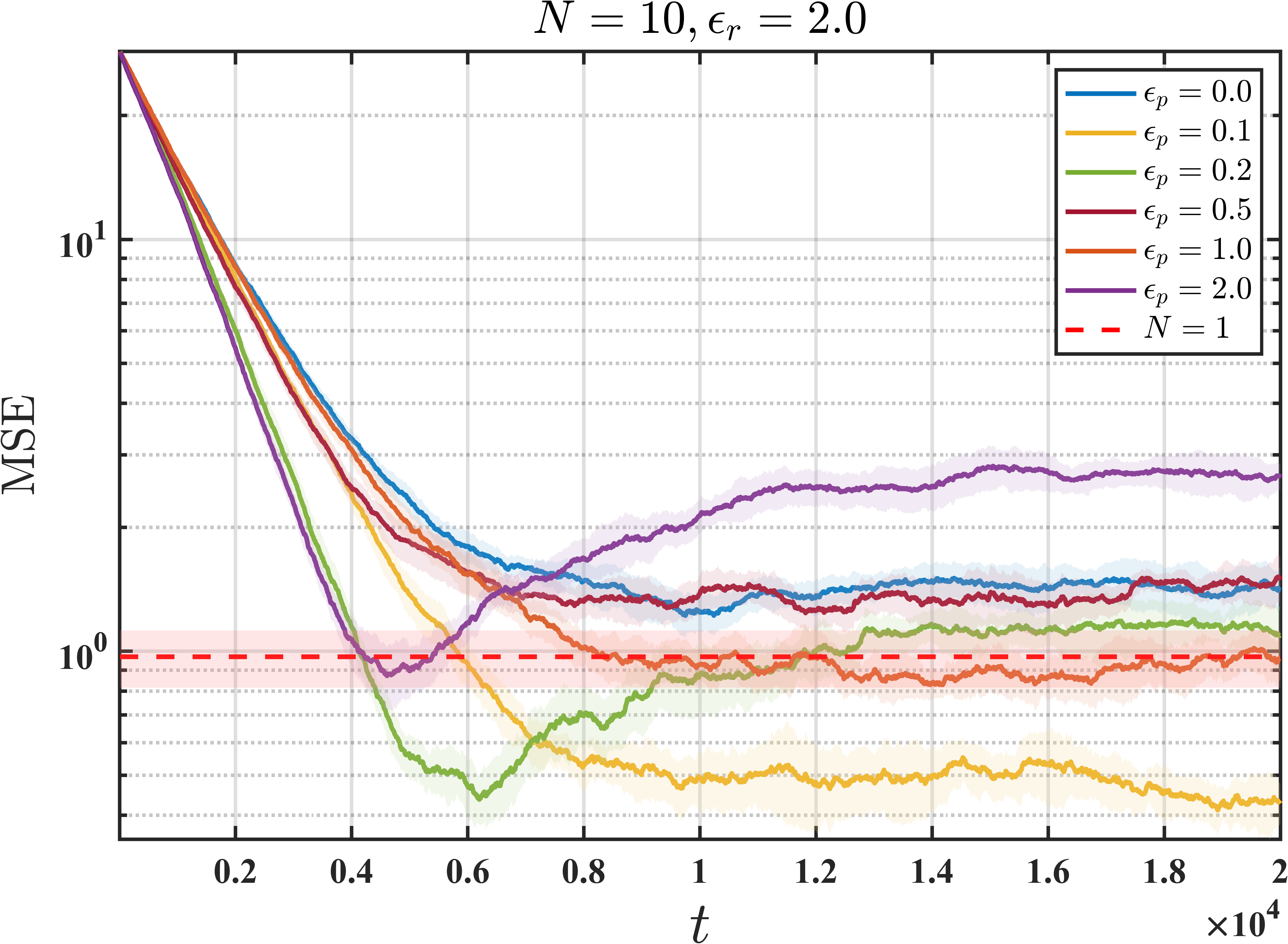}\vspace{-5pt}

      \caption{$N=5,\epsilon_r = 2.0$}
  \end{subfigure}
  \hfill
  \begin{subfigure}[b]{0.325\textwidth}
      \centering
      \includegraphics[width=\textwidth]{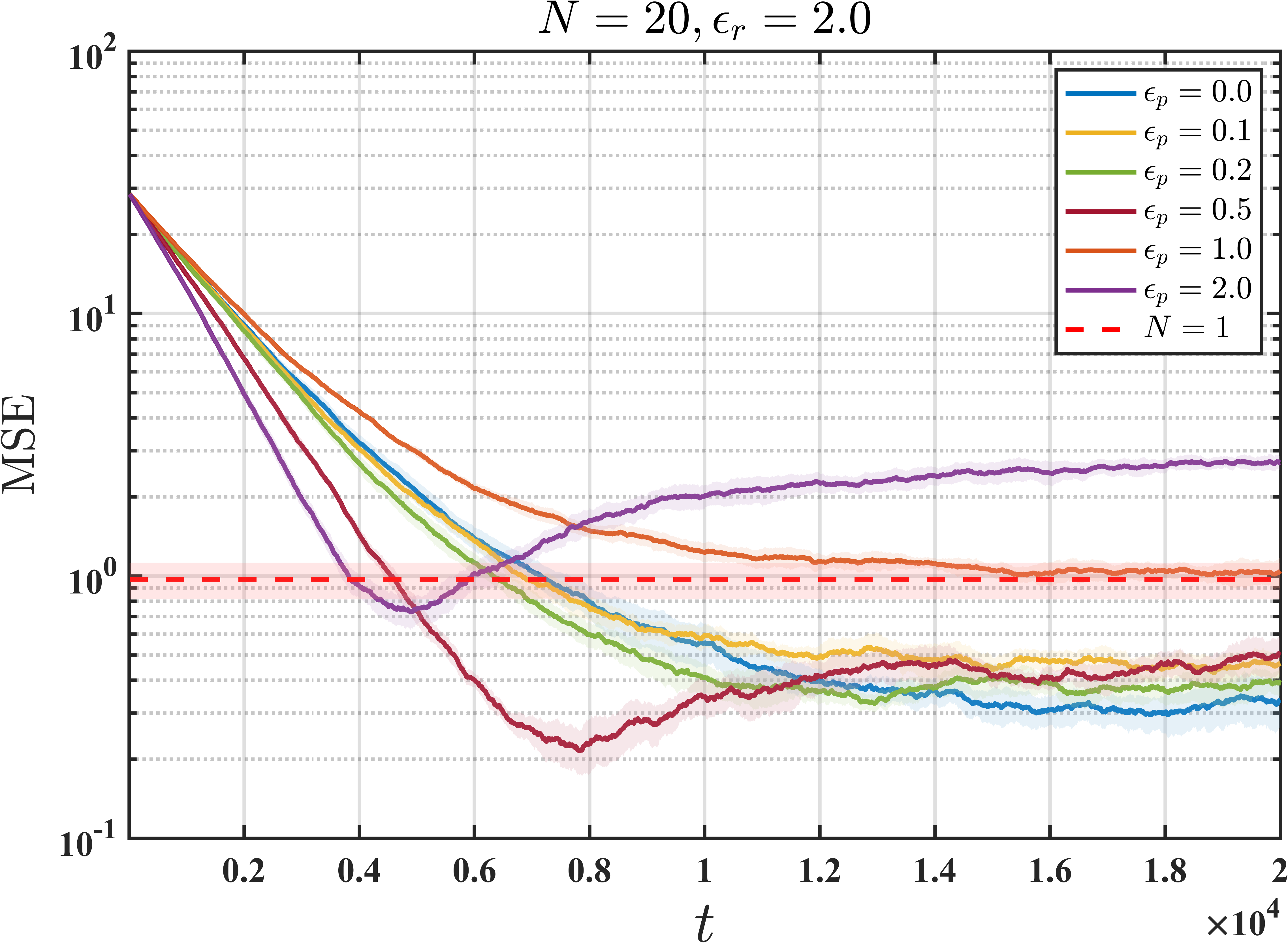}\vspace{-5pt}

      \caption{$N=10, \epsilon_r = 2.0$}
  \end{subfigure}
  \caption{Effect of the kernel heterogeneity on the performance of \fedsarsa.}
  \label{fig:sarsa-nr}
\end{figure}

\subsection{Simulations for Federated TD(0)} \label{sec:exp-td}

As discussed in \cref{sec:alg-local}, \fedsarsa reduces to federated TD(0) \citep{wang2023FederatedTemporal} when the policy improvement operator maps any parameter to a fixed policy \( \pi \).
This corresponds to a fixed transition kernel.
Therefore, we conduct simulations for federated TD(0) to demonstrate the adaptability of \fedsarsa.
We inherit the simulation setup from the previous subsection (\cref{sec:exp-sarsa}), which matches the setup in \citet{wang2023FederatedTemporal}. We fix the behavior policy by fix the transition matrix as the reference matrix \( P_0 \).
The results are presented in \cref{fig:td}, which are similar to the results in \cref{sec:exp}, again validating our theoretical results.

% fig:td
\begin{figure}[ht]
  \centering
  \begin{subfigure}{0.325\textwidth}
      \centering
      \includegraphics[width=\textwidth]{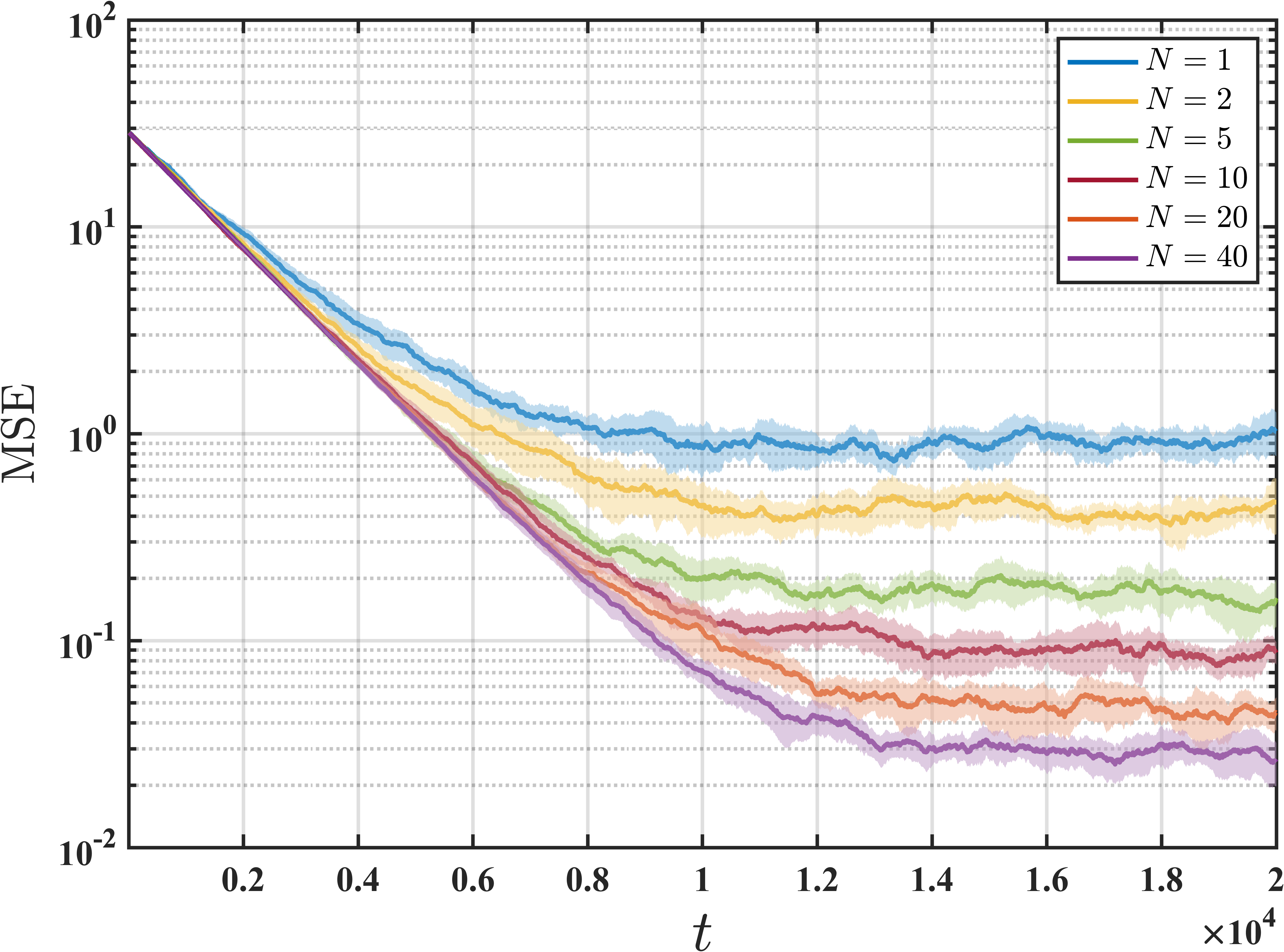}\vspace{-5pt}
      \caption{$\epsilon_p=\epsilon_r = 0$}
  \end{subfigure}
  \hfill
  \begin{subfigure}[b]{0.325\textwidth}
      \centering
      \includegraphics[width=\textwidth]{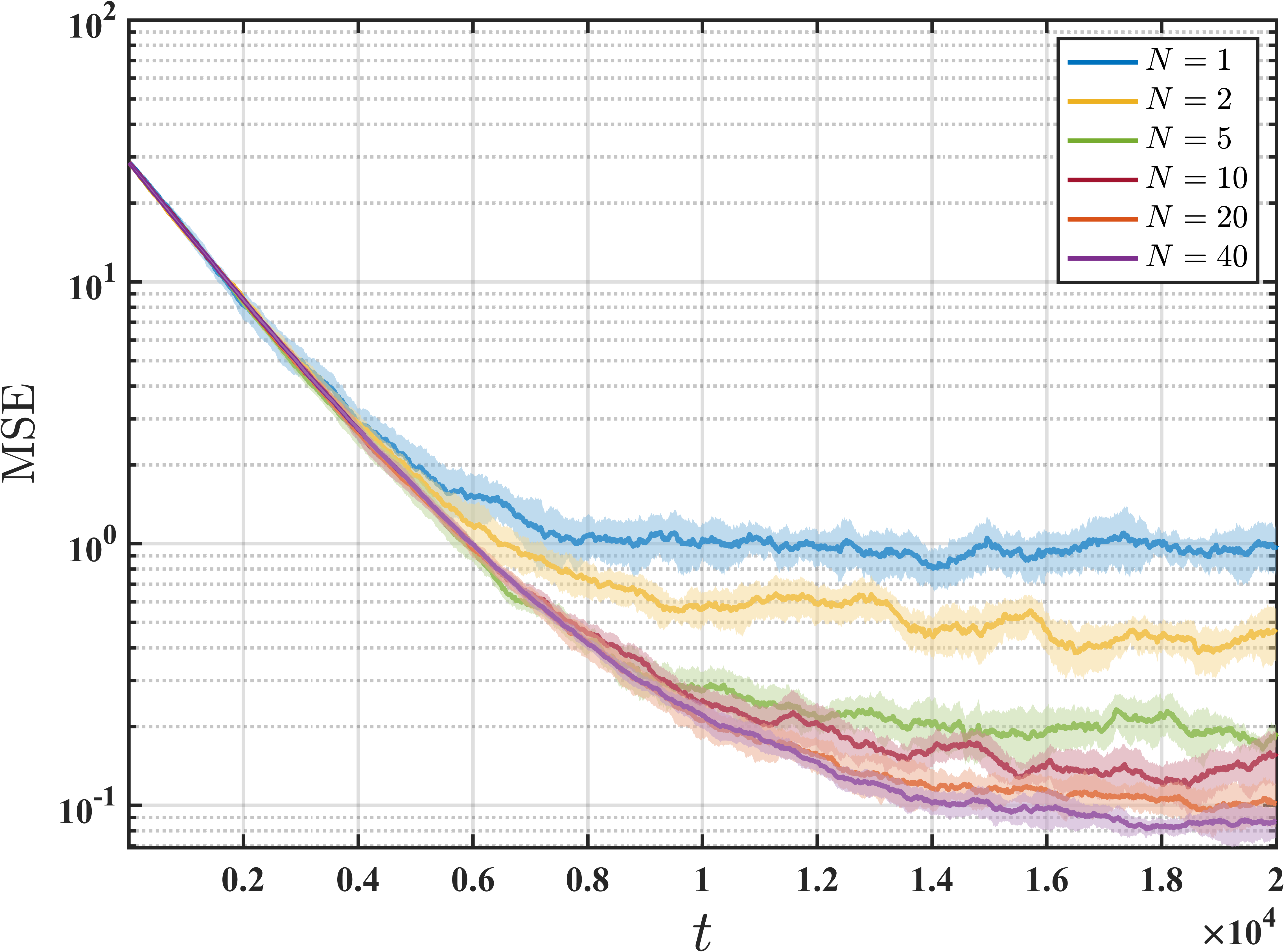}\vspace{-5pt}

      \caption{$\epsilon_p=\epsilon_r = 0.1$}
  \end{subfigure}
  \hfill
  \begin{subfigure}[b]{0.325\textwidth}
      \centering
      \includegraphics[width=\textwidth]{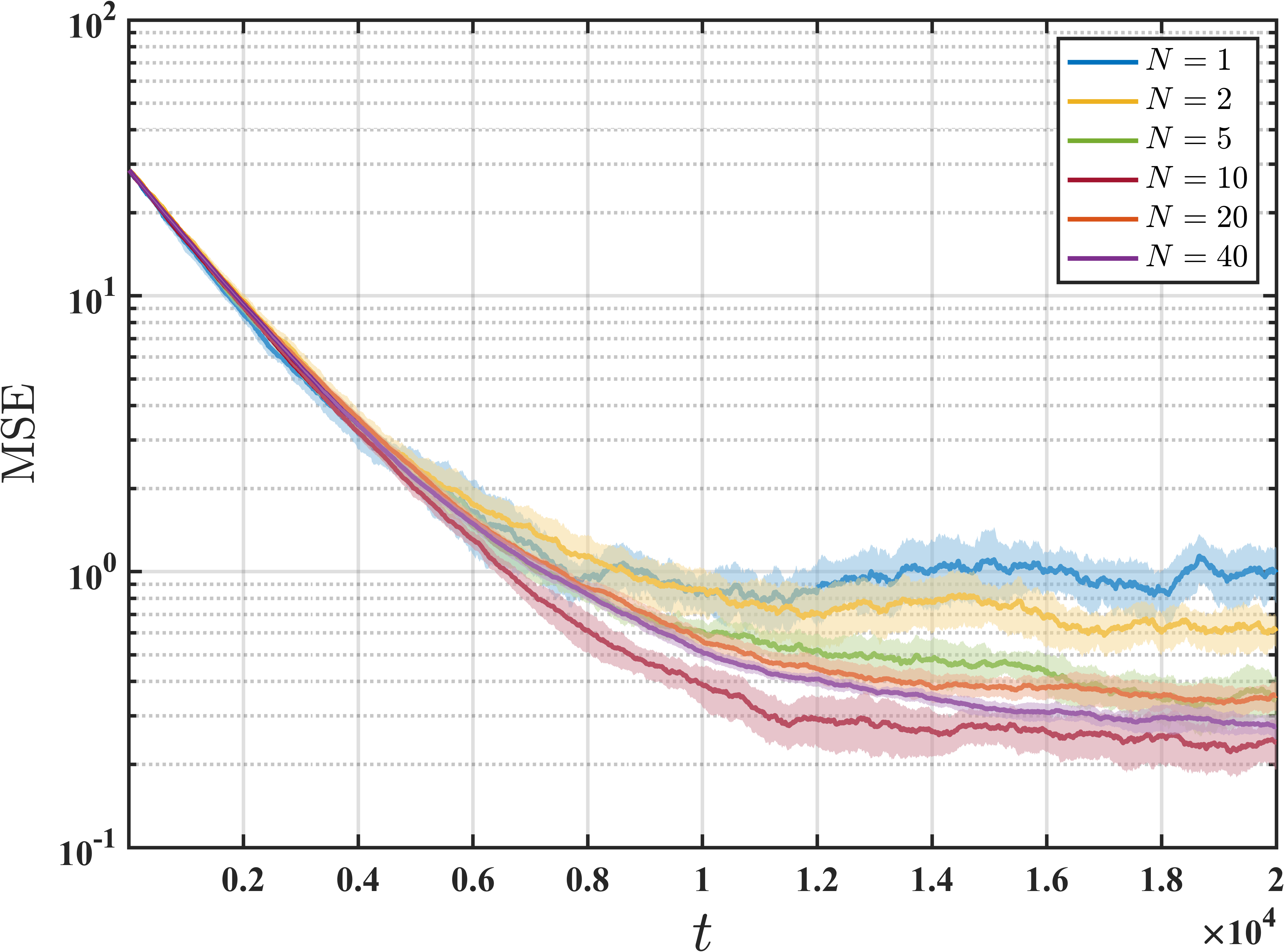}\vspace{-5pt}

      \caption{$\epsilon_p=\epsilon_r = 0.2$}
  \end{subfigure}
  \begin{subfigure}{0.325\textwidth}
      \centering
      \includegraphics[width=\textwidth]{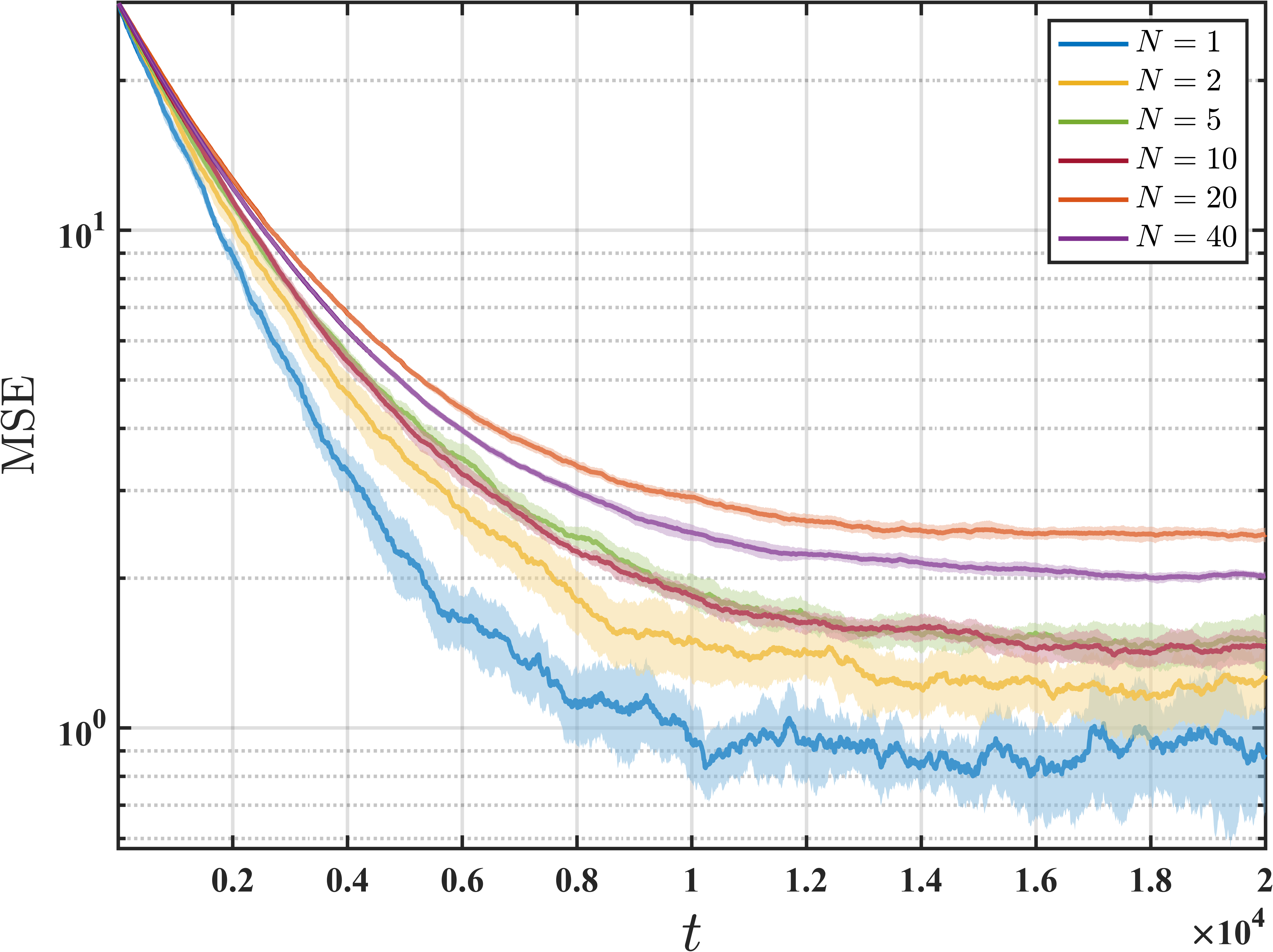}\vspace{-5pt}

      \caption{$\epsilon_p=\epsilon_r = 0.5$}
  \end{subfigure}
  \hfill
  \begin{subfigure}[b]{0.325\textwidth}
      \centering
      \includegraphics[width=\textwidth]{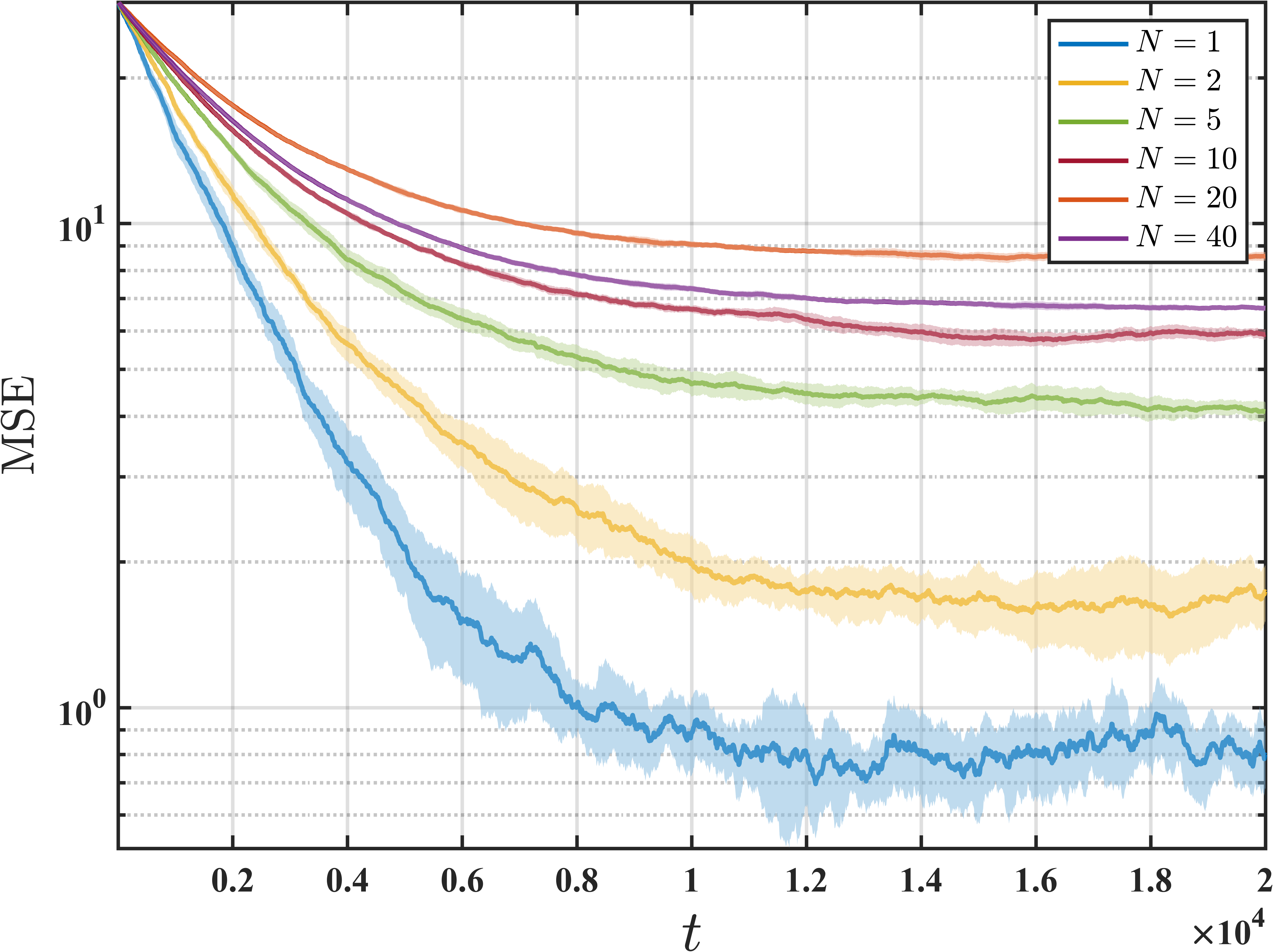}\vspace{-5pt}

      \caption{$\epsilon_p=\epsilon_r = 1$}
  \end{subfigure}
  \hfill
  \begin{subfigure}[b]{0.325\textwidth}
      \centering
      \includegraphics[width=\textwidth]{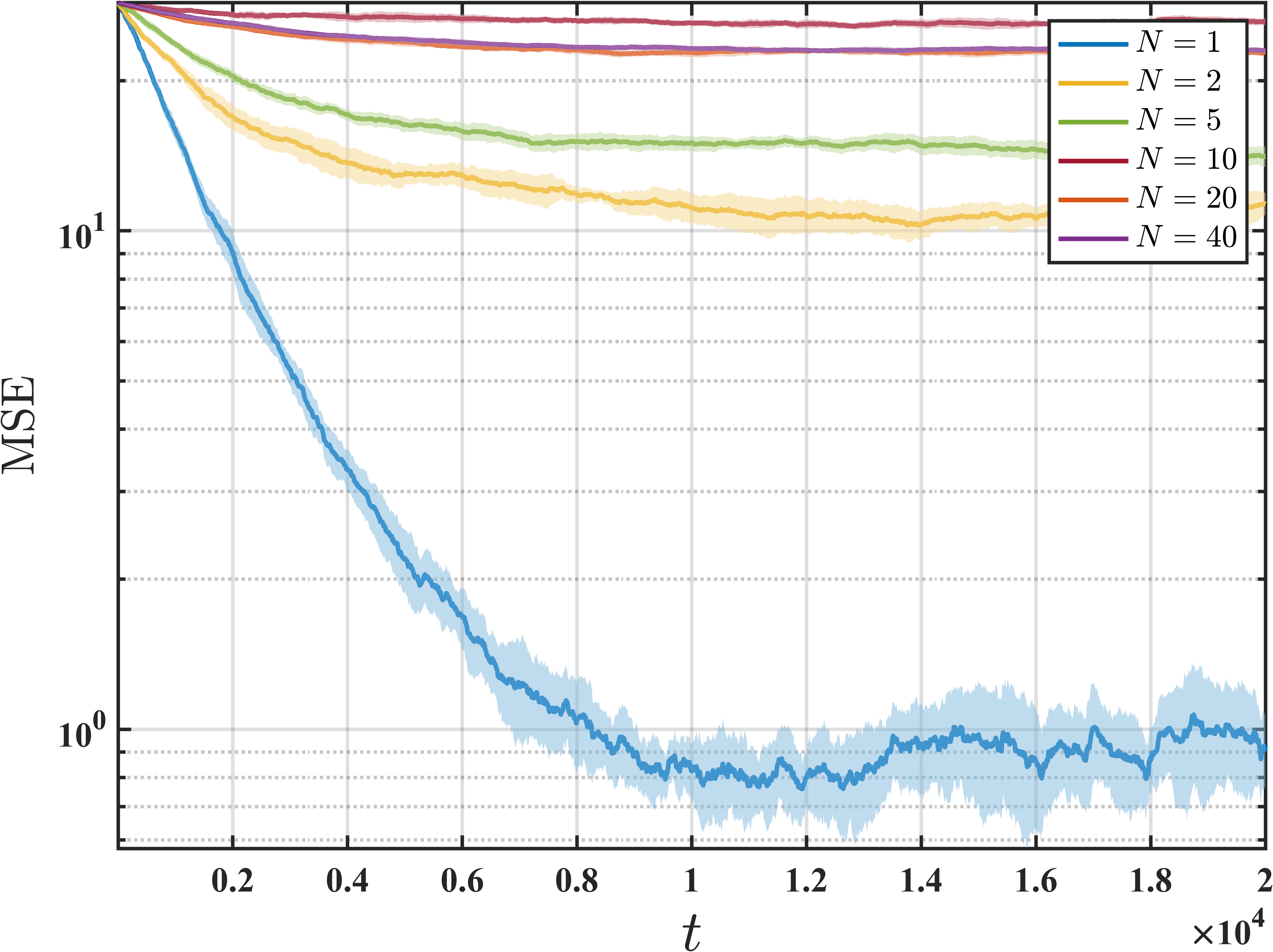}\vspace{-5pt}

      \caption{$\epsilon_p=\epsilon_r = 2$}
  \end{subfigure}
  \caption{Performance of \fedsarsa with a fixed-point policy improvement operator, covering federated TD(0).}
  \label{fig:td}
\end{figure}

\subsection{Simulations for On-Policy Federated Q-Learning}

When equipped with a greedy policy improvement operator, \fedsarsa reduces to on-policy federated Q-Learning.
Specifically, we employ the greedy policy improvement operator:
\[
    \pi_{\theta}(a|s) = \mathbbm{1}\{ a =\argmax_{a'\in \mathcal{A}} \theta^{T}\phi(s,a')\}
,\]
where \( \mathbbm{1} \) is the indicator function.
For the other part of the simulation setup, we inherit the setup from the previous subsection (\cref{sec:exp-sarsa}).
The results are presented in \cref{fig:ql}, which resemble the results in \cref{sec:exp,sec:exp-td}.

% fig:ql
\begin{figure}[ht]
  \centering
  \begin{subfigure}{0.325\textwidth}
      \centering
      \includegraphics[width=\textwidth]{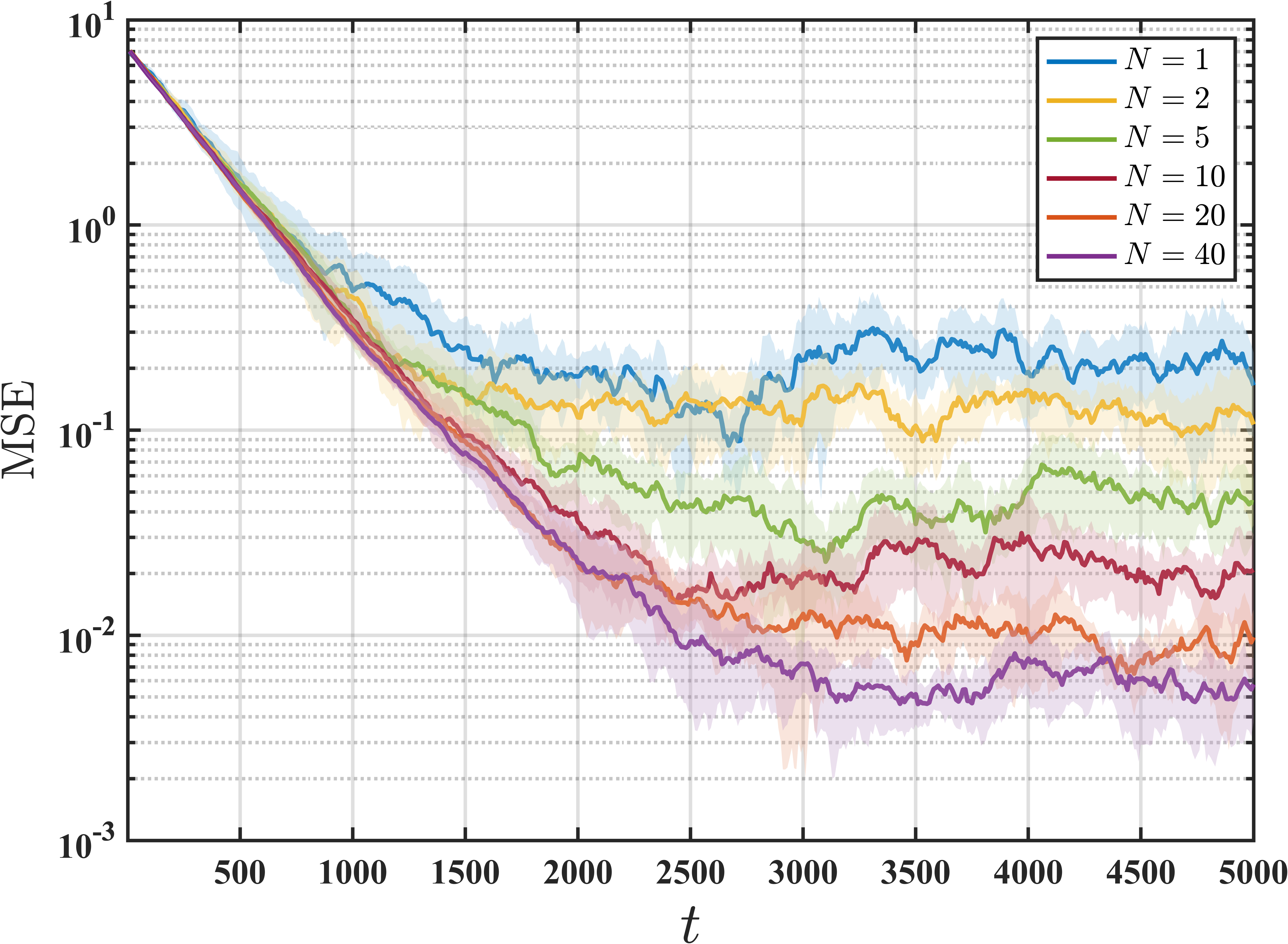}\vspace{-5pt}
      \caption{$\epsilon_p=\epsilon_r = 0$}
  \end{subfigure}
  \hfill
  \begin{subfigure}[b]{0.325\textwidth}
      \centering
      \includegraphics[width=\textwidth]{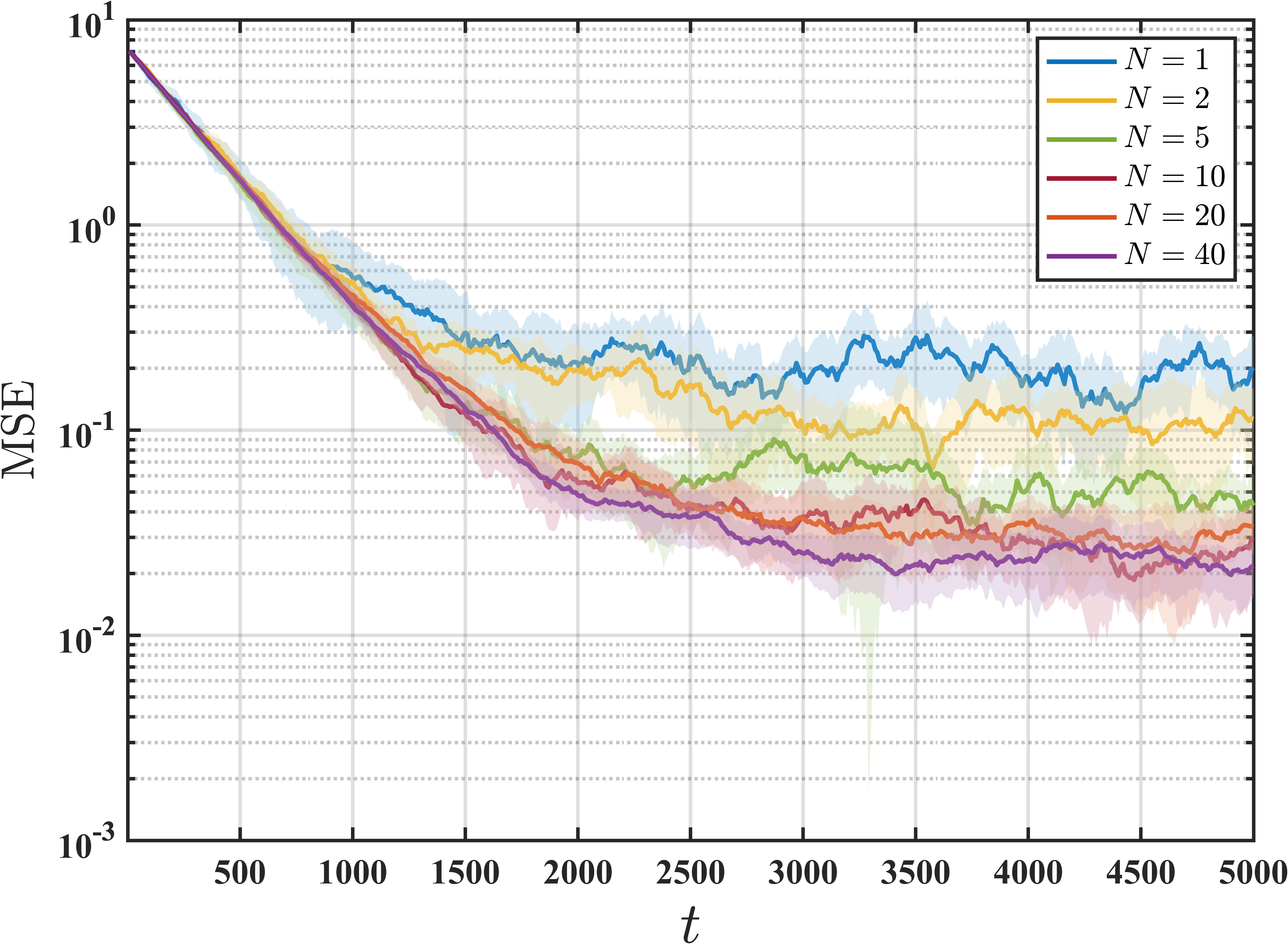}\vspace{-5pt}

      \caption{$\epsilon_p=\epsilon_r = 0.1$}
  \end{subfigure}
  \hfill
  \begin{subfigure}[b]{0.325\textwidth}
      \centering
      \includegraphics[width=\textwidth]{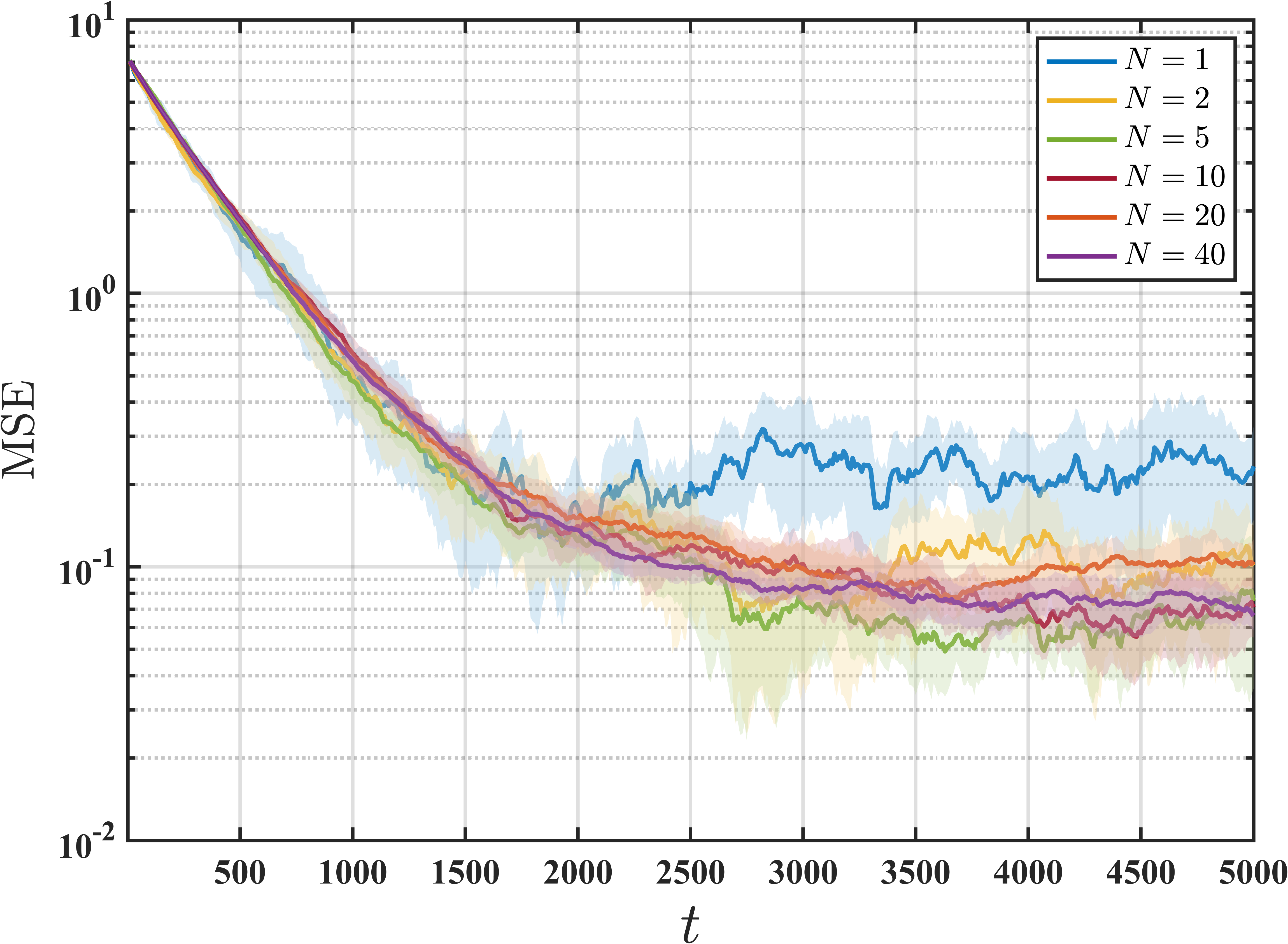}\vspace{-5pt}

      \caption{$\epsilon_p=\epsilon_r = 0.2$}
  \end{subfigure}
  \begin{subfigure}{0.325\textwidth}
      \centering
      \includegraphics[width=\textwidth]{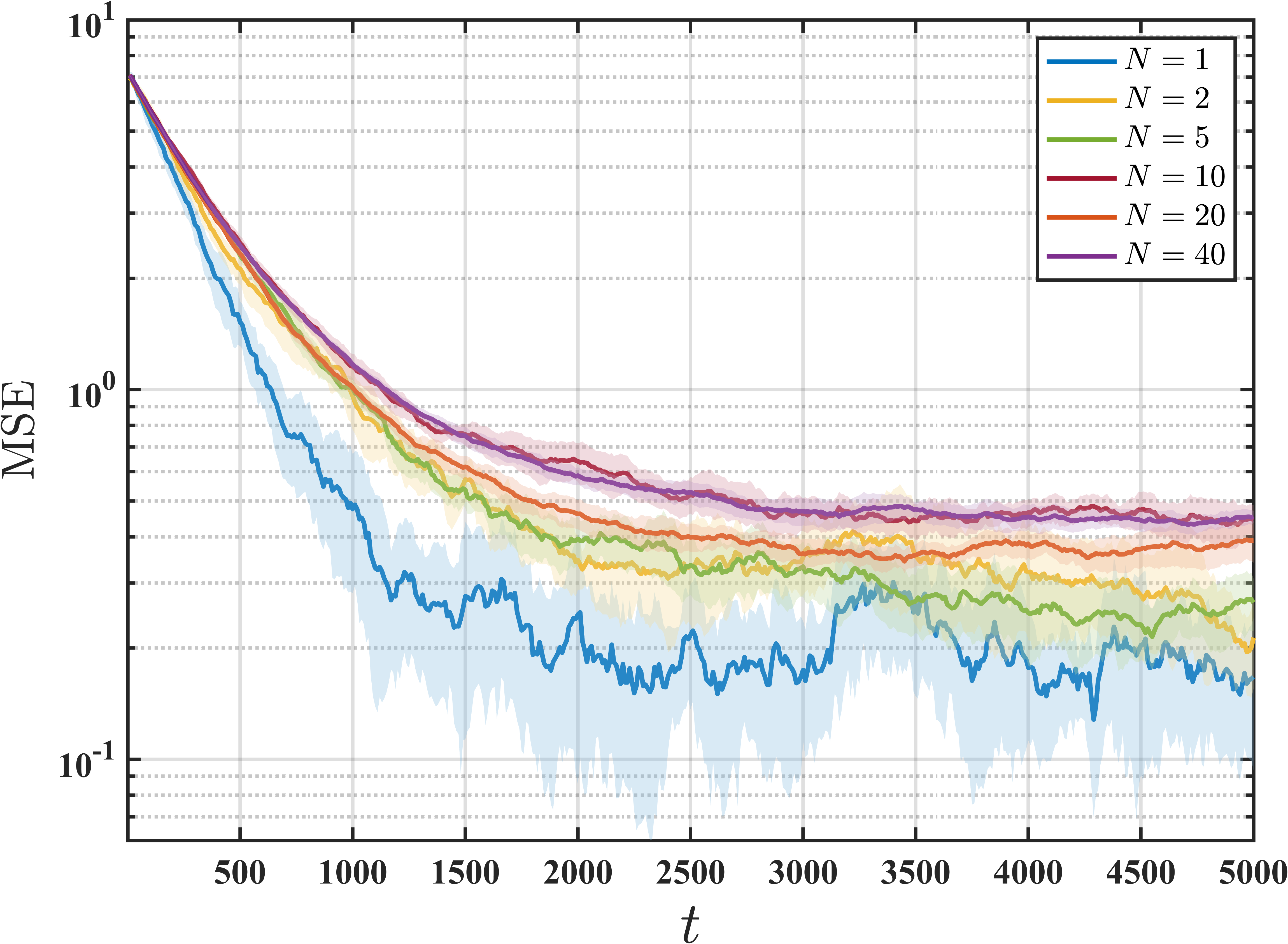}\vspace{-5pt}

      \caption{$\epsilon_p=\epsilon_r = 0.5$}
  \end{subfigure}
  \hfill
  \begin{subfigure}[b]{0.325\textwidth}
      \centering
      \includegraphics[width=\textwidth]{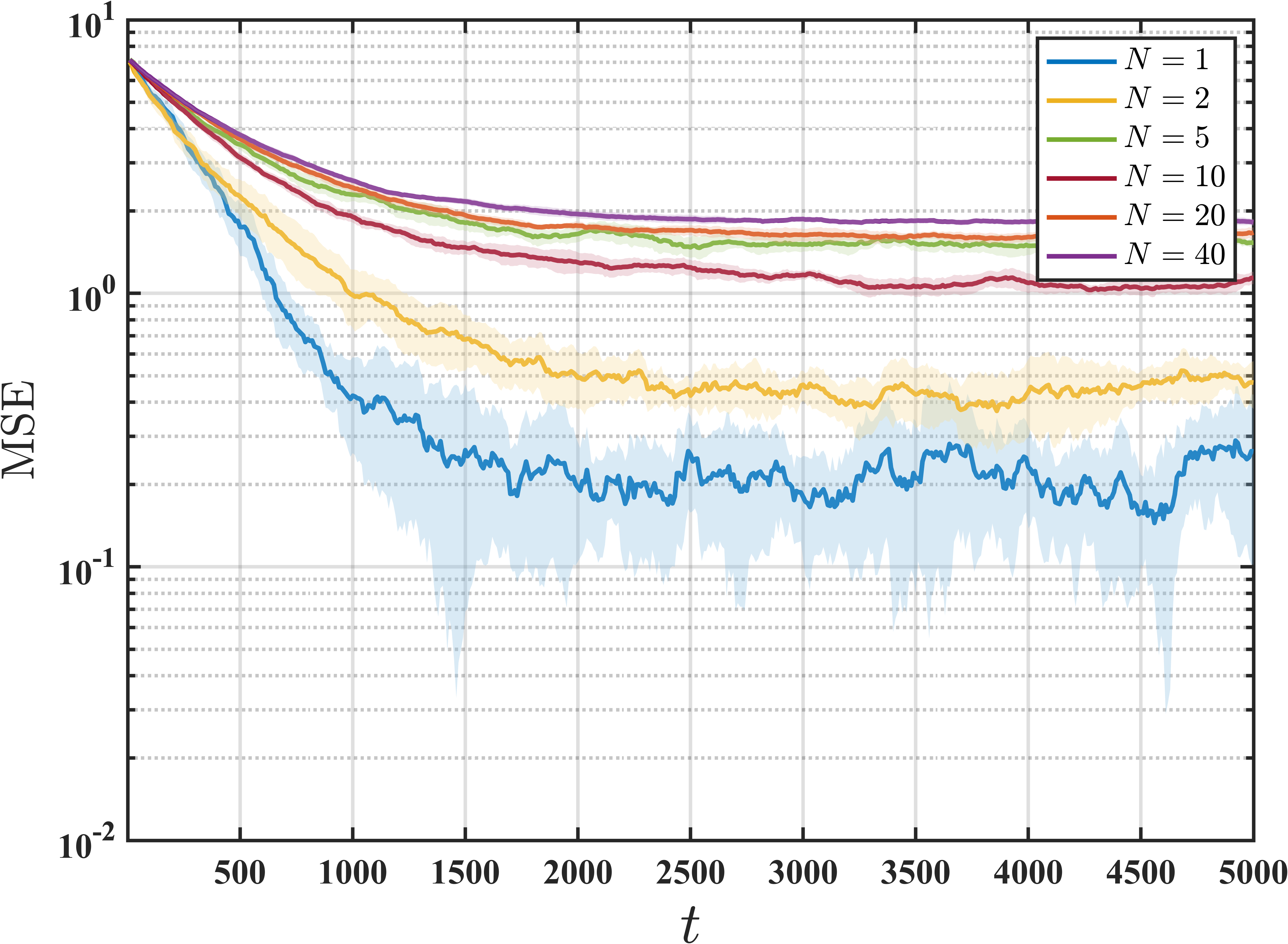}\vspace{-5pt}

      \caption{$\epsilon_p=\epsilon_r = 1$}
  \end{subfigure}
  \hfill
  \begin{subfigure}[b]{0.325\textwidth}
      \centering
      \includegraphics[width=\textwidth]{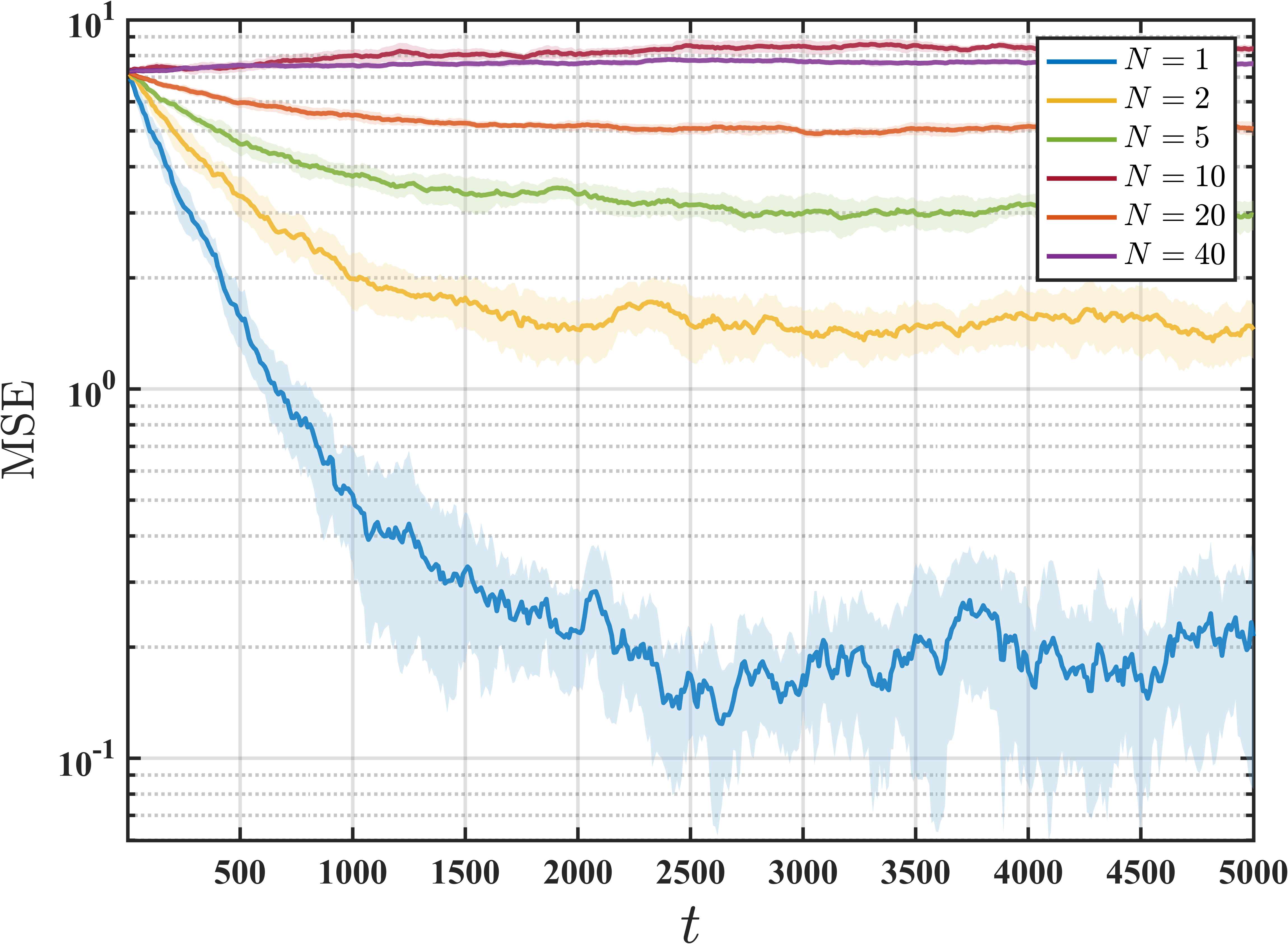}\vspace{-5pt}
      \caption{$\epsilon_p=\epsilon_r = 2$}
  \end{subfigure}
  \caption{Performance of \fedsarsa with a greedy policy improvement operator, covering on-policy federated Q-learning.}
  \label{fig:ql}
\end{figure}

\section{Central MDP} \label{sec:cmdp}

To facilitate our analysis, we introduce a virtual MDP: $\bar{\mathcal{M}} \coloneqq \frac{1}{N}\sumn \mathcal{M}\i$. Specifically, $\bar{\mathcal{M}} = (\mathcal{S},\mathcal{A},\bar{r},\bar{P},\gamma)$, where $\bar{r} = \frac{1}{N}\sumn r\i$, $\bar{P} = \frac{1}{N}\sumn P\i$. We refer to this virtual MDP as the \textit{central MDP}.
The following proposition shows that $\bar{\mathcal{M}}$ is indistinguishable from the collection of actual MDPs and also satisfies \cref{asmp:steady}. 

\begin{proposition}\label{prop:cmdp}
  If MDPs $\{\mathcal{M}\i\}$ are ergodic (aperiodic and irreducible) under a fixed policy $\pi$, the central MDP $\bar{\mathcal{M}}$ is also ergodic under $\pi$.
\end{proposition}
\begin{proof}
  Suppose $\pi$ is given. We first show that $\mathcal{M}$ is also aperiodic. If not, by the definition of aperiodicity \citep[Page 121]{meyn2012Markovchains}, there exists $s\in \mathcal{S}$ such that
  $$
  \bar{d}(s) \coloneqq \operatorname{gcd}\{ n: \bar{P}^{n}(s,s) > 0 \} > 1,
  $$
  where $\operatorname{gcd}$ returns the greatest common divisor and we omit the subscript $\pi$ of $\bar{P}_{\pi}$ since we consider a fixed policy.
  The above inequality indicates that, for any $n\in \mathbb{N}\setminus \{ k \bar{d}(s) \}_{k\in \mathbb{N}}$, it holds that
  \begin{equation}\label{eq:cmdp}
  0 = \bar{P}^{n}(s,s) = \left( \frac{1}{N}\sumn P\i \right)^{n}(s,s) \ge \frac{1}{N^{n}}\left( P\i \right)^{n}(s,s), \quad \forall i\in [N].
  \end{equation}
  Now since $\mathcal{M}\i$ is aperiodic, $d\i(s)\coloneqq \operatorname{gcd}\{ n: (P\i)^{n}(s,s) > 0 \} = 1$. Thus there exists $n\in \mathbb{N}\setminus \{ k \bar{d}(s) \}_{k\in\mathbb{N}}$ such that $(P\i)^{n}(s,s) > 0$ (otherwise $d\i(s) \ge \bar{d}(s)$), which contradicts to \cref{eq:cmdp}.
  Therefore, we conclude that $\bar{d}(s) = 1$ for any $s\in \mathcal{S}$, and thus $\bar{\mathcal{M}}$ is aperiodic.
  We now show that $\bar{\mathcal{M}}$ is irreducible given $\{ \mathcal{M}\i \}$ are irreducible \citep[Page 93]{meyn2012Markovchains}. For any $A\subset \mathcal{B}(\mathcal{S})$ with positive measure, where $\mathcal{B}(\mathcal{S})$ is the Borel $\sigma$-field on $\mathcal{S}$, we have $\min \{ n: (P\i)^{n}(s,A) >0 \} < +\infty$ for any $s\in \mathcal{S}$ and $i\in[N]$. Then again by \cref{eq:cmdp}, we get
  $$
  \min \{ n: \bar{P}^{n}(s,A)>0 \} \le \min_{i\in[N]}\min \left\{ n: \frac{1}{N^{n}}\left( P\i \right)^{n}(s,A) >0 \right\} < +\infty.
  $$
  Therefore, $\bar{\mathcal{M}}$ is irreducible. 
\end{proof}

When the state-action space is finite, \cref{prop:cmdp} covers \citet[Proposition 1]{wang2023FederatedTemporal}.

By \cref{prop:cmdp}, we can confidently regard $\bar{\mathcal{M}}$ as the MDP of a virtual agent, which does not exhibit any distinctive properties in comparison to the actual agents. 
Therefore, we denote $\mathcal{M}^{(0)} \coloneqq \bar{\mathcal{M}}$ and define the extended number set $[\bar{N}] \coloneqq [N] \cup \{ 0 \} = \{ 0,1, \ldots ,N \}$.
When we drop the superscript $(i)$, it should be clear from the context if we are talking about the central MDP $\bar{\mathcal{M}}$ or an arbitrary MDP $\mathcal{M}\i$.
Clearly, the extended MDP set $\{ \mathcal{M}\i \}_{i\in[\bar{N}]}$ still satisfies \cref{asmp:steady,asmp:ker-het,asmp:r-het}, and thus \cref{thm:fix-drift}.
Now we can specify the special parameter $\theta_{*}$ in \cref{thm}: it is the unique solution to \cref{eq:bellman} for $\bar{\mathcal{M}}$.
In other words, \cref{thm} asserts that the algorithm converges to the central optimal parameter.

\section{Notation} \label{sec:nota}

Before presenting lemmas and proofs of our main theorems, we introduce some notation that will aid in our analysis.
% As mentioned in \cref{rmk:proj}, we analyze a variant of \cref{alg} where the central aggregation step is substituted with the following:
% $$
% \bar{\theta}_{t+1} = \Pi_{\bar{G}}\left( \frac{1}{N}\sumn \theta\itt \right), \quad \text{when } t+1\equiv 0\pmod K.
% $$
We introduce a notation for the unprojected central parameter:
$$
\breve{\theta}_{t+1} = \frac{1}{N} \sumn \left( \theta\it + \alpha_t g\it \right), \quad \text{when } t+1\equiv 0\pmod K.
$$
Then, $\theta\itt = \bar{\theta}_{t+1} = \Pi_{\bar{G}}(\breve{\theta}_{t+1})$ when $t+1\equiv 0\pmod K$.
It's easy to verify that for any $\|\theta\|\le \bar{G}$, we have
\begin{equation}\label{eq:proj-prop}
\|\bar{\theta}_t - \theta\| \le \|\breve{\theta}_t -\theta\|.
\end{equation}

Then we define some notations on the MDPs. Note that all these definitions apply to the extended MDP set $\{ \mathcal{M}\i \}_{i\in[\bar{N}]}$ that includes the central MDP.

\begin{definition}[Steady distributions]%\label{def:steady}
	\cref{asmp:steady} guarantees the existence of a steady state distribution for any MDP and policy $\pi$.
	We denote $\eta_{\theta}\i$ as the steady state distribution with respect to MDP $\mathcal{M}\i$ and policy $\pi_{\theta}$, i.e.,
	\[
		\eta\i_{\theta}(s) \coloneqq \lim_{t\to\infty}P\i_{\pi_{\theta}}(S_t = s|S_0 = s_0).
	\]
	Additionally, given a policy $\pi_{\theta}$, the steady state-action distribution is defined as
	\[\mu\i_{\theta}(s,a) \coloneqq  \eta\i_{\theta}(s)\cdot \pi_{\theta}(a|s).\]
	Then, the two-step steady distribution is defined as
	\[
		\varphi\i_{\th}(s,a,s',a') \coloneqq \mu\i_{\th}(s,a)P_{a}\i(s,s')\pi_{\theta}(a'|s').
	\]
	For a local parameter $\theta\it$, we simplify the above notations as follows:
	\[
		\eta\it \coloneqq \eta\i_{\theta\it},\quad
		\mu\it \coloneqq \mu\i_{\theta\it},\quad
		\varphi\it \coloneqq \varphi\i_{\theta\it}.
	\]
	% When we drop the superscript $(i)$, these notations denote the steady distributions with respect to the central MDP $\bar{\mathcal{M}}$.
\end{definition}

We are now ready to provide the precise definitions of the semi-gradients discussed in \cref{sec:anlys}.

\begin{definition}[Semi-gradients] \label{def:semi-grad}
	As indicated by \eqref{eq:g}, a semi-gradient is a function of both the parameter $\theta$ and the observation tuple $O = (s,a,s',a')$, while the observation tuple is dependent on the local Markovian trajectory.
	Therefore, the general form of a semi-gradient is
	\[
		g\ita\left(\theta; O\it\right)
		\coloneqq \phi\left(s\it,a\it\right)\left(r\i\left(s\it,a\it\right) + \gamma\phi^T\left(s\itt,a\itt\right)\theta - \phi^T\left(s\it,a\it\right)\theta\right),
	\]
	where $O\it = \left(s\it,a\it,s\itt,a\itt\right)$ is the observation of agent $i$ at time step $t$, and the subscript $t\!-\!\tau$ indicates that the trajectory after time step $t\!-\!\tau$ follows a fixed policy $\pi_{\theta\ita}$, i.e.,
	\[
		a\i_{t-\tau},a\i_{t-\tau+1}, \ldots, a\i_{t+1} \sim \pi_{\theta\i_{t-\tau}}.
	\]
	% TODO to check if this notation is used
	When $\tau = 0$, the semi-gradient corresponds to an actual SARSA trajectory, and we omit the subscript and the observation argument, i.e.,
	\[
		g\i(\theta) \coloneqq g\it\left(\theta; O\it\right).
	\]
	When $\tau > 0$, the semi-gradient corresponds a virtual trajectory, and we use $\tilde{O}_t$ in place of $O_t$ to indicate it is a virtual observation at the current time step.

	We add a bar to denote the mean-path semi-gradients, i.e.,
	\[
		\bar{g}\ita(\theta) \coloneqq \EE_{\varphi\ita}\left[ g\ita\left( \theta,O \right) \right] = \EE_{\varphi\ita}\left[\phi(s,a)(r\i(s,a) + \gamma\phi^T(s',a')\theta - \phi^T(s,a)\theta)\right],
	\]
	where the expectation is taken over the two-step observation steady distribution:
	\[
		O = (s,a,s',a') \sim \varphi\ita \coloneqq \varphi\i_{\theta\ita}.
	\]
	For mean-path semi-gradients, the randomness of the observation is eliminated, and the parameter $\theta$ is the only argument,
	and we can substitute the subscript $t\!-\!\tau$ with a general parameter $\theta'$; then we define
	\[
		\bar{g}\i_{\theta'}(\theta) \coloneqq \EE_{\varphi\i_{\theta'}}\left[\phi(s,a)(r\i(s,a) + \gamma\phi^T(s',a')\theta - \phi^T(s,a)\theta)\right].
	\]
	% Note that this notation only works for mean-path semi-gradients, where the distribution of observations is uniquely determined by the parameter $\theta$.

	We omit the superscript $(i)$ when referring to the central MDP $\bar{\mathcal{M}}$. For instance,
	\[
		\bar{g}\tta(\theta) \coloneqq \EE_{\varphi\tta}\left[\phi(s,a)(\bar{r}(s,a) + \gamma\phi^T(s',a')\theta - \phi^T(s,a)\theta)\right].
	\]

	Finally, for notational simplicity, we use bold symbols to denote the average semi-gradients, e.g.,
	\[
		\bm{g}_{t-\tau}(\bm{\theta}_t) = \frac{1}{N}\sumn g\ita\left(\theta\it\right).
	\]

	The above notations will be used in combination, e.g.,
	%TODO to check if the notion is used
	\[
		\bar{\bm{g}}(\bm{\theta}_t) = \frac{1}{N}\sumn \bar{g}\i\left(\theta\it\right) = \frac{1}{N}\sumn \bar{g}_t\i\left(\theta\it\right).
	\]

\end{definition}

%TODO to clarify
% {\color{orange} Generally, we use a bar on operators to indicate it takes expectation on the steady distribution, which is policy dependent.}
% Unlike TD learning, where the policy is given and fixed, SARSA aims to find an optimal policy.
We can further decompose semi-gradients into TD operators.

\begin{definition}[TD operators]\label{def:td-op}
	A semi-gradient $g\ita\left( \theta,O\it \right)$ can be decomposed into the following two two operators:
	\[
		g\ita\left( \theta,O\it \right) = A\ita\left( O\it \right)\theta + b\ita\left( O\it \right).
	\]
	where
	\[
		\begin{cases}
			A\i_{t-\tau}\left(O\it\right) & = \phi\left(s\it,a\it\right)\left(\gamma\phi^T\left(s\itt,a\itt\right) - \phi^T\left(s\it,a\it\right)\right), \\[2ex]
			b\ita\left(O\it\right)        & = \phi\left(s\it, a\it\right)r\i\left(s\it,a\it\right),
		\end{cases}
		\quad a\it, a\itt \sim \pi_{\theta\i_{t-\tau}}.
	\]
	Similar to \cref{def:semi-grad}, we can define other TD operators for each semi-gradient, e.g., the mean-path TD operators:
	\[
		\begin{cases}
			\bar{A}\i_{\theta} & = \EE_{\varphi\i_{\theta}}[A\i\left(O\right)], \\[1ex]
			\bar{b}\i_{\theta} & = \EE_{\mu\i_{\theta}}[b\i\left(O\right)].     \\
		\end{cases}
	\]
\end{definition}

We summarize the notations defined in this section and other notations used in our analysis in \cref{tb:notation}.
% Please refer to \cref{def:steady,def:semi-grad,def:td-op} for the formal definitions.

%TODO to change at last
\begin{table}[ht]
	\caption{~Notation}
	\label{tb:notation}
	\renewcommand{\arraystretch}{1.2}
	\centering\begin{tabular}{cl}
		\toprule
		Notation                                           & Definition                                                                  \\
		\midrule
		$[N], [\bar{N}]$                                   & The set of $N$ numbers and the set of $N+1$ numbers including $0$           \\
		$\mathcal{M}\i, \bar{\mathcal{M}}$                 & Markov decision processes                                                   \\
		$\S, \A,\Theta$                                    & State space, action space, and parameter space                              \\
		$r\i, \bar{r}, P\i, \bar{P}$                       & Reward functions and transition kernels                                     \\
		$S\it, U\it, O\it$                                 & Agent $i$'s state, action, and observation random variable at time step $t$ \\
		$s, a, o$                                          & Instances of the state, action, and observation                             \\
		$\pi, \Gamma$                                      & A policy and the policy improvement operator                                \\
		$\|\cdot \|_{\mathrm{TV}}$                         & Total variation distance and its induced norm for transition kernels        \\
		$q, Q$                                             & True Q-value function and estimated Q-value function                        \\
		$\phi, \theta$                                     & Feature map and feature weight (parameter)                                  \\
		$\Pi_{\pi}, \Pi_{\bar{G}}$                         & Orthogonal projection operator                                              \\
		$T_{\pi}$                                          & Bellman operator                                                            \\
		$\pi\i_{*}, \theta\i_{*}$                          & Optimal policies and optimal parameters                                     \\
		$\eta\i_{\theta}, \mu\i_{\theta}, \varphi\i_\theta$ & Steady distributions                                                        \\
		$g$                                                & Semi-gradient                                                               \\
		$A, b, Z$                                          & Temporal difference operators                                               \\
		$h$                                                & $h(\theta) \coloneqq R + (1+\gamma)\|\theta\|$                              \\
		$\Omega_t, \omega_t$                               & Client drift                                                                \\
		$\mathcal{F}_t$                                    & Filtration containing all randomness prior to time step $t$                 \\
		\bottomrule
	\end{tabular}
\end{table}

\section{Constants} \label{sec:const}

We first introduce two important constants that serve as base constants throughout the paper.
The first one is the upper bound of the norm of the central parameter, denoted by $G \ge \|\bar{\theta}\|$.
For this bound to hold, we require the projection radius $\bar{G}$ to be large enough such that $\left\| \theta\i_{*} \right\|\le \bar{G}$ for $i\in[\bar{N}]$.
The explicit expression for $G$ will be given in \cref{cor:G}.
Then, we define
\begin{equation}\label{eq:H}
	H = R + (1 + \gamma)G.
\end{equation}
The constant $H$ can be viewed as the scale of the problem, analogous to $|\S||\A|$ for the tabular setting that will be discussed in \cref{sec:tab}.
For local parameters, we define a similar function $h(\theta) \coloneqq R + (1+\gamma)\|\theta\|$.

We summarized the constants that appear in our analysis in \cref{tb:const}.
Notice that $\tau,\alpha_0$, and $\alpha_{t}$ in \cref{tb:const} refer to the constants in the case with a linearly decaying step-size. For the case with a constant size, these constants are fixed and specified in \cref{cor:err-con}.

\begin{table}[ht]
	\caption{~Constants}
	\label{tb:const}
	\centering\begin{tabular}{cp{0.309\textwidth}p{0.19\textwidth}c}%{tabularx}{\textwidth}{cLp{2.8cm}c}
		\toprule
		Notation                    & Meaning                                                        & Reference                                                & Range or Order                  \\
		\midrule
		$N$                         & Number of agents                                               & \cref{sec:pre-mdp}                                       & $\mathbb{N}$                    \\
		$R$                         & Reward cap                                                     & \cref{sec:pre-mdp}                                       & $(0,+\infty)$                   \\
		$S, A$                      & Measures of the state space and action space                   & \cref{sec:pre-mdp}                                       & $(0,+\infty]$                   \\
		$\gamma$                    & Discount factor                                                & \cref{sec:pre-mdp}                                       & $(0,1)$                         \\
		$m_i, m$                    & Markov chain mixing constant                                   & \cref{asmp:steady}\newline and \cref{lem:distri-hetero}  & $[1,+\infty)$                   \\
		$\rho_i, \rho$              & Markov chain mixing rate                                       & \cref{asmp:steady} \newline and \cref{lem:distri-hetero} & $(0,1)$                         \\
		$\sigma, \sigma'$           & Steady distribution perturbation constant                      & \cref{lem:distri-hetero}                                 & $O(\log m /(1-\rho))$           \\
		$\bar{G}$                   & Algorithm projection radius                                    & Algorithm~\ref{alg}                                      & $(0,+\infty)$                   \\
		$G$                         & Parameter norm upper bound                                     & \cref{cor:G}                                             & $O(\bar{G} + R)$                \\
		$H$                         & Problem scale                                                  & Equation~\cref{eq:H}                                     & $O(\bar{G} + R)$                \\
		$L$                         & Lipschitz constant for the policy improvement operator         & \cref{asmp:lip}                                          & $[0, w /(H\sigma)]$             \\
		$K$                         & Local update period                                            & \cref{alg:agg}                                           & $\mathbb{N}$                    \\
		$\epsilon_{p},\epsilon_{r}$ & Environmental heterogeneity ratio                              & \cref{asmp:ker-het,asmp:r-het}                           & $[0,2]$                         \\
		$\Lambda$                   & Environmental heterogeneity                                    & \cref{thm:fix-drift}                                     & $O(H(\epsilon_p + \epsilon_r))$ \\
		$\lambda\i, \lambda$        & Exploration constant                                           & Equation~\eqref{eq:lam}                                  & $(0,1)$                         \\
		$w_{i}, w$                  & Convergence constant                                           & Equation~\cref{eq:w}                                     & $[(1-\gamma)\lambda/2, 1 /2)$   \\
		$\tau$                      & Backtracking period                                            & \cref{lem:decomp}                                        & $O(\log T)$                     \\
		$\alpha_0$                  & Initial step-size                                              & \cref{sec:alg-local}                                     & $(0, \min\{ 1 /8K, w /64 \}]$   \\
		$\alpha_t$                  & General step-size                                              & \cref{sec:alg-local}                                     & $O(1 /t)$                       \\
		$C_{\mathrm{drift}}$        & Client drift constant                                          & \cref{lem:drift}                                         & $O(KH)$                         \\
		$C_{\mathrm{prog}}$         & Parameter progress constant                                    & \cref{lem:v-para-prog}                                   & $O(H\tau)$                  \\
		$C_{\mathrm{back}}$         & Backtracking constant                                          & \cref{lem:stat}                                          & $O(\tau^2 w)$                   \\
		$C_{\mathrm{var}}$          & Gradient variance constant                                     & \cref{lem:grad-norm}                                     & $O(H^2w^2\tau^4))$       \\
		$\beta$                     & Young's inequality constant                                    & \cref{sec:thm-pf}                                        & $(0, w /7)$                     \\
		$H_{\mathrm{drift}}$        & Another drift constant                                         & \cref{sec:thm-pf}                                        & $O(H)$                          \\
		$C_{\alpha}$                & Step-size constant                                             & \cref{sec:thm-pf}                                        & $O(1)$                          \\
		$C_1$                       & First-order constant                                           & Equation~\cref{eq:per-const}                             & $O((1-\gamma)^{-1})$            \\
		$C_2$                       & Second-order constant                                          & Equation~\cref{eq:per-const}                             & $O(H^2\tau)$                    \\
		$C_3$                       & Third-order constant                                           & Equation~\cref{eq:per-const}                             & $O(H^2w\tau^4)$                 \\
		$C_4$                       & Fourth-order constant                                          & Equation~\cref{eq:per-const}                             & $O(H^2w^2\tau^5)$               \\
		$B$                         & Square of the convergence region radius for constant step-size & \cref{cor:err-con}                                       & see \cref{cor:err-con}          \\
		\bottomrule
	\end{tabular}
\end{table}

\section{Preliminary Lemmas} \label{sec:aux}

In this section, we present two preliminary lemmas that will be used throughout the analysis.

\begin{lemma}[Steady distribution differences] \label{lem:distri-hetero}
	For the same MDP, the TV distance between the steady distributions with regard to two different policies is bounded as follows:
	\[
		\begin{aligned}
			\left\|\eta_{\theta_1} -  \eta_{\theta_2}\right\|_{\mathrm{TV}}       & \le L\sigma'\left\| \theta_1 - \theta_2  \right\|_{2},     \\
			\left\|\mu_{\theta_1} -  \mu_{\theta_2}\right\|_{\mathrm{TV}}         & \le L(1+\sigma')\left\| \theta_1 - \theta_2  \right\|_{2}, \\
			\left\|\varphi_{\theta_1} -  \varphi_{\theta_2}\right\|_{\mathrm{TV}} & \le L(2+\sigma)\left\| \theta_1 - \theta_2  \right\|_{2},  \\
		\end{aligned}
	\]
	where $L$ is the Lipschitz constant of the policy improvement operator specified in \cref{asmp:lip} and $\sigma'$ is a constant determined by $m$ and $\rho$ specified in \cref{asmp:steady}. Letting $\sigma \coloneqq \sigma' + 2$, all three TV distances above are bounded by $L\sigma\|\theta_{1} - \theta_{2}\|_{2}$.
	Next, for a fixed parameter $\theta$, the TV distance between the steady distributions with regard to two MDPs is bounded as follows:
	\[
		\begin{aligned}
			\left\| \eta_{\theta}\i - \eta_{\theta}\j \right\|_{\mathrm{TV}}       & \le \sigma'\epsilon_{p},     \\
			\left\| \mu_{\theta}\i - \mu_{\theta}\j \right\|_{\mathrm{TV}}         & \le \sigma'\epsilon_{p},     \\
			\left\| \varphi_{\theta}\i - \varphi_{\theta}\j \right\|_{\mathrm{TV}} & \le (\sigma'+1)\epsilon_{p}.
		\end{aligned}
	\]
	By the above inequalities, for different MDPs and different parameters, we have
	%TODO check if only this difference is used
	\[
		\left\| \mu_{\theta\i}\i - \mu_{\theta\j}\j \right\|_\tv \le \sigma'\epsilon_p + L\sigma\|\theta\i - \theta\j\|_{2}.
	\]
\end{lemma}
\begin{proof}
	For the same MDP, by \cite[Corollary 3.1]{mitrophanov2005SensitivityConvergence}, we get
	\[
		\left\|\eta_{\theta_1} -  \eta_{\theta_2}\right\|_{\mathrm{TV}} \le \sigma'\|P_{\theta_1} - P_{\theta_2}\|_{\mathrm{TV}},
	\]
	where $P_{\theta}(s,s') = \int_{\mathcal{A}}P_{a}(s,s')\pi_{\theta}(a|s)\d a$, and
	\[\|P_{\theta}\|_{\mathrm{TV}} = \sup_{\|q\|_{\mathrm{TV}}= 1} \|qP_{\theta}\|_{\mathrm{TV}} = \sup_{\|q\|_{\mathrm{TV}}=1} \left\|\int_{\mathcal{S}}q(s)P_{\theta}(s,\cdot )\d s\right\|_{\mathrm{TV}}.\]
	And the constant $\sigma'$ is defined by
	\begin{equation}\label{eq:sigma}
		\sigma' = \hat{n} + \frac{m\rho^{\hat{n}}}{1-\rho},
	\end{equation}
	where $\hat{n} = \left\lceil \log_{\rho}m^{-1} \right\rceil$, $m \coloneqq \max_{i\in[N]}m_i$, and $\rho \coloneqq \max_{i\in[N]}\rho_i$ with $m_i,\rho_i$ specified in \cref{asmp:steady}.
	Note that in the above inequalities, we actually should use $\sigma'_i$ defined by $m_i$ and $\rho_i$; but $\sigma'_{i}$ is bounded by $\sigma'$ for all $i\in[N]$, so we use this possibly looser bound for notational simplicity.
	Then, by \cref{asmp:lip}, we have
	\[
		\begin{aligned}
			\|P_{\theta_1} - P_{\theta_2}\|_{\mathrm{TV}}
			 & = \sup_{\|q\|_{\mathrm{TV}}=1}\int_{\mathcal{S}}\left| \int_{\mathcal{S}}q(s)(P_{\theta_1}(s,s') - P_{\theta_2}(s,s'))\d s' \right| \d s                                   \\
			 & = \sup_{\|q\|_{\mathrm{TV}}=1}\int_{\mathcal{S}}\left| \int_{\mathcal{S\times \mathcal{A}}}q(s)(P_{a}(s,s')\pi_{\theta_1}(a|s) - P_{a}(s,s')\pi_{\theta_2}(a|s))\d a\d s' \right| \d s                                   \\
			 & \le \sup_{\|q\|_{\mathrm{TV}}=1}\int_{\mathcal{S}^{2}\times \mathcal{A}}\left|q(s)\right| P_{a}(s,s') \left|\pi_{\theta_1}(a|s) - \pi_{\theta_2}(a|s)\right|\d a\d s' \d s \\
			 & = \sup_{\|q\|_{\mathrm{TV}}=1}\int_{\mathcal{S}\times \mathcal{A}}\left|q(s)\right| \left|\pi_{\theta_1}(a|s) - \pi_{\theta_2}(a|s)\right|\d a\d s                         \\
			 & = \sup_{\|q\|_{\mathrm{TV}}=1}\int_{\mathcal{S}}\left|q(s)\right|  \left\| \pi_{\theta_1}(\cdot |s) -  \pi_{\theta_2}(\cdot |s) \right\|_{\mathrm{TV}}\d s                 \\
			 & \le L \|\theta_{1} - \theta_2\|_{2}\sup_{\|q\|_{\mathrm{TV}}=1}\int_{\mathcal{S}}\left|q(s)\right| \d s                                                                    \\
			 & = L\|\theta_1 - \theta_2\|_{2}.
		\end{aligned}
	\]
	Therefore, we get
	\[
		\left\|\eta_{\theta_1} -  \eta_{\theta_2}\right\|_{\mathrm{TV}} \le L\sigma'\|\theta_1-\theta_2\|_{2}.
	\]
	Next, for the state-action distribution, we have
	\[
		\begin{aligned}
			\left\|\mu_{\theta_1} - \mu_{\theta_2}\right\|_{\mathrm{TV}}
			 & = \int_{\mathcal{S}\times \mathcal{A}}\left| \eta_{\theta_1}(s)\pi_{\theta_1}(a|s) - \eta_{\theta_2}(s)\pi_{\theta_2}(a|s) \right| \d s\d a                                                                                                            \\
			 & \le \int_{\mathcal{S}\times \mathcal{A}}\eta_{\theta_1}(s)\left| \pi_{\theta_1}(a|s) - \pi_{\theta_2}(a|s) \right| \d s\d a + \int_{\mathcal{S}\times \mathcal{A}}\left| \eta_{\theta_1}(s) - \eta_{\theta_2}(s) \right| \pi_{\theta_2}(a|s) \d a \d s \\
			 & \le L\|\theta_{1}-\theta_2\|_{2} + \left\|\eta_{\theta_1} - \eta_{\theta_2}\right\|_{\mathrm{TV}}                                                                                                                                                          \\
			 & \le L(1 + \sigma') \|\theta_1-\theta_2\|_{2}.
		\end{aligned}
	\]
	Similarly, we have
	\[
		\left\|\varphi_{\theta_1} -  \varphi_{\theta_2}\right\|_{\mathrm{TV}} \le L(2 + \sigma')\|\theta_1-\theta_2\|_{2}.
	\]

	Also by \cite[Corollary 3.1]{mitrophanov2005SensitivityConvergence}, we get
	\[
		\left\| \eta_{\theta}\i - \eta_{\theta}\j \right\|_{\mathrm{TV}} \le \sigma' \|P_{\theta}\i - P_{\theta}\j\|_{\mathrm{TV}} \le \sigma'\epsilon_{p},
	\]
	where $\epsilon_{p}$ is defined in \cref{asmp:ker-het}. Then, for the state-action distribution, we have
	\[
		\left\| \mu_{\theta}\i - \mu_{\theta}\j \right\|_{\mathrm{TV}} = \left\| \eta_{\theta}\i\cdot \pi_{\theta} - \eta_\theta\j \cdot \pi_{\theta} \right\|_{\mathrm{TV}} = \left\| \eta_{\theta}\i - \eta_{\theta}\j \right\|_{\mathrm{TV}} \le \sigma'\epsilon_p.
	\]
	And similarly, we have
	\[
		\begin{aligned}
			    & \left\| \varphi_{\theta}\i - \varphi_{\theta}\j \right\|_{\mathrm{TV}}                                                                                                                                   \\
			= & \int_{S^{2}\times A^{2}} \left| \mu_{\theta}\i(s,a)\pi_{\theta}(a|s)P\i_{a}(s,s')\pi_{\theta}(a'|s') - \mu_{\theta}\j(s,a)\pi_{\theta}(a|s)P\j_{a}(s,s')\pi_{\theta}(a'|s')  \right| \d s\d s'\d a\d a'  \\
			\le & \int_{S^{2}\times A^{2}} \left| \mu_{\theta}\i(s,a)\pi_{\theta}(a|s)P\i_{a}(s,s')\pi_{\theta}(a'|s') - \mu_{\theta}\j(s,a)\pi_{\theta}(a|s)P\i_{a}(s,s')\pi_{\theta}(a'|s')  \right| \d s\d s'\d a\d a'  \\
			    & +\int_{S^{2}\times A^{2}} \left| \mu_{\theta}\j(s,a)\pi_{\theta}(a|s)P\i_{a}(s,s')\pi_{\theta}(a'|s') - \mu_{\theta}\j(s,a)\pi_{\theta}(a|s)P\j_{a}(s,s')\pi_{\theta}(a'|s')  \right| \d s\d s'\d a\d a' \\
			\le & \left\| \mu_{\theta}\i - \mu_{\theta}\j \right\|_{\mathrm{TV}} + \|P\i - P\j\|_{\mathrm{TV}}                                                                                                             \\
			\le & (\sigma' + 1)\epsilon_{p} \le \sigma \epsilon_{p}
		\end{aligned}
	\]

	Finally, by the triangle inequality, we get
	\[
		\left\| \mu_{\theta\i}\i - \mu_{\theta\j}\j \right\|_{\mathrm{TV}} \le \sigma'\epsilon_{p} + L\sigma\|\theta\i - \theta\j\|_{2}.
	\]
\end{proof}

Similarly, we can bound the differences between TD operators defined in Definition~\ref{def:td-op}.

\begin{lemma}[TD operator differences] \label{lem:td-hetero}
	For the same MDP, the difference between the mean-path TD operators with regard to different parameters is bounded as follows:
	\[
		\begin{cases}
			\left\| \bar{A}_{\theta_1} - \bar{A}_{\theta_2} \right\|
			 & \le (1+\gamma)L\sigma \left\| \theta_{1} - \theta_{2} \right\|_{2}, \\[1ex]
			\left\| \bar{b}_{\theta_1} - \bar{b}_{\theta_2} \right\|
			 & \le RL\sigma \left\| \theta_{1} - \theta_{2} \right\|_{2}.
		\end{cases}
	\]
	Next, for a fixed parameter $\theta$, the difference between the mean-path TD operators with regard to different MDPs is bounded as follows:
	\[
		\begin{cases}
			\left\| \bar{A}\i_{\theta} - \bar{A}\j_{\theta} \right\|
			 & \le (1+\gamma)\sigma \epsilon_{p},    \\[1.5ex]
			\left\| \bar{b}\i_{\theta} - \bar{b}\j_{\theta} \right\|
			 & \le R(\epsilon_r + \sigma\epsilon_{p}).
		\end{cases}
	\]
	Then, by the triangle inequality, we get
	\[
		\begin{cases}
			\left\| \bar{A}\i_{\theta\i} - \bar{A}\j_{\theta\j} \right\|
			 & \le (1+\gamma)\sigma \left( L \left\| \theta\i-\theta\j \right\|_{2}+\epsilon_{p} \right),      \\[1.5ex]
			\left\| \bar{b}\i_{\theta\i} - \bar{b}\j_{\theta\j} \right\|
			 & \le R(\epsilon_r + \sigma \epsilon_{p}) + RL\sigma\left\| \theta\i-\theta\j \right\|_{2}.
		\end{cases}
	\]
\end{lemma}
\begin{proof}
	For the same MDP, by \cref{def:td-op}, we have
	\[
		\begin{aligned}
			\left\| \bar{A}_{\theta_1} - \bar{A}_{\theta_2} \right\|
			=   & \left\| \int_{\S^2\times \A^2}\phi(s,a)(\gamma\phi^T(s',a') - \phi^T(a,s))(\d \varphi_{\theta_1}(s,a,s',a') - \d \varphi_{\theta_2}(s,a,s',a')) \right\| \\
			\le & (1+\gamma)\left\| \varphi_{\theta_1} - \varphi_{\theta_2} \right\|_{\tv}                                                                                 \\
			\le & (1+\gamma) L\sigma \left\| \theta_1-\theta_2 \right\|_{2},
		\end{aligned}
	\]
	where the last inequality comes from Lemma~\ref{lem:distri-hetero}. Similarly, we have
	\[
		\begin{aligned}
			\left\| \bar{b}_{\theta_1} - \bar{b}_{\theta_2} \right\|
			=   & \left\| \int_{\S\times \A}\phi(s,a)r(s,a)(\d \mu_{\theta_1}(s,a) - \d \mu_{\theta_2}(s,a)) \right\| \\
			\le & R \left\| \mu_{\theta_1}-\mu_{\theta_2} \right\|_{\tv}                                              \\
			\le & RL\sigma\left\| \theta_1-\theta_2 \right\|_{2}.
		\end{aligned}
	\]

	Then for the same parameter $\theta$, we have
	\[
		\begin{aligned}
			\left\| \bar{A}\i_{\theta} - \bar{A}\j_{\theta} \right\|
			=   & \left\| \int_{\S^2\times \A^2}\phi(s,a)(\gamma\phi^T(s',a') - \phi^T(a,s))(\d \varphi\i_{\theta}(s,a,s',a') - \d \varphi\j_{\theta}(s,a,s',a')) \right\| \\
			\le & (1+\gamma)\left\| \varphi\i_{\theta} - \varphi\j_{\theta} \right\|_{\tv}                                                                                 \\
			\le & (1+\gamma) \sigma \epsilon_{p},
		\end{aligned}
	\]
	where the last inequality comes from Lemma~\ref{lem:distri-hetero}. Similarly, we have
	\[
		\begin{aligned}
			\left\| \bar{b}\i_{\theta} - \bar{b}\j_{\theta} \right\|
			=   & \left\| \int_{\mathcal{S}\times \mathcal{A}}\phi(s,a)\left( r\i(s,a)\d \mu\i_{\theta}(s,a) - r\j(s,a)\d \mu\j_{\theta}(s,a)  \right)\right\|                                                             \\
			\le & \int_{\mathcal{S}\times \mathcal{A}}\left| r\i(s,a) - r\j(s,a)\right|\d \mu\i_{\theta}(s,a) + \int_{\mathcal{S}\times \mathcal{A}}r\j(s,a)\left| \d \mu\i_{\theta}(s,a) - \d \mu\j_{\theta}(s,a) \right| \\
			\le & R\epsilon_r + R\sigma'\epsilon_{p},                                                                                                                                                                           \\
		\end{aligned}
	\]
	where the last inequality comes from \cref{asmp:r-het} and \cref{lem:distri-hetero}.
\end{proof}

% Similar to \cite[Theorem 1]{wang2023FederatedTemporal}, we give bounds on the difference of SARSA fixed points, which describe the near optimality of the universal policy.

\section{Proof of Theorem~\ref{thm:fix-drift}}\label{sec:thm-drift-pf}

\begin{reptheorem}{thm:fix-drift}
  For any $i,j\in[\bar{N}]$, we have
  \[
    \left\| \theta\j_{*} - \theta\i_{*} \right\|_{2}
    \le \frac{1}{w_{j}}\left( R\epsilon_r + H\sigma\epsilon_{p}\right) \le \frac{\Lambda(\epsilon_{p},\epsilon_r)}{w},
  \]
  where $w \coloneqq \min_{i\in[\bar{N}]}w_i$; $w_i$ is defined in \cref{lem:des-dir} and $\Lambda(\epsilon_{p},\epsilon_r)$ is defined in \cref{lem:grad-hetero}.
\end{reptheorem}
\begin{proof}
  First, we formulate the Bellman optimal equation in terms of TD operators defined in \cref{def:td-op}:
  \[
    \bar{A}\i_*\theta\i_{*} + \bar{b}\i_* = 0,
  \]
  for any $i\in[\bar{N}]$, where
  \[
    \bar{A}\i_* \coloneqq \bar{A}\i_{\theta\i_{*}},\quad \bar{b}\i_{*} \coloneqq \bar{b}\i_{\theta\i_{*}}.
  \]
  Then for any $i,j\in[\bar{N}]$, we have
  \[
    \left(\bar{A}\j_{*} - \bar{A}\i_{*}\right)\theta\i_{*} + \bar{A}\j_{*}\left( \theta\j_{*}-\theta\i_{*} \right) = \bar{b}\i_{*} - \bar{b}\j_{*}.
  \]
  By \citet[Theorem 2]{tsitsiklis1997analysistemporaldifference}, $\bar{A}\j_{*}$ is negative definite %in the sense that $x^*Ax < 0$ for any nonzero vector $x$. 
	Therefore, $\bar{A}\j_{*}$ is non-singular, and we get
  \[
    \left\| \theta\j_{*} - \theta\i_{*} \right\|_{2}
    \le \left\| \left( \bar{A}\j_{*} \right)^{-1} \right\| \left\| \left( \bar{A}\i_{*}-\bar{A}\j_{*} \right)\theta\i_{*} + \left( \bar{b}\i_{*} - \bar{b}\j_{*} \right) \right\|_{2}.
  \]
  And we have
  \begin{align}\label{eq:fix-1}
    \left\| \left( \bar{A}\j_{*} \right)^{-1} \right\|
    =   & \sigma_{\min}^{-1}\left( \bar{A}\j_{*} \right)                                         \\\label{eq:fix-2}
    =   & \frac{1}{\left|\lambda_{\max}\left( \bar{A}\j_{*} \right)\right|}                      \\\label{eq:fix-3}
    \le & \frac{1}{-\Re{\lambda_{\max}\left( \bar{A}\j_{*} \right)}}                             \\\label{eq:fix-4}
    \le & \frac{1}{-\lambda_{\max}\left( \operatorname{sym}\left( \bar{A}\j_{*} \right) \right)} \\\label{eq:fix-5}
    =   & \frac{1}{2w_j},
  \end{align}
  where \eqref{eq:fix-1} uses the spectrum norm equality and $\sigma_{\min}$ returns the smallest singular value of a matrix;
  \eqref{eq:fix-2} and \eqref{eq:fix-3} use the fact that $\bar{A}\j_{*}$ is negative definite;
  \eqref{eq:fix-4} is by \cite[Theorem 10.28]{zhang2011Matrixtheory};
  and lastly, \eqref{eq:fix-5} is the definition of $w_j$ (see \cref{lem:des-dir}).

  Therefore, letting $G$ be large enough to contain $\{ \theta\i_{*} \}_{i\in[\bar{N}]}$, we get
  \[
    \left\| \theta\j_{*} - \theta\i_{*} \right\|_{2}
    \le \frac{1}{2w_j}\left( \left\| A\i_{*}-A\j_{*} \right\|G + \left\| b\i_{*} - b\j_{*} \right\| \right).
  \]
  By Lemma~\ref{lem:td-hetero}, we get
  \[
    \begin{aligned}
      \left\| \theta\j_{*} - \theta\i_{*} \right\|_{2}
      \le & \frac{1}{2w_{j}}\left( (1+\gamma)\sigma G\left(\epsilon_{p} + L\left\| \theta\i_{*}-\theta\j_{*} \right\|_{2}  \right) + R(\epsilon_r +\sigma\epsilon_{p})+ RL\sigma \left\| \theta\i_{*}-\theta\j_{*} \right\|_{2}\right) \\
      \le & \frac{1}{2 w_{j}}\left(R\epsilon_r + H\sigma\epsilon_{p} + LH\sigma \left\| \theta\i_{*}-\theta\j_{*} \right\|_{2}  \right).
    \end{aligned}
  \]
  We require that $LH\sigma \le w_j$ (the same restriction \eqref{eq:L} in \cref{lem:des-dir}); then we get
  \[
    \left\| \theta\j_{*} - \theta\i_{*} \right\|_{2}
    \le \frac{1}{2 w_{j}}\left(R\epsilon_r + H\sigma\epsilon_{p}\right) + \frac{w_j}{2 w_{j}} \left\| \theta\i_{*}-\theta\j_{*} \right\|_{2},
  \]
  which gives
  \[
    \left\| \theta\j_{*} - \theta\i_{*} \right\|_{2}
    \le \frac{1}{w_{j}}\left(R\epsilon_r + H\sigma\epsilon_{p}\right) \le \frac{\Lambda(\epsilon_{p},\epsilon_r)}{w},
  \]
  where $w \coloneqq \min_{i\in[\bar{N}]}w_i$ and $\Lambda(\epsilon_{p},\epsilon_r):= R\epsilon_{r} + H\sigma\epsilon_{p}$ (the same definition in \cref{lem:grad-hetero}).

\end{proof}

\section{Key Lemmas} \label{sec:use-lem}

In this section, we first decompose the mean squared error and then present seven lemmas, each bounding one term in the decomposition.

\subsection{Error Decomposition} \label{sec:err-dec}

\begin{lemma}[Error decomposition] \label{lem:decomp}
	The one-step mean squared error can be decomposed recursively as follows:
	\[
		\begin{aligned}\notag
			\EE & \left\| \bar{\theta}_{t+1} - \theta_{*}\right\|^{2} \le \EE\bigl\| \breve{\theta}_{t+1} - \theta_{*} \bigr\|^{2}  = \EE\left\| \bar{\theta}_{t} - \theta_{*} \right\|^{2}                                         \\[1.2ex]
			    & + 2\alpha _{t}\EE\left<\bar{\theta}_{t}-\theta_{*},\bar{g}\left(\bar{\theta}_t\right) - \bar{g}\left(\theta_{*}\right) \right>                                       & ~ & \textup{(descent direction)}      \\%\label{eq:decomp-dd}       \\
			    & + \frac{2\alpha_{t}}{N}\sumn\EE\left<\bar{\theta}_{t}-\theta_{*},\bar{g}\i\left(\bar{\theta}_t\right) - \bar{g}\left(\bar{\theta}_{t}\right) \right>                 & ~ & \textup{(gradient heterogeneity)} \\%\label{eq:decomp-het} \\
			    & + \frac{2\alpha_{t}}{N}\sumn\EE\left<\bar{\theta}_{t}-\theta_{*},\left( \bar{g}\i\left(\theta\it\right) - \bar{g}\i\left(\bar{\theta}_{t}\right) \right) \right>     & ~ & \textup{(client drift)}           \\%\label{eq:decomp-drift}         \\
			    & + \frac{2\alpha_{t}}{N}\sumn\EE\left<\bar{\theta}_{t}-\theta_{*},\bar{g}\ita\left(\theta\it\right) - \bar{g}\i\left(\theta\it\right) \right>                         & ~ & \textup{(gradient progress)}      \\%\label{eq:decomp-prog}     \\
			    & + \frac{2\alpha_{t}}{N}\sumn\EE\left<\bar{\theta}_{t}-\theta_{*},g\i\tta\left(\theta\it, \too\it\right) - \bar{g}\i\tta\left(\theta\it\right) \right>                & ~ & \textup{(mixing)}                 \\%\label{eq:decomp-mix}                 \\
			    & + \frac{2\alpha_{t}}{N}\sumn\EE\left<\bar{\theta}_{t}-\theta_{*},g\it\left(\theta\it, O\it\right) - g\i\tta\left(\theta\it, \too\it\right) \right>                   & ~ & \textup{(backtracking)}           \\%\label{eq:decomp-back}             \\
			    & + \al_{t}^{2}\EE\left\| \frac{1}{N}\sumn g\it\left(\theta\it\right) \right\|^{2}.                                                                                    & ~ & \textup{(gradient variance)}      \\%\label{eq:decomp-var}
		\end{aligned}
	\]
\end{lemma}
One can verify the above decomposition given $\bar{g}(\theta_{*}) = 0$.
% subsection Error Decomposition (end)

\subsection{Descent Direction} \label{sec:des-dir}

\begin{lemma}[Descent direction] \label{lem:des-dir}
	There exist positive constants $\{ w_i \}_{i\in[\bar{N}]}$ such that for any $\|\theta\| \le G$, we have
	\[
		\begin{aligned}
			\left<\theta - \theta_{*}\i, \bar{g}\i(\theta) - \bar{g}\i(\theta_{*}\i) \right> & \le -w_i\left\| \theta - \theta_{*}\i \right\|^{2},\quad \forall i\in[\bar{N}]. \\
			% \left<\theta - \theta_{*}, \bar{g}(\theta) - \bar{g}(\theta_{*}) \right>         & \le -w\left\| \theta - \theta_{*} \right\|^{2}.
		\end{aligned}
	\]
\end{lemma}
\begin{proof}
	We drop the subscript $(i)$ in this lemma since the following derivation holds for all MDPs. We first denote $\Delta\theta = \theta - \theta_{*}$. Then, we have
	\begin{align}\notag
		\left<\theta - \theta_{*},\bar{g}(\theta) - \bar{g}(\theta_{*}) \right>
		=         & \Delta\theta^{T}\left( \left( \bar{A}_{\theta}\theta + \bar{b}_{\theta} \right) - \left( \bar{A}_{\theta_{*}}\theta_{*} + \bar{b}_{\theta_*} \right) \right)                             \\\notag
		=         & \Delta\theta^{T}\bar{A}_{\theta_{*}}\Delta\theta + \Delta\theta^T\left( \bar{A}_{\theta} - \bar{A}_{\theta_{*}} \right)\theta + \Delta\theta^T (\bar{b}_{\theta} - \bar{b}_{\theta_{*}}) \\\notag
		% \le       & \Delta\theta^{T}\bar{A}_{\theta_{*}}\Delta\theta
		% + \|\Delta\theta\| \left\|\int_{\mathcal{S}\times \mathcal{A}}\phi(s,a)\bar{r}(s,a)(\d \mu_{\theta}(s,a) - \d \mu_{\theta_{*}}(s,a))\right\|                                                                     \\\notag
		%           & + \|\Delta\theta\|\|\theta\|\left\|\int_{\S^{2}\times \A^{2}}\phi(s,a)\left(\gamma\phi^{T}(s',a') - \phi^T(s,a)\right) (\d \varphi_{\theta}(s,a,s',a') - \d \varphi_{\theta_{*}}(s,a,s',a'))\right\| \\\label{eq:gd-1}
		\le       & \Delta\theta^{T}\bar{A}_{\theta_{*}}\Delta\theta
		+ \|\Delta\theta\| \left\|\bar{A}_{\theta} - \bar{A}_{\theta_{*}}\right\|\|\theta\|
		+ \|\Delta\theta\|\left\|\bar{b}_{\theta} - \bar{b}_{\theta_{*}}\right\|                                                                                                                             \\
		% \label{eq:gd-1}
		% \le       & \Delta\theta^{T}\bar{A}_{\theta_{*}}\Delta\theta
		% + \|\Delta\theta\|\cdot R \|\mu_{\theta}-\mu_{\theta_{*}}\|_{\mathrm{TV}} + \|\Delta\theta\|\|\theta\|\cdot (1 + \gamma) \|\varphi_{\theta}- \varphi_{\theta_{*}}\|_{\mathrm{TV}}                             \\
		\label{eq:gd-2}
		\le       & \Delta\theta^{T}\bar{A}_{\theta_{*}}\Delta\theta + (1+\gamma)L\sigma\|\theta\|\|\Delta\theta\|^{2} + RL\sigma\|\Delta\theta\|^{2}                                                        \\\notag
		=         & \Delta\theta^T(\bar{A}_{\theta_{*}} + L\sigma(R + (1+\gamma)\|\theta\|)I)\Delta\theta                                                                                                    \\\label{eq:gd-1}
		\le       & \Delta\theta^T\left( \bar{A}_{\theta_{*}} + L\sigma H\cdot I\right)\Delta\theta                                                                                                          \\\notag
		\eqqcolon & \Delta\theta^T\widetilde{A}_{\theta_{*}}\Delta\theta,
	\end{align}
	where \eqref{eq:gd-2} uses Lemma \ref{lem:td-hetero}, and \cref{eq:gd-1} uses the fact that $\|\theta\|\le G$ and $H\coloneqq R + (1+\gamma)G$.
	% Similarly, for any $i\in[N]$, we have
	% \[
	% 	\left<\theta - \theta_{*}\i, \bar{g}\i\left(\theta\right) - \bar{g}\i\left(\theta_{*}\i\right) \right> \le \left(\Delta\theta\i\right)^T\left(\bar{A}\i_{\theta\i_{*}} + L\sigma H \cdot I\right)\Delta\theta\i \coloneqq \left(\Delta\theta\i\right)^T\widetilde{A}_{\theta\i_{*}}\i\Delta\theta\i.
	% \]
	By \cite[Theorem 2]{tsitsiklis1997analysistemporaldifference}, $\bar{A}_{\theta_{*}}$ is negative definite in the sense that $x^*Ax < 0$ for any vector $x\in \R^{d}$. Specifically, for any nonzero $x\in\R^{d}$, we denote $u = x^T\phi(S,A)$ and $u' = x^T\phi(S',A')$. Then, for any $x\neq 0$, by \cref{def:td-op}, we have
	\begin{equation}\label{eq:gd-3}
		x^T\bar{A}_{\theta}x = \EE_{\varphi_{\theta}}\left[ \gamma uu' - u^2\right]
		= \gamma\EE[uu'] - \EE[u^2]
		\le \frac{\gamma}{2}\left( \EE[u^2] + \EE[u'^2] \right) - \EE[u^2]
		= (\gamma - 1)\EE[u^2]
		< 0,
	\end{equation}
	where we use the fact that $\EE[u^2] = \EE[u'^2]$ under a steady distribution.
	Let $\Phi\i_{\theta} \coloneqq \EE_{\mu\i_{\theta}}[\phi(S,A)\phi^T(S,A)]\succ 0$. We define
	\begin{equation}\label{eq:lam}
		\lambda\i \coloneqq \lambda_{\min}\left(\Phi\i_{\theta\i_{*}}\right), \quad
		\lambda \coloneqq \min_{i\in[\bar{N}]}\lambda\i.
	\end{equation}
	By \cref{eq:gd-3}, we have
	\begin{equation}\label{eq:w}
		-w_i \coloneqq \frac{1}{2}\lambda_{\max}\left( \operatorname{sym}\left( \bar{A}\i_{\theta\i_{*}} \right)\right) \le \frac{1}{2}(\gamma-1) \lambda_{\min}\left( \Phi\i_{\theta\i_{*}} \right) = \frac{\gamma-1}{2}\lambda\i.
	\end{equation}
	where $\operatorname{sym}(A) \coloneqq \frac{1}{2}(A + A^*)$ maps general matrices to Hermitian matrices.
	% We define function $\operatorname{sym}(A) = \frac{1}{2}(A + A^*)$ mapping general matrices to Hermitian matrices. Then we denote
	% \[
	% 	-w_i \coloneqq \frac{1}{2}\lambda_{\max}\left( \operatorname{sym}\left( \bar{A}\i_{\theta\i_{*}} \right)\right),\quad
	% 	-w \coloneqq \frac{1}{2}\lambda_{\max}\left( \operatorname{sym}\left( \bar{A}_{\theta_{*}} \right) \right).
	% \]
	Due to the positive definiteness of $\Phi\i_{\theta\i_{*}}$, We know $w_i > 0$ for any $i\in[\bar{N}]$.
	By \citet[Theorem 10.21]{zhang2011Matrixtheory} and the linearity of the $\operatorname{sym}$ function, we know
	\[
		\lambda_{\max}\left(\operatorname{sym}\left(\widetilde{A}\i_{\theta\i_{*}}\right)\right)
		\le \lambda_{\max}\left(\operatorname{sym}\left(\bar{A}\i_{\theta\i_{*}}\right)\right) + \lambda_{\max}(\operatorname{sym}(L\sigma H \cdot I)) = -2w_{i} + L\sigma H.
	\]
	Let $w = \min_{i\in[\bar{N}]}\{ w_i \}$. Then, we can choose $L$ to be small enough such that
	\begin{equation}\label{eq:L}
		L \le \frac{w}{\sigma H},
	\end{equation}
	which gives
	\[
		-w_i \ge \lambda_{\max}\left( \operatorname{sym}\left( \widetilde{A}\i_{\theta\i_{*}} \right) \right).
	\]
	% We denote
	% \[
	% 	-w\i \coloneqq \lambda_{\max}\left( \frac{1}{2}\left( \widetilde{A}_{\theta_*}\i + \left(\widetilde{A}_{\theta_*}\i\right)^T \right) \right),\quad
	% 	-w \coloneqq \lambda_{\max}\left( \frac{1}{2}\left( \widetilde{A}_{\theta_*} + \left(\widetilde{A}_{\theta_*}\right)^T \right) \right).
	% \]
	Therefore, for any $i\in[\bar{N}]$, we have
	\begin{equation}\label{eq:des-dir}
		\begin{aligned}
			\left<\theta - \theta_{*}\i, \bar{g}\i(\theta) - \bar{g}\i(\theta_{*}\i) \right>
			\le \lambda_{\max}\left( \operatorname{sym}\left( \widetilde{A}\i_{\theta_{*}\i} \right) \right) \left\|\theta-\theta\i_*\right\|^2
			 & \le -w_i\left\| \theta - \theta_{*}\i \right\|^{2}.
		\end{aligned}
	\end{equation}
\end{proof}

\begin{remark}[Convergence constant]
	\Cref{eq:des-dir} mirrors the result of stochastic gradient descent (SGD) \citep{bottou2018Optimizationmethods}, with $w$ being analogous to the Lipschitz constant of a function's gradient. Therefore, similar to SGD, $w$ controls the convergence rate of our algorithm.
\end{remark}

\begin{remark}[Exploration constant]\label{rmk:lam}
	The value of $w$ depends on $\lambda$, a constant that reflects the \textit{exploration difficulty} of the environment. We can see this by considering a simple tabular setting, where the feature map $\phi$ is simply the indicator function (see \cref{sec:tab} for detailed definitions).
	Then $\EE_{\mu}[\phi(S,A)\phi^T(S,A)]$ reduces to $\operatorname{diag} \{ \mu(s,a) \}_{(s,a)\in\S\times \A}$.
	In this case, the minimal eigenvalue of $\Phi$ is $\min_{(s,a)\in\S\times \A}\mu(s,a)$, i.e., the probability of visiting the least probable state-action pair under the steady distribution.

	We say an environment is \textit{hard to explore} if some state-action pairs have a very small probability of being visited under the steady distribution, then $\lambda$ is small.
	Conversely, $\lambda$ is large when the environment is easy to explore.
	Intuitively, an environment that is hard to explore requires more samples to learn an optimal policy.
	% Also, for this tabular case, we have the upper bound $\lambda \le 1/(SA)$, which can be extremely small when $SA$ is large.
	% This is another reason why linear function approximation can expedite the learning process.

	In the context of LFA, the value of $\lambda$, and consequently $w$, is determined by the conditions of both the MDPs and the feature map $\phi$.
	If the environments in the feature space are easy to explore under the MDPs, $\lambda$ and $w$ will take on larger values, and the algorithm converges faster.
\end{remark}

\subsection{Gradient Heterogeneity} \label{sec:grad-hetero}

\begin{lemma}[Gradient heterogeneity] \label{lem:grad-hetero}
	For $\|\theta\|\le G$, we have
	\[
		\left\| \bar{g}(\theta) - \frac{1}{N}\sumn \bar{g}\i(\theta) \right\| \le
		% 2\sigma\epsilon_p\left(R + (1+\gamma)\|\theta\|\right) + \epsilon_r.
		H\sigma\epsilon_{p} + R\epsilon_r
		\eqqcolon \Lambda(\epsilon_{p},\epsilon_r)
	\]
\end{lemma}
\begin{proof}
	Directly applying the decomposition in \cref{def:td-op} and \cref{lem:td-hetero} gives
	\[
		\begin{aligned}
			\left\| \bar{g}(\theta) - \frac{1}{N}\sumn \bar{g}\i(\theta) \right\|
			= & \left\| (\bar{A}_{\theta}\theta + \bar{b}_{\theta}) - \frac{1}{N}\sumn (\bar{A}_{\theta}\i\theta + \bar{b}\i_{\theta}) \right\|              \\
			% \le & \frac{1}{N}\sumn \left\| (\bar{A}_{\theta}\theta + \bar{b}_{\theta}) - (\bar{A}_{\theta}\i\theta + \bar{b}\i_{\theta}) \right\|              \\
			\le & \frac{1}{N}\sumn \left( \left\| \bar{b}_{\theta} - \bar{b}_{\theta}\i \right\| +  \|\bar{A}_{\theta} - \bar{A}\i_{\theta}\|\|\theta\|\right) \\
			% \le & \frac{1}{N}\sumn \left(\|\bar{A}_{\theta} - \bar{A}\i_{\theta}\|\|\theta\| + \left\| \int_{\mathcal{S}\times \mathcal{A}}\phi(s,a)\left( \bar{r}(s,a)\d \mu_{\theta}(s,a) - r\i(s,a)\d \mu_{\theta}\i(s,a)  \right)\right\| \right)                                                           \\
			% \le & \frac{1}{N}\sumn \left(\|\bar{A}_{\theta} - \bar{A}\i_{\theta}\|\|\theta\| + \int_{\mathcal{S}\times \mathcal{A}}\left| \bar{r}(s,a) - r\i(s,a)\right|\d \mu_{\theta}(s,a) + \int_{\mathcal{S}\times \mathcal{A}}r\i(s,a)\left| \d \mu_{\theta}(s,a) - \d \mu\i_{\theta}(s,a) \right| \right) \\
			% \le & \frac{1}{N}\sumn \left( (1+\gamma)\|\theta\| \left\| \varphi_{\theta} - \varphi_{\theta}\i \right\|_{\tv} + \epsilon_r + R \left\|\mu_{\theta} - \mu_{\theta}\i \right\|_{\tv}\right)                                                                                                         \\
			\le & \sigma\epsilon_{p}\left(R + (1+\gamma)\|\theta\|\right) + R\epsilon_r,
		\end{aligned}
	\]
	% where the penultimate inequality uses \eqref{eq:gd-1}-\eqref{eq:gd-2} and Assumption~\ref{asmp:r-het}.
\end{proof}

\subsection{Client Drift} \label{sec:drift}

Before bounding the gradient progress, we first bound the client drift.

\begin{lemma}[Client drift] \label{lem:drift}
	If $\|\bar{\theta}_{t}\| \le {G}$ holds for all $t\in \mathbb{N}$, then
	\[
		\frac{1}{N}\sumn \left\| \bar{g}\i(\theta\it) - \bar{g}\i(\bar{\theta}_{t}) \right\|^{2}
		% \le \left( 1+\gamma + \sigma ACH \right)^2\cdot 4k\alpha _{t-k}^{2} \exp\left( 4k(3 + k\alpha _{t-k}^{2}) \right)\left( kR^{2} + 4G^2  \right),
		\le \alpha _{t-k}^2 \left( 1+\gamma + \sigma LH \right)^2 C^2_{\mathrm{drift}},
	\]
	where $k$ is the smallest integer such that $t-k \equiv 0 \pmod{K}$, and
	\[
		C^2_{\mathrm{drift}} = 4K^2H^2.
	\]
\end{lemma}
\begin{proof}
	Similar to \cref{eq:gd-2} in the proof of Lemma~\ref{lem:des-dir}, we have
	\begin{equation}\label{eq:Lip-mean-path}
		\left\| \bar{g}\i(\theta_t\i) - \bar{g}\i(\bar{\theta}_t) \right\| \le
		\left( 1+\gamma + L\sigma\left( R + (1+\gamma)\|\bar{\theta}_t\| \right) \right)\left\| \theta\it - \bar{\theta}_t \right\|
	\end{equation}
	Then, since $\|\bar{\theta}_{t}\|\le G$, we have
	\begin{equation}\label{eq:drift-0}
		\frac{1}{N}\sumn\left\| \bar{g}\i(\theta_t\i) - \bar{g}\i(\bar{\theta}_t) \right\|^{2}
		\le \left(1+\gamma+\sigma LH\right)^{2} \cdot \frac{1}{N}\sumn\left\| \theta\it - \bar{\theta}_t \right\|^{2}.
	\end{equation}
	Let $\Omega_{t} \coloneqq \frac{1}{N}\sumn\left\| \theta\it - \bar{\theta}_t \right\|^{2}$. We then need to bound $\Omega_t$.
	First, if $t\equiv 0\pmod K$, we have $\Omega_t = 0$.
	Now suppose $t\not\equiv 0\pmod K$.
	Let $k$ be the smallest integer such that $t-k \equiv 0\pmod K$. Then we know that there is no aggregation step between time step $t-k$ and $t$, and $\bar{\theta}_{t-l} = 1/N\sumn \theta\i_{t-l}$ for $0\le l\le k$.
	Therefore, we have
	\[
		\begin{aligned}
			\left\| \theta\it - \bar{\theta}_t \right\|^{2}
			 & = \left\| \theta_{t-k}\i - \bar{\theta}_{t-k} + \sum_{l=1}^k\alpha _{t-l}\left( g\i_{t-l}(\theta\i_{t-l}) - \bm{g}_{t-l}(\bm{\theta}_{t-l}) \right) \right\|^2 \\
			 & \le k\alpha _{t-k}^{2}\sum_{l=1}^{k}\left\| g_{t-l}\i(\theta\i_{t-l}) - \bm{g}_{t-l}(\bm{\theta}_{t-l}) \right\|^{2},
		\end{aligned}
	\]
	where $\bm{g}_{t}(\bm{\theta}_{t}) = \frac{1}{N}\sumn g\it(\theta\it)$, and we choose $\alpha$ to be non-increasing.
	Since for a random vector $X$, $\operatorname{Var}(X)\le \EE\|X\|^{2}$, we have
	\begin{align}\notag
		\Omega_{t}
		 & \le k\alpha _{t-k}^{2} \sum_{l=1}^{k}\frac{1}{N}\sumn \left\| g_{t-l}\i(\theta_{t-l}\i) - \bm{g}_{t-l}(\bm{\theta}_{t-l}) \right\|^{2}                                                                    \\\notag
		 & \le k\alpha _{t-k}^{2} \sum_{l=1}^{k}\frac{1}{N} \sumn \left\| g_{t-l}\i(\theta_{t-l}\i) \right\|^{2}                                                                                                     \\\notag
		 & \le k\alpha _{t-k}^{2} \sum_{l=1}^{k}\frac{1}{N} \sumn 2\left( \left\| g_{t-l}\i(\theta_{t-l}\i) - g_{t-l}\i(\bar{\theta}_{t-l})\right\|^{2} + \left\| g_{t-l}\i(\bar{\theta}_{t-l}) \right\|^{2} \right)
	\end{align}
	where we also used Jensen's inequality.
	By the definition of the Markovian semi-gradients (see \cref{def:semi-grad}), they are linear and Lipschitz continuous with the Lipschitz constant bounded by $\|A\i_{t-l}\| \le 1+\gamma$. 
	However, it is worth emphasizing that the mean-path semi-gradients are non-linear and non-Lipschitz continuous (unless $\|\bar{\theta}_t\|$ is bounded; see \cref{eq:Lip-mean-path}).
	Given the Lipschitz continuity, we have
	\begin{align}\notag
		\Omega_{t}
		 & \le 2k\alpha _{t-k}^{2} \sum_{l=1}^{k}\frac{1}{N} \sumn \left( (1+\gamma)^2\left\|\theta_{t-l}\i - \bar{\theta}_{t-l} \right\|^{2} + H^2\right) \\\label{eq:drift-01}
		 & = 2k\alpha _{t-k}^{2}\left( kH^2 + (1+\gamma)^2\sum_{l=1}^{k}\Omega_{t-l} \right).
	\end{align}
	% Also, we have
	% \begin{equation} \label{eq:drift-02}
	% 	\begin{aligned}
	% 		\|\bar{\theta}_t - \bar{\theta}_{t-1}\|
	% 		=   & \alpha _{t-1}\left\| \frac{1}{N}\sumn g_{t-1}(\bm{\theta}_{t-1}) \right\|                                                                                                      \\
	% 		\le & \alpha _{t-1}R + \frac{\alpha _{t-1}}{N}\sumn\left\|  A\i\left(O\it\right)\bar{\theta}_{t-1} + A\i\left(O\it\right)\left( \theta\i_{t-1} - \bar{\theta}_{t-1} \right) \right\| \\
	% 		\le & \alpha _{t-1}(R + 2\|\bar{\theta}_{t-1}\|) + 2\alpha _{t-1}\omega_{t-1},
	% 	\end{aligned}
	% \end{equation}
	% where $\omega_{t} = \frac{1}{N}\sumn \left\| \theta\it - \bar{\theta}_t \right\|$. Since the square root is a concave function, we have $\omega_t \le \sqrt{\Omega_t}$.
	% Inequality~\eqref{eq:drift-02} gives
	% \begin{equation} \label{eq:drift-03}
	% 	\|\bar{\theta}_t\|
	% 	\le (1 + 2\alpha _{t-1})\|\bar{\theta}_{t-1}\| + \alpha _{t-1}(R + 2\omega_{t-1}).
	% \end{equation}
	Recursively applying \eqref{eq:drift-01} gives
	\begin{align}\notag
		\Omega_{t}
		\le & 2k\alpha _{t-k}^{2} \left( kH^{2} + (1+\gamma)^2\sum_{l=2}^{k}\Omega_{t-l} \right)
		+ 2k\alpha _{t-k}^{2} (1+\gamma)^2 \cdot 2(k-1)\alpha _{t-k}^{2}\left( kH^{2} + (1+\gamma)^2\sum_{l=2}^{k}\left( \Omega_{t-l} \right) \right) \\\notag
		\le & 2k\alpha _{t-k}^{2} \left( 1 + 8(k-1)\alpha _{t-k}^{2} \right)\left( kH^{2} + (1+\gamma)^2\sum_{l=2}^{k}\Omega_{t-l} \right)            \\\notag
		\le & 2k\alpha _{t-k}^{2}\prod_{j=1}^{k} \left( 1 + 8(k-j)\alpha _{t-k}^{2} \right)\left( kH^{2} + (1+\gamma)^2\Omega_{t-k} \right)           \\\notag
		\le & 2k\alpha _{t-k}^{2}\left( 1 + 8k\alpha _{t-k}^{2} \right)^{k}\cdot kH^{2},
	\end{align}
	where we use the fact that $\Omega_{t-k} = 0$.
	To continue, we impose a constraint on the initial step-size by
	requiring $4K\alpha_0 \le 1$, which gives $16k^2\alpha _{t-k}^2 \le 1$. Then, we have
	\begin{equation}\label{eq:bi}
		(1 + 8k\alpha _{t-k}^2)^{k}
		\le 1 + \sum_{l=1}^{k}k^{l}(8k\alpha _{t-k}^2)^{l}
		\le 1 + \frac{8k^2\alpha _{t-k}^2}{1-8k^2\alpha _{t-k}^2}
		\le 1 + 16k^2\alpha^2_{t-k} \le 2.
	\end{equation}
	Therefore, we get
	\begin{align}\label{eq:drift-1}
		\Omega_{t}
		\le 2k^2H^2 \alpha_{t-k}^2 (1 + 16k^2\alpha _{t-k}^2) 
		\le 4k^2H^2\alpha_{t-k}^2 \le \alpha _{t-k}^2 C^2_{\mathrm{drift}}.
	\end{align}
	Plugging \eqref{eq:drift-1} back into \eqref{eq:drift-0} gives the final result.
\end{proof}

\begin{corollary-lem}\label{cor:drift}
For future reference, we extract two bounds on the client drift from the proof of Lemma \ref{lem:drift}:
\[
	\Omega_t \le \alpha _{t-k}^2C^2_{\mathrm{drift}}, \quad \omega_t \le \alpha _{t-k} C_{\mathrm{drift}},
\]
where $\omega_t \coloneqq \frac{1}{N}\sumn\|\theta\it-\bar{\theta}_t\|$.
\end{corollary-lem}

% We denote $h(\theta) \coloneqq R + (1+\gamma)\|\theta\|$ in case we do not know if $\|\theta\|\le G$ or not.

\subsection{Gradient Progress} \label{sec:grad-prog}

To bound the gradient progress, we first need to bound the parameter progress.
Instead of directly bounding the client parameter progress, we bound the central parameter progress, which then gives the client parameter progress combining Lemma~\ref{lem:drift}.

\begin{lemma}[Central parameter progress]\label{lem:v-para-prog}
	If $\left\| \bar{\theta}_l \right\| \le G$ for any $l \le t$, then we have
	\[
		\left\| \bar{\theta}_{t} - \bar{\theta}_{t-\tau} \right\|
		% \le \alpha _{sK}(V_2 + V_3\left\| \bar{\theta}_t \right\| )
		\le \alpha_{sK}C_{\mathrm{prog}}(\tau),
	\]
	where $s$ is the largest integer such that $sK \le t-\tau$ and
	% where $V_{2} = 2(\tau + 2K + 2^{\tau + 2K})(R + 2\alpha_0 \sqrt{C_{\mathrm{drift}}}), V_3 = 2^{\tau+2K + 2}K,  V = V_2 + V_{3}G$.
	$$C_{\mathrm{prog}}(\tau) =  2(\tau+2K) (H + 2\alpha_{0}C_{\mathrm{drift}}) = O(\tau).$$
\end{lemma}
\begin{proof}
	Bounding the central parameter progress is harder than bounding the client parameter progress since
	\[
		\|\bar{\theta}_{t} - \bar{\theta}_{t-1}\| \not\le \alpha_{t-1}(R + (1+\gamma)\|\bar{\theta}_{t-1}\|).
	\]

	Therefore we need to introduce the client drift, and then bound the parameter progress using Lemma \ref{lem:drift}.
	% Also notice that in Lemma \ref{lem:drift}, we map the client drift to a previous parameter $\|\bar{\theta}_{t-k}\|$ instead of the current parameter; thus, together with this lemma, we can map the client drift to the current parameter.
	First, for any $t$, we have
	\begin{align}\notag
		\|\bar{\theta}_t\| \le \|\breve{\theta}_t\|
		 & = \left\| \bar{\theta}_{t-1} + \frac{\alpha _{t-1}}{N}\sumn g\i_{t-1}\left( \theta\i_{t-1} \right) \right\|                                                                                                                                    \\\notag
		 & = \left\| \bar{\theta}_{t-1} + \frac{\alpha _{t-1}}{N}\sumn \left( A\i\left( O\i_{t-1} \right)\bar{\theta}_{t-1} + A\i\left( O\i_{t-1} \right)\left( \theta\i_{t-1} - \bar{\theta}_{t-1} \right) + b\i\left( O\i_{t-1} \right)\right) \right\| \\
		 & \le (1+\alpha _{t-1}(1+\gamma))\|\bar{\theta}_{t-1}\| + \alpha_{t-1}(R + 2 \omega_{t-1}), \label{eq:v-prog-0}
	\end{align}
	where $\omega_{t-1}$ is defined in \cref{cor:drift}.
	Let $k$ be the smallest positive integer such that $t-k \equiv 0 \pmod K$ (if $t\equiv0\pmod K$, then $k=K$). Recursively applying \cref{eq:v-prog-0} gives
	\begin{align}\notag
		\|\bar{\theta}_t\|
		\le & \prod_{l=t-k}^{t-1}(1 + 2\alpha _{l}) \left\| \bar{\theta}_{t-k} \right\| + \sum_{j=0}^{k-1}(1 + 2\alpha _{t-j})^{j}\alpha _{t-j-1}(R + 2\omega_{t-1})           \\
		\label{eq:v-prog-02}
		\le & (1 + 2\alpha_{t-k})^{k}\left\| \bar{\theta}_{t-k} \right\| + \alpha _{t-k}(R + 2\alpha _{t-k}C_{\mathrm{drift}}) \frac{(1+2\alpha _{t-k})^{k}-1}{2\alpha _{t-k}} \\
		\label{eq:v-prog-03}
		\le & (1 + 4k\alpha_{t-k})\left\| \bar{\theta}_{t-k} \right\| + 2k\alpha _{t-k}(R + 2\alpha _{t-k}C_{\mathrm{drift}})                                                  \\
		\label{eq:v-prog-1}
		\le & 2\|\bar{\theta}_{t-k}\| + 2k\alpha _{t-k}(R + 2\alpha _{t-k}{C_{\mathrm{drift}}}),
	\end{align}
	where \eqref{eq:v-prog-02} uses \cref{cor:drift} and we require $\alpha$ to be non-increasing;
	and in \eqref{eq:v-prog-03} and \eqref{eq:v-prog-1}, we require that $4\alpha_{0}K \le 1$, which gives \((1 + 2 \alpha _{t-k})^{k} \le 1 + 4k\alpha _{t-k}\) with the similar reasoning in \cref{eq:bi}.

	% For future convenience, we give another bound of the magnitude of previous parameters by the magnitude of the current parameter.
	% With \eqref{eq:drift-02}, similar to \eqref{eq:para-prog-1}, we have
	% \begin{align}\notag
	% 	\|\bar{\theta}_{t-k}\|
	% 	\le & (1 - 2\alpha _{t-k})^{-1}\left( \|\bar{\theta}_{t-k+1}\| + \alpha _{t-k}(R + 2\omega_{t-k+1}) \right)                                             \\\notag
	% 	\le & (1 - 2\alpha _{t-k})^{-k} \left\| \bar{\theta}_{t} \right\| + \sum_{j=0}^{k-1}(1 - 2\alpha _{t-k})^{-(j+1)}\alpha _{t-k+j}(R + 2{\omega_{t-k+j}}) \\\label{eq:v-prog-2}
	% 	\le & 2\|\bar{\theta}_{t}\| + 2k\alpha _{t-k}\Bigl(R + 2\alpha _{t-k}\sqrt{C_{\mathrm{drift}}}\Bigr),
	% \end{align}
	% where we require that $4k\al_{t-k} \le 4K \alpha_0 \le 1$; then by Bernoulli inequality, we have
	% \[
	% 	(1 - 2\alpha _{t-k})^{-k} \le (1 + 4\alpha _{t-k}) \le 2.
	% \]

	Now we are ready to bound any central parameter progress between two aggregation steps.
	Since $t-k\equiv 0\pmod K$, we have $\|\bar{\theta}_{t-k}\|\le \bar{G}$. Then by \cref{eq:proj-prop}, we get
	\begin{align}\notag
		\|\bar{\theta}_t - \bar{\theta}_{t-k}\|
		\le                                & \left\| \breve{\theta}_t - \bar{\theta}_{t-k} \right\|
		\le \sum_{l=t-k}^{t-1}\|\bar{\theta}_{l+1}-\bar{\theta}_l\|                                                                                                                                \\\notag
		\stackrel{\cref{eq:v-prog-0}}{\le} & \sum_{l=t-k}^{t-1} \alpha_l\left( R + (1+\gamma)\|\bar{\theta}_l\| + 2\omega_l \right)                                                                \\\notag
		\stackrel{\cref{eq:v-prog-1}}{\le} & k\alpha _{t-k}(R + 2(1+\gamma)\|\bar{\theta}_{t-k}\| + 4k \alpha _{t-k}(R + 2\alpha _{t-k}{C_{\mathrm{drift}}}) + 2\alpha _{t-k}{C_{\mathrm{drift}}}) \\\label{eq:v-prog-2}
		\le                                & 2k\alpha _{t-k}\left( R + (1+\gamma)\|\bar{\theta}_{t-k}\| + 2\alpha _{t-k}C_{\mathrm{drift}} \right),
	\end{align}
	where we use the fact that $\gamma<1$ and $4k\alpha _{t-k}\le 1$.
	% where
	% \[
	% 	V_1 = R + 2\alpha_0 \sqrt{C_{\mathrm{drift}}}.
	% \]

	Finally, we need to bound the central parameter progress for general time period $\tau$. For any $t>\tau >1$, let $s$ be the largest integer such that $sK \le t-\tau$. And let $s'$ be the largest integer such that $s'K \le t$. Then we have
	\begin{align}\notag
		\|\bar{\theta}_t - \bar{\theta}_{t-\tau}\|
		\le                                & \sum_{j=1}^{s'-s}\|\bar{\theta}_{(s+j)K} - \bar{\theta}_{(s+j-1)K}\| + \|\bar{\theta}_{t} - \bar{\theta}_{s'K}\| + \|\bar{\theta}_{t-\tau} - \bar{\theta}_{sK}\| \\\notag
		\stackrel{\cref{eq:v-prog-2}}{\le} & 2(\tau+2K)\alpha _{sK}(R + 2\alpha _{sK}C_{\mathrm{drift}})                                                                                                      \\\notag
		                                   & + 2(1+\gamma)\alpha _{sK}\left( \sum_{j=1}^{s'-s}K\|\bar{\theta}_{(s+j-1)K}\| + (t-s'K)\|\bar{\theta}_{s'K}\| + (t-\tau-sK)\|\bar{\theta}_{sK}\|\right)          \\\notag
		\le                                & 2\alpha _{sK}(\tau+2K)(R + 2\alpha _{sK}C_{\mathrm{drift}}) + 2\alpha _{sK}(\tau + 2K)(1+\gamma)G                                                                \\\notag
		\le                                & 2\alpha _{sK}(\tau + 2K)\left( R + (1+\gamma)G + 2\alpha _{sK}C_{\mathrm{drift}} \right)                                                                         \\\label{eq:v-prog-3}
		\le                                & \alpha _{sK} C_{\mathrm{prog}}(\tau),
	\end{align}
	where $C_{\mathrm{prog}}(\tau) \coloneqq 2(\tau+2K) (H + 2\alpha_{0}C_{\mathrm{drift}}) = O(\tau)$.

	% By \eqref{eq:v-prog-2}, we get
	% \[
	% 	\begin{aligned}
	% 		\|\bar{\theta}_{(s'-j)K}\|
	% 		\le & 2\|\bar{\theta}_{(s'-j+1)K}\| + 2K \alpha_{sK}V_1                      \\
	% 		\le & 2^{j}\|\bar{\theta}_{s'K}\| + 2K \alpha _{sK}V_1 \sum_{l=0}^{j-1}2^{l} \\
	% 		\le & 2^{j+1}(\|\bar{\theta}_{t}\| + 2K \alpha _{sK}V_1).
	% 	\end{aligned}
	% \]

	% Plugging back into \eqref{eq:v-prog-3}, we get
	% \[
	% 	\begin{aligned}
	% 		    & \|\bar{\theta}_t - \bar{\theta}_{t-\tau}\|                                                                                                                   \\
	% 		\le & \alpha _{sK}(1 + 2K \alpha _{sK})(\tau + 2K)V_1 + K\alpha _{sK}\left( \sum_{j=1}^{s'-s} 2^{j+1}\left(\|\bar{\theta}_{t}\| + 2\alpha _{sK}KV_1 \right)\right. \\
	% 		    & \left. + 2^{s'-s+1}(\|\bar{\theta}_{t}\| + 2\alpha _{sK}KV_1)+ 2(\|\bar{\theta}_{t}\| + 2\alpha _{sK}KV_1)\right)                                            \\
	% 		\le & \alpha _{sK}(1+2K\alpha _{sK})(\tau + 2K)V_1 + 2^{s'-s+2}K\alpha _{sK}\left( \|\bar{\theta}_t\| + 2\alpha _{sK}K V_1 \right)                                 \\
	% 		\le & \alpha _{sK}\left(V_2(\tau) + V_3(\tau) \left\| \bar{\theta}_t \right\| \right),
	% 	\end{aligned}
	% \]
	% where
	% \[
	% 	V_2(\tau) = 2(\tau + 2K + 2^{\tau + 2K})V_1, \quad V_3(\tau) = 2^{\tau + 2K + 2}K,
	% \]
	% where we use the fact that $2K\alpha _{sK}\le 2K \alpha_0 \le 1/2$.

	% Furthermore, if $\left\| \bar{\theta}_t \right\| \le G$, we have
	% \[
	% 	\left\|\bar{\theta}_t - \bar{\theta}\tta\right\| \le \alpha _{sK} V(\tau),
	% \]
	% where $V(\tau) = V_2(\tau) + V_3(\tau)G$.

\end{proof}

\begin{corollary-lem}[Client parameter progress] \label{cor:para-prog}
If $\left\| \bar{\theta}_{l} \right\| \le G$ holds for all $l\le t$, we also have
\[
	\left\| \theta\it - \theta\ita \right\| \le \alpha _{sK}C_{\mathrm{prog}}(\tau),
\]
where $s$ is the largest integer such that $sK\le t-\tau$.
\end{corollary-lem}
\begin{proof}
	If $t\equiv 0\pmod K$ and $t-\tau\equiv 0\pmod K$, then $\theta\it = \bar{\theta}_t$ and $\theta\ita = \bar{\theta}\tta$, and the result directly follows \cref{lem:v-para-prog}.
	Without loss of generality, we assume $t\not\equiv 0\pmod K$ and $t-\tau\not\equiv 0\pmod K$.
	Let $s$ be the largest integer such that $sK < t-\tau$. And let $s'$ be the largest integer such that $s'K < t$. Similar to \eqref{eq:v-prog-3}, we have
	\[
		\begin{aligned}
			\left\|\theta\it - \theta\ita\right\|
			\le & \|\bar{\theta}_{s'K} - \bar{\theta}_{sK}\| + \left\|\theta\it - \bar{\theta}_{s'K}\right\| + \left\|\theta\ita - \bar{\theta}_{sK}\right\|. \\
		\end{aligned}
	\]
	By \cref{lem:v-para-prog}, we have
	\[
		\left\| \bar{\theta}_{s'K} - \bar{\theta}_{sK} \right\| \le \alpha_{sK}C_{\mathrm{prog}}(s'K-sK-2K),
	\]
	where we subtract $2K$ to offset the addition of $2K$ in \cref{lem:v-para-prog} for general $t$ and $\tau$.

	Then to bound the client parameter progress after a synchronization, we first notice that when $t\not\equiv 0\pmod K$, we have
	\[
		\left\| \theta\it - \theta\i_{t-1} \right\|
		\le 2\alpha _{t-1}\left\| \theta\i_{t-1} \right\| + \alpha _{t-1}R.
	\]
	Similar to \eqref{eq:v-prog-0}-\eqref{eq:v-prog-1}, we have
	\[
		\begin{aligned}
			\left\| \theta\it \right\|
			\le & (1 + 2\alpha _{t-1})\left\| \theta\it \right\| + \alpha _{t-1}R                                                                      \\
			\le & \prod_{l=t-k}^{t-1}(1 + 2\alpha _{l}) \left\| \bar{\theta}_{t-k} \right\| + R\sum_{j=0}^{k-1}(1 + 2\alpha _{t-j})^{j}\alpha _{t-j-1} \\
			\le & 2\left( \|\bar{\theta}_{t-k}\| + k\alpha _{t-k}R \right),
		\end{aligned}
	\]
	where $k$ is the smallest integer such that $t-k\equiv 0\pmod K$.
	Then, we get
	\[
		\begin{aligned}
			\left\| \theta\it - \theta\i_{t-k} \right\|
			\le & \sum_{l=t-k}^{t-1} \left\| \theta\i_{l+1} - \theta\i_{l} \right\| \\		\le & \sum_{l=t-k}^{k-1} \alpha_l\left( (1+\gamma)\left\| \theta_l\i \right\| + R \right)                  \\		\le & k\alpha _{t-k} \left( 2(1+\gamma)\left( \|\bar{\theta}_{t-k}\| + k\alpha _{t-k}R \right) + R \right) \\
			\le & 2k\alpha _{t-k}\left(R + (1+\gamma)\|\bar{\theta}_{t-k}\|\right),
		\end{aligned}
	\]
	where we use the fact that $4k\alpha _{t-k}\le 1$.
	Therefore, we have
	\[
		\begin{aligned}
			\left\| \theta\it - \theta_{s'K} \right\| & \le 2(t-s'K) \alpha _{s'K}H \le 2\alpha _{sK}KH,    \\
			\left\| \theta\ita - \theta_{sK} \right\| & \le 2(t-\tau-sK) \alpha _{sK}H \le 2\alpha _{sK}KH. \\
		\end{aligned}
	\]
	Putting all together gives
	\[
		\left\| \theta\it - \theta\ita \right\| \le \alpha _{sK}(C_{\mathrm{prog}}(s'K-sK-2K) + 4KH) \le \alpha _{sK}C_{\mathrm{prog}}(\tau).
	\]
\end{proof}

With the above corollary, we are ready to bound the gradient progress.

\begin{corollary-lem}[Graident progress]\label{lem:grad-prog}
If $\left\| \bar{\theta}_{l} \right\| \le G$ holds for all $l\le t$, then for any $\theta$, we have
\[
	\left\| \bar{g}\ita(\theta) - \bar{g}\it (\theta) \right\|
	\le L\sigma h(\theta)\alpha _{sK}C_{\mathrm{prog}}(\tau),
\]
where $s$ is the largest integer such that $sK\le t-\tau$.
\end{corollary-lem}
\begin{proof}
	\[
		\left\| \bar{g}\ita(\theta) - \bar{g}\it (\theta) \right\|
		= \left\| \left(\bar{A}\ita - \bar{A}\it \right)\theta + \bar{b}\ita - \bar{b}\it \right\|
		\le L\sigma\left( R + (1+\gamma)\|\theta\| \right)\left\| \theta\ita - \theta\it \right\|,
	\]
	where the inequality uses Lemma~\ref{lem:td-hetero}.
	Then we get the desired result by plugging in \cref{cor:para-prog}.
\end{proof}

The third corollary of Lemma~\ref{lem:v-para-prog} is the expression of $G$, which was stated as an assumption in previous lemmas.

\begin{corollary-lem}[Parameter bound]\label{cor:G}
Given the explicit projection $\Pi_{\bar{G}}$, for any $t\in\NN$, we have
\[
	\left\| \bar{\theta}_t \right\| \le G \coloneqq \frac{2(2\bar{G} + R)}{1-16\alpha_0 ^2K^2\gamma} \le \frac{2(2\bar{G}+R)}{1-\gamma}.
\]
\end{corollary-lem}
\begin{proof}
	For any $t\in\mathbb{N}$, by \eqref{eq:v-prog-1} in Lemma~\ref{lem:v-para-prog}, we have
	\[
		\left\| \bar{\theta}_t \right\|
		\le 2\left( \bar{G} + K\alpha_0 (R + 2\alpha_0 C_{\mathrm{drift}}) \right).
	\]
	Plugging the expression of $C_{\mathrm{drift}}$ in \cref{lem:drift} into the above inequality gives the recursive definition:
	$$
		G = 2(\bar{G} + \alpha_{0}K(R + 2\alpha_0 \cdot 2K(R + (1+\gamma)G))).
	$$
	Note that we require $4K\alpha_0\le 1$ in \cref{lem:v-para-prog}. Thus, we have
	$$
		G \le 2\bar{G} + R + 8\alpha_{0}^2K^2(1+\gamma)G,
	$$
	which gives
	$$
		G \le \frac{2(2\bar{G} + R)}{1-16\alpha^2_{0}K^2\gamma}.
	$$
	Therefore, we let $G\coloneqq 2(2\bar{G} + R) /(1-16\alpha_0 ^2K^2\gamma)$; and then, we have
	$$
		\left\| \bar{\theta}_t \right\| \le 2\bar{G} + R + 8\alpha_0 ^2 K^2(1+\gamma)G \le  G.
	$$
\end{proof}

%%%%%%%%%%%%%%%%%%%%%%%%%%%%%%%%%%%%%%%%%%%%%%%%%%%%%%%%%%%%%%%%%%%%%%%%%%%%%%%%%%%%%%%%%
\subsection{Mixing}\label{sec:mix}

Unlike stationary MDPs in TD(0) and off-policy Q-learning, the mixing process in our algorithm is a virtual process.
After backtracking, we fixed the policy as $\Gamma(\theta\ita)$, which then introduces a virtual stationary MDP. We denote $\tilde{O}\it = (\tilde{S}\it,\tilde{U}\it,\tilde{S}\itt,\tilde{U}\itt)$ the observation of this virtual MDP at time step $t$.

\begin{lemma}[Mixing]\label{lem:mix} Let $\mathcal{F}\tta$ denote the filtration containing all preceding randomness up to time step $t-\tau$. For any deterministic $\theta$ conditioned on $\mathcal{F}\tta$---such as a constant parameter or a parameter determined by $\mathcal{F}\tta$---we have
	\[
		\left\|\EE\left[ g\ita(\theta,\tilde{O}\it)- \bar{g}\ita(\theta) \given \mathcal{F}\tta \right]\right\|
		\le m_{i}\rho_{i}^{\tau}h\left(\theta\right)
	\]
\end{lemma}
\begin{proof}
	We define a new TD operator:
	\begin{equation*}%\label{eq:mix-1}
		Z\ita\left(\theta,\too\it\right)
		\coloneqq g\ita\left( \theta,\too\it \right) - \bar{g}\ita\left( \theta \right).
	\end{equation*}
	Then, we have
	\[
		\begin{aligned}
			    & \left\|\EE \left[Z\ita\left( \theta, \tilde{O}\it \right) \given \mathcal{F}_{t-\tau} \right]\right\|                                                                                                                           \\
			=   & \left\|\EE \left[ g\ita\left( \theta,\too\it \right) \given \mathcal{F}_{t-\tau} \right] - \bar{g}\ita(\theta)\right\|                                                                                                          \\
			=   & \left\| \int_{\mathcal{S}^2\times \mathcal{A}^2}  \phi(s,a)\left( r\i(s,a) + \gamma\phi^T(s',a')\theta - \phi^T(s,a)\theta\right) \left( P\ita\left( \too\it=O\given\mathcal{F}\tta\right) - \varphi\ita(O) \right)\d O\right\| \\
			\le & \left(R + (1+\gamma)\left\| \theta \right\|\right) \cdot \left\| P\ita(\tilde{S}\it=\cdot \given\mathcal{F}\tta) -  \eta\ita\right\|_{\mathrm{TV}}                                                                              \\
			\le & m_{i}\rho_{i}^{\tau}\left( R + (1+\gamma)\left\| \theta\right\|  \right),
		\end{aligned}
	\]
	where the last inequality is by Assumption \ref{asmp:steady}.

\end{proof}

%%%%%%%%%%%%%%%%%%%%%%%%%%%%%%%%%%%%%%%%%%%%%%%%%%%%%%%%%%%%%%%%%%%%%%%%%%%%%%%%%%%%%%%%%%
\subsection{Backtracking}\label{sec:stat}

\begin{lemma}[Backtracking] \label{lem:stat}
	If $\|\bar{\theta}_l\| \le G$ for all $l\le t$, then for any deterministic $\theta$ conditioned on $\mathcal{F}\tta$, we have
	\[
		\left\|\EE\left[ g\it(\theta, O\it) - g\i\tta(\theta, \tilde{O}\it)\given \mathcal{F}\tta \right]\right\|
		\le  \alpha _{sK}C_{\mathrm{back}}(\tau)h\left(\theta\right),
	\]
	where
	$$C_{\mathrm{back}}(\tau) = \tau LC_{\mathrm{prog}}(\tau) = O(\tau^2).$$
\end{lemma}
\begin{proof}
	First, we have
	\[
		\begin{aligned}
			    & \left\|\EE\left[ g\it(\theta, O\it) - g\i\tta(\theta, \tilde{O}\it)  \given \mathcal{F}\tta\right]\right\|                                                                                                         \\
			\le & \left( R + (1+\gamma)\left\| \theta \right\|\right) \left\| P\i_{\theta\it}(O\it = \cdot \mid \mathcal{F}_{t-\tau}) -  P\i_{\theta\i_{t-\tau}}(\tilde{O}\it = \cdot \mid \mathcal{F}_{t-\tau}) \right\|_{\mathrm{TV}}.
		\end{aligned}
	\]
	% where $\mathcal{F}_{t}$ is the filtration containing all randomness prior to time step $t$.
	For a specific client, for notation simplicity, we omit the superscript $(i)$ and denote $P_{\theta_t}$ by $P_{t}$. Let $O = (s,a,s',a')$; then we have
	\[
		\begin{aligned}
			P_{t}(O_t=O|\mathcal{F}_{t-\tau})
			= & \int_{\Theta^2}P_t(S_t=s,U_t=a,S_{t+1}=s',U_{t+1}=a',\theta_{t-1}=\theta,\theta_t=\theta'|\mathcal{F}\tta)\d\theta\d\theta'           \\
			= & \int_{\Theta^2}P_t(S_t=s|\mathcal{F}\tta)
			\cdot P_t(\theta_{t-1}=\theta|\mathcal{F}\tta,S_t=s)                                                                                      \\
			  & \quad\ \cdot P_t(U_t=a|\mathcal{F}\tta,S_t=s,\theta_{t-1}=\theta)                                                                     \\
			  & \quad\ \cdot P_t(S_{t+1}=s'|\mathcal{F}\tta,S_t=s,\theta_{t-1}=\theta,a_t=a)                                                          \\
			  & \quad\ \cdot P_t(\theta_t=\theta'|\mathcal{F}_{t-\tau},S_t = s,\theta_{t-1}=\theta,U_t = a,S_{t+1}=s')                                \\
			  & \quad\ \cdot P_t(U_{t+1}=a'|\mathcal{F}_{t-\tau},S_t = s,\theta_{t-1}=\theta,U_t = a,S_{t+1}=s',\theta_t=\theta') \d \theta\d \theta' \\
			= & \int_{\Theta^2}P_t(S_t=s|\theta_{t-\tau},S_{t-\tau})
			\cdot P_t(\theta_{t-1}=\theta|\theta_{t-\tau},S_{t-\tau},S_t=s)
			\cdot \pi_{\theta}(a|s)                                                                                                                   \\
			  & \cdot P_{a}(s,s')\cdot P_t(\theta_t=\theta'|\theta_{t-\tau},S_{t-\tau},\theta_{t-1}=\theta,S_t = s,U_t = a)
			\cdot \pi_{\theta'}(a'|s')\d \theta\d \theta',
		\end{aligned}
	\]
	where we use that fact that $U_t$ is dependent on $\theta_{t-1}$ instead of $\theta_t$; and when $\theta_{t-1}$ is determined, $\theta_{t}$ is not dependent on $S_{t+1}$. Notice that for any $(s,s',a)\in\S^2\times \A$, we have
	$$
		\int_{\Theta^2}P_t(\theta_{t-1}=\theta|\mathcal{F}_{t-\tau},S_t=s)\cdot P_t(\theta_t=\theta'|\mathcal{F}_{t-\tau},\theta_{t-1}=\theta,S_t=s,U_t=a)\d\theta\d\theta' = 1.
	$$
	Thus, for $P_{t-\tau}(\tilde{O}|\mathcal{F}\tta)$, we have a similar expression:
	\begin{equation*}%\label{eq:stat-2}\tag{$S_2$}
		\begin{aligned}
			P_{t-\tau}(\tilde{O}_t=O|\mathcal{F}_{t-\tau})
			= & \int_{\Theta^2}P_{t-\tau}(\tilde{S}_t=s,\tilde{U}_t=a,\tilde{S}_{t+1}=s',\tilde{U}_{t+1}=a'|\mathcal{F}\tta) \cdot P_t(\theta_{t-1}=\theta|\mathcal{F}_{t-\tau},S_t=s) \\
			  & \cdot P_t(\theta_t=\theta'|\mathcal{F}_{t-\tau},\theta_{t-1}=\theta,S_t=s,U_t=a)\d\theta\d\theta'                                                                      \\
			= & \int_{\Theta^2}P_{t-\tau}(\tilde{S}_t=s|\theta_{t-\tau},S_{t-\tau})\cdot \pi_{\theta_{t-\tau}}(a|s) \cdot  P_{a}(s,s') \cdot \pi_{\theta_{t-\tau}}(a'|s')              \\
			  & \cdot P_t(\theta_{t-1}=\theta|\theta_{t-\tau},S_{t-\tau},S_t=s)
			\cdot P_t(\theta_t=\theta'|\theta_{t-\tau},S_{t-\tau},\theta_{t-1}=\theta,S_t=s,U_t=a)\d\theta\d\theta'
		\end{aligned}
	\end{equation*}

	Therefore, we decompose the observation distribution discrepancy as follows:
	\[
		\left\|P_t(O_t|\mathcal{F}_{t-\tau})- P_{t-\tau}(\tilde{O}_t|\mathcal{F}_{t-\tau})\right\|_{\tv}
		\le\underset{\scriptscriptstyle{\S^2\times \A^2}}{\int}\biggl( \underbrace{\Bigl| P_{t}(O_t = O|\mathcal{F}_{t-\tau})\!-\!Q_t(O) \Bigr|}_{S_1}\!
		+\!\underbrace{\left| Q_t(O)\!-\!P_{t-\tau}(\tilde{O}_t = O|\mathcal{F}_{t-\tau}) \right|}_{S_2}  \biggr)\!\d O,
	\]
	where
	\[
		\begin{aligned}
			Q_l(O) \coloneqq & \int_{\Theta^2}P_{t-\tau}(\tilde{S}_l=s|\theta_{t-\tau},S_{t-\tau})\cdot \pi_{\theta_{t-\tau}}(a|s) \cdot  P_{a}(s,s') \cdot \pi_{\theta'}(a'|s') \\
			                 & \cdot P_l(\theta_{l-1}=\theta|\theta_{t-\tau},S_{t-\tau},S_l=s)
			\cdot P_l(\theta_l=\theta'|\theta_{t-\tau},S_{t-\tau},\theta_{l-1}=\theta,S_l=s,U_l=a)\d\theta\d\theta'.
		\end{aligned}
	\]

	For $S_1$, we have
	\[
		\begin{aligned}
			    & \int_{\S^2\times \A^2}\left| P_{t-\tau}(\tilde{O}_t=O|\mathcal{F}_{t-\tau}) - Q_t(O) \right|\d O                                                                                                               \\
			\le & \int_{\S^2\times \A^2\times \Theta^2} P_{t-\tau}(\tilde{S}_t=s|\theta_{t-\tau},S_{t-\tau}) \pi_{\theta_{t-\tau}}(a|s) P_{a}(s,s') P_t(\theta_{t-1}=\theta|\theta_{t-\tau},S_{t-\tau},S_t=s)                    \\
			    & \cdot P_t(\theta_t=\theta'|\theta_{t-\tau},S_{t-\tau},\theta_{t-1}=\theta,S_t=s,U_t=a) \left| \pi_{\theta_{t-\tau}}(a'|s') - \pi_{\theta'}(a'|s') \right| \d O\d\theta\d\theta'                                \\
			=   & \int_{\S^2\times \A\times \Theta^2} P_{t-\tau}(\tilde{S}_t=s|\theta_{t-\tau},S_{t-\tau}) \pi_{\theta_{t-\tau}}(a|s)   P_{a}(s,s') P_t(\theta_{t-1}=\theta|\theta_{t-\tau},S_{t-\tau},S_t=s)                    \\
			    & \cdot P_t(\theta_t=\theta'|\theta_{t-\tau},S_{t-\tau},\theta_{t-1}=\theta,S_t=s,U_t=a) \cdot \left\|\pi_{\theta_{t-\tau}}(\cdot |s') - \pi_{\theta'}(\cdot |s') \right\|_{\tv} \d s\d s'\d a\d\theta\d\theta'.
		\end{aligned}
	\]
	By the Lipschitzness of the policy improvement operator (see \cref{asmp:lip}), we know
	\[
		\sup_{s'\in\S}\left\| \pi_{\theta_{t-\tau}}(\cdot |s') -  \pi_{\theta'}(\cdot |s') \right\|_{\mathrm{TV}} \le L\left\| \theta_{t-\tau} - \theta' \right\|.
	\]
	Then for any $\theta'\in\Theta$ for which $P_t(\theta_t=\cdot |\mathcal{F}_{t-\tau})$ has non-zero density, meaning that $\theta'$ is reachable at time step $t$, Corollary \ref{cor:para-prog} implies
	\[
		\left\| \theta_{t-\tau}\i - \theta' \right\| \le \alpha _{sK}C_{\mathrm{prog}}(\tau),
	\]
	where $s$ is the largest integer such that $sK \le t-\tau$.

	Therefore, we have
	\[
		\int_{\S^2\times \A^2}\left| P_{t-\tau}(\tilde{O}_t=O|\mathcal{F}_{t-\tau}) - Q_t(O) \right|\d O
		\le \alpha _{sK}LC_{\mathrm{prog}}(\tau).
	\]

	For $S_2$, we have
	\[
		\begin{aligned}
			    & \int_{\S^2\times \A^2} \left| P_{t}(O_t=O|\mathcal{F}_{t-\tau}) - Q_t(O) \right| \d O                                                                                                     \\
			\le & \int_{\S^2\times \A^2\times \Theta^2}P_t(\theta_{t-1}=\theta|\theta\tta,S\tta,S_t=s)P_t(\theta_t=\theta'|\theta\tta,S\tta,S_t=s,U_t=a,\theta_{t-1}=\theta)\pi_{\theta'}(a'|s')P_{a}(s,s') \\
			    & \cdot \left| P\tta(\tilde{S}_t=s|\theta\tta,S\tta)\pi_{\theta\tta}(a|s) - P_t(S_t=s|\theta\tta,S\tta)\pi_{\theta}(a|s)\right| \d O\d\theta\d\theta'                                       \\
			=   & \int_{\S\times \A\times \Theta}P_t(\theta_{t-1}=\theta|\theta\tta,S\tta,S_t=s)                                                                                                            \\
			    & \cdot \left| P\tta(\tilde{S}_t=s|\theta\tta,S\tta)\pi_{\theta\tta}(a|s) - P_t(S_t=s|\theta\tta,S\tta)\pi_{\theta}(a|s)\right| \d s\d a\d\theta                                            \\
			\le & \int_{\S\times \A\times \Theta}P_t(\theta_{t-1}=\theta|\theta\tta,S\tta,S_t=s)                                                                                                            \\
			    & \cdot \left(\left| P\tta(\tilde{S}_t=s|\theta\tta,S\tta)\pi_{\theta\tta}(a|s) - P\tta(\tilde{S}_t=s|\theta\tta,S\tta)\pi_{\theta}(a|s)\right|\right.                                      \\
			    & + \left.\left| P\tta(\tilde{S}_t=s|\theta\tta,S\tta)\pi_{\theta}(a|s) - P_{t}(S_t=s|\theta\tta,S\tta)\pi_{\theta}(a|s)\right| \right)\d s\d a\d\theta                             \\
			\le & \sup_{s\in\S}\|\pi\tta(\cdot |s)-\pi_{\theta}(\cdot |s)\|_{\mathrm{TV}} + \left\| P\tta(\tilde{S}_{t} = \cdot |\mathcal{F}\tta) -  P_t(S_t=\cdot |\mathcal{F}\tta) \right\|_{\mathrm{TV}} \\
			\le & \alpha _{sK} LC_{\mathrm{prog}}(\tau-1)+ \left\| P\tta(\tilde{S}_{t} = \cdot |\mathcal{F}\tta) -  P_t(S_t=\cdot |\mathcal{F}\tta) \right\|_{\mathrm{TV}}.                                 \\
		\end{aligned}
	\]
	Substituting $S_1$ and $S_2$ with the above bounds gives
	\begin{equation}\label{eq:back-1}
		\left\|P_t(O_t|\mathcal{F}_{t-\tau})- P_{t-\tau}(\tilde{O}_t|\mathcal{F}_{t-\tau})\right\|_{\tv}
		\le \left\| P\tta(\tilde{S}_{t} = \cdot |\mathcal{F}\tta) -  P_t(S_t=\cdot |\mathcal{F}\tta) \right\|_{\mathrm{TV}} + \alpha _{sK}L\sum_{l=\tau-1}^{\tau}C_{\mathrm{prog}}(l).
	\end{equation}

	Applying a similar decomposition as $S_1$ and $S_2$, we can obtain an analogous bound to \cref{eq:back-1} for the state distribution discrepancy:
	\[
		\begin{aligned}
			    & \left\| P\tta(\tilde{S}_{t} = \cdot |\mathcal{F}\tta) -  P_t(S_t=\cdot |\mathcal{F}\tta) \right\|_{\mathrm{TV}}                                                 \\
			% \le & \left\| P\tta(\tilde{O}_{t-1} = \cdot |\mathcal{F}\tta) -  Q_{t-1}(\cdot) \right\|_{\mathrm{TV}}                                                                \\
			\le & \left\| P\tta(\tilde{S}_{t-1} = \cdot |\mathcal{F}\tta) -  P_{t-1}(S_{t-1}=\cdot |\mathcal{F}\tta) \right\|_{\mathrm{TV}} + \alpha _{sK} LC_{\mathrm{prog}}(\tau-2) \\
			\le & \left\| P\tta(\tilde{S}_{t-\tau} = \cdot |\mathcal{F}\tta) -  P\tta(S_{t-\tau}=\cdot |\mathcal{F}\tta) \right\|_{\mathrm{TV}}
			+ \sum_{l=1}^{\tau-2}\alpha _{sK} LC_{\mathrm{prog}}(l)                                                                                                               \\
			\le   & (\tau-2)\alpha _{sK} LC_{\mathrm{prog}}(\tau).
		\end{aligned}
	\]
	Putting this bound back into \cref{eq:back-1} gives
	\[
		\begin{aligned}
			\left\| P_t(O_t|\mathcal{F}\tta) -  P\tta(\tilde{O}_t|\mathcal{F}\tta) \right\|_{\mathrm{TV}}
			\le & \tau a_{sK}LC_{\mathrm{prog}}(\tau).
		\end{aligned}
	\]

	Finally, we get
	\begin{equation*}%\label{eq:stat-1}
		\begin{aligned}
			\left\|\EE\left[g\it(\theta, O\it) - g\i\tta(\theta, \tilde{O}\it) \given \mathcal{F}\tta\right]\right\|
			\le & \tau \alpha _{sK}LC_{\mathrm{prog}}(\tau)\left( R + (1+\gamma)\left\| \theta \right\|  \right).
		\end{aligned}
	\end{equation*}

\end{proof}

\subsection{Gradient Variance}\label{sec:grad-norm}

\begin{lemma}[Gradient variance] \label{lem:grad-norm}
	\[
		\EE\left\| \bm{g}_t(\bm{\theta}_t) \right\|^2
		\le 64\left( \EE\left\| \bar{\theta}_t - \theta_{*} \right\|^2 + \frac{\Lambda^2(\epsilon_p,\epsilon_r)}{w ^2}\right) + \alpha _{sK}^2C_{\mathrm{var}}(\tau) + 4m^2\rho^{2\tau}H^2 + \frac{32H^2}{N},
	\]
	where
	\[
		C_{\mathrm{var}}(\tau) = 4 \left( 4(3 + H^2L^2\sigma^2)C^2_{\mathrm{drift}} + 4H^2L^2\sigma^2C_{\mathrm{prog}}^2(\tau) + H^2C_{\mathrm{back}}^2(\tau) \right).
	\]
	% where $g_t(\bm{\theta}_t) = \frac{1}{N}\sumn g\it(\theta\it)$, $\left\| \bm{\theta}_t \right\|^2 = \frac{1}{N}\sumn \left\| \theta\it \right\|^2$, and
	% \[
	% h(\alpha\tta) = \dots 
	% \]
\end{lemma}
\begin{proof}
	Similar to Lemma \ref{lem:decomp}, we first decompose the gradient variance and establish the linear speedups for the backtracking and mixing terms.
	\begin{align}\label{eq:grad-norm-1}
		\left\| \bm{g}_t(\bm{\theta}_t) \right\|^2
		=   & \left\| \bm{g}_t(\bm{\theta}_t) - \bar{\bm{g}}(\bm{\theta}_{*}) \right\|^2                                                                                                  \\\notag
		=   & \left\| \bm{g}_t(\bm{\theta}_t) - \bm{g}_t(\bm{\theta}_{*}, \bm{O}_t)
		+ \bm{g}_t(\bm{\theta}_{*}, \bm{O}_t) - \bm{g}_{{t-\tau}}(\bm{\theta}_{*}, \tilde{\bm{O}}_{t})\right.                                                                             \\\label{eq:grad-norm-2}
		    & \left.\ + \bm{g}\tta(\bm{\theta}_{*}, \tilde{\bm{O}}_{t}) - \bar{\bm{g}}\tta(\bm{\theta}_{*})
		+ \bar{\bm{g}}\tta(\bm{\theta}_{*}) - \bar{\bm{g}}(\bm{\theta}_{*}) \right\|^2                                                                                                    \\\notag
		\le & \frac{4}{N} \sumn \Biggl(
		\underbrace{\left\| g\it\left(\theta\it\right) - g\it\left( \theta_{*}\i \right) \right\|^2}_{G_1}
		+\underbrace{\left\| \bar{g}\i\tta\left( \theta_{*}\i \right) - \bar{g}\left( \theta\i_{*} \right)\right\|^2}_{G_2, \text{ gradient progress}} \Biggr)                            \\\notag
		    & + 4\underbrace{\left\| \bm{g}_t\left(\bm{\theta}_{*}, \bm{O}_t\right) - \bm{g}\tta\left( \bm{\theta}_{*},\tilde{\bm{O}}_{t} \right) \right\|^2}_{G_3, \text{ backtracking}}
		+ 4\underbrace{\left\| \bm{g}\tta\left( \bm{\theta}_{*}, \tilde{\bm{O}}_{t} \right) - \bar{\bm{g}}\tta(\bm{\theta}_{*})\right\|^2}_{G_4, \text{ mixing}},
	\end{align}
	where \eqref{eq:grad-norm-1} uses the fact that $\bar{\bm{g}}(\bm{\theta}_{*}) = \frac{1}{N}\sumn \bar{g}\i\left( \theta\i_{*} \right) = 0$; and in \eqref{eq:grad-norm-2}, we denote $\bm{g}\tta(\bm{\theta}_{*},\tilde{\bm{O}}_{t}) = \frac{1}{N}\sumn g\i\tta\left( \theta\i_{*},\tilde{O}\it \right)$, and the same notation applies to other semi-gradients.

	By the Lipschitzness of semi-gradient $g\it$, $G_1$ is bounded by
	\[
		\begin{aligned}
			\left\| g\it\left(\theta\it\right) - g\it\left(\theta\i_{*}\right) \right\|^2
			\le & 4\left\| \theta\it - \theta\i_{*} \right\|^2
			\\
			\le & 12 \left( \left\| \theta\it - \bar{\theta}_{t} \right\|^2 + \left\| \bar{\theta}_t - \theta_{*} \right\|^2 + \left\| \theta\i_{*} - \theta_{*} \right\|^2 \right).
		\end{aligned}
	\]
	% Then by Corollary \ref{cor:drift} and Lemma \ref{lem:fix-drift}, we get
	% \[
	% 	\begin{aligned}
	% 		    & \frac{1}{N}\sumn \left\| g\it(\theta\it) - g\it(\theta\i_{*}) \right\|^2                                                        \\
	% 		\le & 12\left( \alpha _{t-k}^2 C_{\mathrm{drift}} + \left\| \bar{\theta}_t - \theta_{*} \right\|^2 + B(\epsilon_p,\epsilon_r)^2 \right)
	% 	\end{aligned}
	% \]

	By \cref{lem:distri-hetero}, $G_2$ is bounded by
	\[
		\begin{aligned}
			\left\| \bar{g}\i\tta\left( \theta\i_{*} \right) - \bar{g}\left( \theta\i_{*} \right) \right\|^2
			\le & \left(\left( R + (1+\gamma)\left\| \theta\i_{*} \right\|  \right)\left\| \mu\ita - \mu\i_{*} \right\|_{\mathrm{TV}}\right)^2                                                                                                         \\
			\le & H^2L^2\sigma^2\left\| \theta\i\tta - \theta\i_{*}\right\|^2                                                                                                                                                                          \\
			\le & 4H^2L^2\sigma^2 \left( \left\| \theta\ita -\bar{\theta}\tta \right\|^2 + \left\| \bar{\theta}\tta - \bar{\theta}_t \right\|^2 + \left\| \bar{\theta}_t - \theta_{*} \right\|^2 + \left\| \theta_{*}-\theta\i_{*} \right\|^2 \right).
		\end{aligned}
	\]

	Now we are left with $G_3$ and $G_4$. However, we only have the bound of their first moment by Lemma \ref{lem:stat} and \ref{lem:mix}.
	We first note that for a set of functions $\{ g_i \}_{i=1}^{N}$ such that $\|g_i\|_{\infty}\le a$ and independent random variables $\{ x_{i} \}_{i=1}^{N}$ such that $\|\EE g_i(x_i)\| \le b$, we have
	\begin{align}\notag
		\EE\|\bm{g}(\bm{x})\|^2
		= & \EE\left<\frac{1}{N}\sumn g_i(x_i), \frac{1}{N}\sumn g_i(x_i) \right>
		\\=&\frac{1}{N^2} \sumn \EE \|g_i(x_i)\|^2 + \frac{1}{N^2}\sum_{i\ne j}\left<\EE g_i(x_i), \EE g_j(x_j) \right>
		\notag
		\\\le& \frac{a^2}{N} + \frac{1}{N^2}\sumn\sum_{j=1}^{N}\|\EE g_i(x_i)\|\|\EE g_j(x_j)\|
		\notag
		\\\le & \frac{a^2}{N} + b^2. \label{eq:grad-norm-3}
	\end{align}

	By \eqref{eq:grad-norm-3} and Lemma \ref{lem:stat}, the expectation of $G_3$ is bounded by
	\begin{equation}\label{eq:grad-norm-4}
		\EE\left[\left\| \bm{g}_t\left(\bm{\theta}_{*}, \bm{O}_t\right) -\bm{g}\tta\left( \bm{\theta}_{*},\tilde{\bm{O}}_{t} \right) \right\|^2 \given \mathcal{F}_{t-\tau}\right]
		\le \frac{4H^2}{N} + \alpha _{sK}^2C_{\mathrm{back}}^2H^2.
	\end{equation}

	By \eqref{eq:grad-norm-3} and Lemma \ref{lem:mix}, the expectation of $G_4$ is bounded by
	\[
		\EE\left[\left\| \bm{g}\tta\left( \bm{\theta}_{*}, \tilde{\bm{O}}_{t} \right) - \bar{\bm{g}}\tta(\bm{\theta}_{*})\right\|^2\given \mathcal{F}_{t-\tau}\right]
		= \EE\left[\left\| \bm{Z}\tta(\bm{\theta}_{*},\tilde{\bm{O}}_{t}) \right\|^2\given \mathcal{F}_{t-\tau}\right]
		\le \frac{4H^2}{N} + m^2\rho^{2\tau}H^2.
	\]
	Combining all together with Lemma~\ref{lem:drift}, \ref{lem:v-para-prog}, and Theorem~\ref{thm:fix-drift}, we get
	\[
		\begin{aligned}
			\EE\left[\|\bm{g}_t(\bm{\theta}_t)\|^2|\mathcal{F}_{t-\tau}\right]
			\le & 4\left( 4(3 + H^2L^2\sigma^2)\left( \EE\left[ \left\| \bar{\theta}_t-\theta_{*} \right\|^2 |\mathcal{F}\tta\right] + \alpha _{sK}^2C^2_{\mathrm{drift}} + \frac{\Lambda^2(\epsilon_p,\epsilon_r)}{w^2} \right)  \right.   \\
			    & \left. + 4H^2L^2\sigma^2 \alpha _{sK}^2C_{\mathrm{prog}}^2 + \frac{8H^2}{N} + \left( \alpha _{sK}^2C_{\mathrm{back}}^2 + m^2\rho^{2\tau} \right)H^2\right)                                                                \\
			\le & 64\left( \EE\left[ \left\| \bar{\theta}_t - \theta_{*} \right\|^2 |\mathcal{F}\tta\right] + \frac{\Lambda^2(\epsilon_p,\epsilon_r)}{w^2}\right) + \alpha _{sK}^2C_{\mathrm{var}} + 4m^2\rho^{2\tau}H^2 + \frac{32H^2}{N},
		\end{aligned}
	\]
	where we use the fact that $LH\sigma \le w \le 1$ required by \eqref{eq:L}, and
	\[
		C_{\mathrm{var}} = 4 \left( 4(3 + H^2L^2\sigma^2)C^2_{\mathrm{drift}} + 4H^2L^2\sigma^2C_{\mathrm{prog}}^2 + H^2C_{\mathrm{back}}^2 \right).
	\]
	Finally, we get
	$$
		\EE\|\bm{g}_t(\bm{\theta}_t)\|^2 = \EE\left[ \EE\left[ \|\bm{g}_t(\bm{\theta}_t)\|^2 | \mathcal{F}_{t-\tau} \right] \right] \le
		64\left( \EE\left\| \bar{\theta}_t - \theta_{*} \right\|^2 + \frac{\Lambda^2(\epsilon_p,\epsilon_r)}{w^2}\right) + \alpha _{sK}^2C_{\mathrm{var}} + 4m^2\rho^{2\tau}H^2 + \frac{32H^2}{N}.
	$$
\end{proof}

Recall that \cref{lem:v-para-prog} bounds the central parameter progress, which gives a \textit{naive} bound of the mean square central parameter progress $\EE\|\bar{\theta}_t - \bar{\theta}\tta\|^2\le \alpha^2_{sK}C^2_{\mathrm{prog}}(\tau)$ for any $\tau\le t$, where $s$ is the largest integer such that $sK\le t-\tau$.
However, with the help of \cref{lem:grad-norm}, we can derive a tighter bound of the mean square central parameter progress, which is essential for proving \cref{thm} later.

\begin{corollary-lem}[Mean square central parameter progress]\label{cor:msv}
For any $\tau\le t$, we have
$$
	\EE \left\| \bar{\theta}_t-\bar{\theta}\tta \right\|^2 \le 4(\tau+K)(\tau+3K)\alpha^2_{sK}
	\left( 64\EE\left\| \bar{\theta}_{sK} - \theta_{*} \right\|^2 + V(\tau) \right),
$$
where $s$ is the largest integer such that $sK\le t-\tau$ and
$$
	V(\tau) \coloneqq \frac{64 \Lambda^2(\epsilon_{p},\epsilon_{r})}{w^2} + \alpha^2_{sK}C_{\mathrm{var}}(\tau) + 4m^2\rho^{2\tau}H^2 + \frac{32H^2}{N}
$$
is part of the gradient variance bound in \cref{lem:grad-norm}.
\end{corollary-lem}
\begin{proof}
	Recall in \cref{lem:v-para-prog}, we utilize a \textit{naive} bound of $\|\bm{g}_t(\bm{\theta}_t)\|$ by $(R +(1+\gamma)\|\bar{\theta}_t\| + 2\omega_l)$; the key difference in this proof is that we will bound $\EE\|\bm{g}_t(\bm{\theta}_t)\|^2$ using \cref{lem:grad-norm}.
	Therefore, similar to \cref{eq:v-prog-3}, let $s$ and $s'$ be the largest integer such that $sK\le t-\tau$ and $s'K\le t$ respectively. Then we have
	\begin{align}\notag
		\EE\|\bar{\theta}_t - \bar{\theta}_{t-\tau}\|^2
		\le & (s'-s+2)\left(\EE\|\bar{\theta}_{t} - \bar{\theta}_{s'K}\|^2 + \sum_{j=1}^{s'-s}\EE\|\bar{\theta}_{(s+j)K} - \bar{\theta}_{(s+j-1)K}\|^2 + \EE\|\bar{\theta}\tta - \bar{\theta}_{sK}\|^2 \right)       \\\notag
		\le & (s'-s+2)\left(\EE\|\breve{\theta}_{t} - \bar{\theta}_{s'K}\|^2 + \sum_{j=1}^{s'-s}\EE\|\breve{\theta}_{(s+j)K} - \bar{\theta}_{(s+j-1)K}\|^2 + \EE\|\breve{\theta}\tta - \bar{\theta}_{sK}\|^2 \right) \\\notag
		\le & 2(s'-s+2) K \sum_{l=sK}^{t-1}\alpha^2_{l}\EE \|\bm{g}_{l}(\bm{\theta}_{l})\|^2                                                                                                                              \\\notag
		\le & 2(\tau + 3K)\alpha^2_{sK} \sum_{l=sK}^{t-1}\EE\|\bm{g}_{l}(\bm{\theta}_{l})\|^2.
	\end{align}
	By \cref{lem:grad-norm}, we get
	\begin{equation}\label{eq:msv-1}
		\EE\|\bar{\theta}_t - \bar{\theta}_{t-\tau}\|^2
		\le 2(\tau+3K) \alpha^2_{sK}\sum_{l=sK}^{t-1} \left( 64\EE\|\bar{\theta}_{l}-\theta_{*}\|^2 + V(l-sK) \right).
	\end{equation}
	Then, similar to \cref{eq:v-prog-1}, we want to bound $\EE\|\bar{\theta}_{l}-\theta_{*}\|^2$ by $\EE\|\bar{\theta}_{sK}-\theta_{*}\|^2$ for $sK< l\le t-1$. We have
	\begin{align}\notag
		\mathbb{E}\left\|\bar{\theta}_{l}-\theta_*\right\|^2
		 & \le \mathbb{E}\left\|\breve{\theta}_{l}-\theta_*\right\|^2                                                    \\\notag
		 & =\mathbb{E}\| \bar{\theta}_{l\!-\!1}-\theta_*+\alpha_{l-1} \bm{g}_{l-1}\left(\bm{\theta}_{l-1}\right) \|^2         \\\notag
		 & =\mathbb{E}\| \bar{\theta}_{l\!-\!1}-\theta_*\left\|^2
		+2 \alpha_{l\!-\!1} \mathbb{E}\left\langle\bar{\theta}_{l-1}-\theta_*, \bm{g}_{l-1}\left(\bm{\theta}_{l-1}\right)\right\rangle
		+\alpha_{l\!-\!1}^2 \mathbb{E}\right\| \bm{g}_{l-1}\left(\bm{\theta}_{l-1}\right) \|^2                                \\\label{eq:msv-2}
		 & \le \left(1+\alpha_{l\!-\!1}\right)\EE\left\|\bar{\theta}_{l-1}-\theta_{*}\right\|^2
		+\alpha_{l\!-\!1}\left(1+\alpha_{l-1}\right) \mathbb{E}\left\|\bm{g}_{l-1}\left(\bm{\theta}_{l-1}\right)\right\|^2    \\\label{eq:msv-3}
		 & \le\left(1+\alpha_{l\!-\!1}\right)\left(1+64 \alpha_{l-1}\right)\mathbb{E}\| \bar{\theta}_{l-1}-\theta_* \|^2
		+\alpha_{l\!-\!1}\left(1+\alpha_{l-1}\right) V(l-1-sK),
	\end{align}
	where \cref{eq:msv-2} uses Young's inequality and \cref{eq:msv-3} uses \cref{lem:grad-norm}.
	We require $64\alpha_{sK} \le 1$, which gives $(1+\alpha_{l-1})(1+64\alpha _{l-1}) \le (1+66\alpha _{l-1})$. Recursively applying \cref{eq:msv-3} gives
	\begin{align}\notag
		\mathbb{E}\left\|\bar{\theta}_{l}-\theta_*\right\|^2
		 & \le\left(1+66\alpha_{sK}\right)^{l-sK}\mathbb{E}\| \bar{\theta}_{sK}-\theta_* \|^2
		+\alpha_{sK}\left(1+\alpha_{sK}\right) V(\tau)\sum_{j=0}^{l-1-sK}(1+66\alpha _{sK})^{j},
	\end{align}
	where we use the fact that $V$ is monotonically increasing.
	We further requires that $132(\tau+K)\alpha_{sK}\le 1$. Then, similar to \cref{eq:bi}, we get
	\begin{align}\label{eq:msv-4}
		\mathbb{E}\left\|\bar{\theta}_{l}-\theta_*\right\|^2
		\le 2\mathbb{E}\| \bar{\theta}_{sK}-\theta_* \|^2
		+2\alpha_{sK}\left(1+\alpha_{sK}\right)(\tau + K) V(\tau).
	\end{align}
	Combining \cref{eq:msv-1,eq:msv-4} gives
	\begin{align}\notag
		\EE\|\bar{\theta}_t - \bar{\theta}_{t-\tau}\|^2
		\le & 2(\tau+3K) \alpha^2_{sK}\sum_{l=sK}^{t-1} \left(128
		\mathbb{E}\| \bar{\theta}_{sK}-\theta_* \|^2
		+128\alpha_{sK}\left(1+\alpha_{sK}\right)(\tau+K) V(\tau)
		+ V(\tau) \right)                                         \\\label{eq:msv-5}
		\le & 2(\tau+3K) \alpha^2_{sK}\sum_{l=sK}^{t-1} \left(128
		\mathbb{E}\| \bar{\theta}_{sK}-\theta_* \|^2
		+\left(\frac{128}{132}\cdot \frac{133}{132} + 1\right)V(\tau)
		\right)                                                   \\\notag
		\le & 4(\tau+K)(\tau+3K) \alpha^2_{sK}\left(64
		\mathbb{E}\| \bar{\theta}_{sK}-\theta_* \|^2
		+V(\tau)
		\right),
	\end{align}
	where \cref{eq:msv-5} uses our requirement that $\alpha_{t-\tau}\le (\tau+K) \alpha _{sK}\le 1/132$.
\end{proof}

\section{Proof of Theorem~\ref{thm}}\label{sec:thm-pf}

\begin{reptheorem}{thm}
	If $\left\| \bar{\theta}_l \right\| \le G$ holds for all $l\le t$, then
	\begin{equation*}
		\EE\left\| \bar{\theta}_{t+1}-\theta_{*} \right\|^2
		\le (1 - \alpha _{t}w)\EE\left\| \bar{\theta}_t-\theta_{*} \right\|^2
		+ \alpha  _t C_1\Lambda^2(\epsilon_p,\epsilon_r)
		+ \alpha^2_t \frac{C_2}{N}
		+ \alpha^3_t C_3
		+ \alpha^4_t C_4,
	\end{equation*}
	where $C_1, C_2, C_3, C_4$ are constants defined in \eqref{eq:per-const}.
\end{reptheorem}
\begin{proof}
	We need to pre-process the results from \cref{lem:grad-hetero,lem:drift,lem:mix,lem:stat,lem:v-para-prog} before plugging them back into Lemma \ref{lem:decomp}.
	Throughout this proof, let $s$ and $s'$ be the largest integers such that $sK \le t-\tau$ and $s'K \le t$.
	First, for Lemma \ref{lem:grad-hetero}, by Young's inequality $ab \le \frac{1}{2}\left( \beta a^2 + \frac{1}{\beta}b^2 \right)$, for any positive $\beta$, we have
	\begin{equation}\label{eq:per-1}
		2\EE\left<\bar{\theta}_t-\theta_{*}, \frac{1}{N}\sumn \bar{g}\i\left( \bar{\theta}_t \right) - \bar{g}\left( \bar{\theta}_t \right) \right>
		\le \beta\EE\left\| \bar{\theta}_t - \theta_{*} \right\|^2 + \frac{\Lambda^2(\epsilon_p,\epsilon_r)}{\beta}.
	\end{equation}

	Similarly, for Lemma \ref{lem:drift} and \ref{lem:grad-prog}, we have
	\begin{align}\label{eq:per-2}
		2\EE\left<\bar{\theta}_{t}-\theta_{*},\frac{1}{N}\sumn\left( \bar{g}\i(\theta\it) - \bar{g}\i(\bar{\theta}_{t}) \right) \right>
		\le & \beta\EE\left\| \bar{\theta}_t - \theta_{*} \right\|^2
		+ \frac{1}{\beta}\alpha _{s'K}^2(1+\gamma+\sigma LH)^2C^2_{\mathrm{drift}}, \\\label{eq:per-3}
		\frac{1}{N}\sumn 2\EE\left<\bar{\theta}_{t}-\theta_{*},\bar{g}\ita(\theta\it) - \bar{g}\i(\theta\it) \right>
		\le & \beta\EE\left\| \bar{\theta}_t - \theta_{*} \right\|^2
		+ \frac{1}{\beta}\alpha _{sK}^2C_{\mathrm{prog}}^2L^2\sigma^2\EE h^2\left( \bm{\theta}_t \right),
	\end{align}
	where $h^2(\bm{\theta}_t) = \frac{1}{N}\sumn h^2\left( \theta\it \right)$, and
	\[
		\EE h^2\left( \bm{\theta}_t \right) = 2H^2 + 2(1+\gamma)^2\EE\left[ \Omega_t  \right] \le 2H^2 + 8\alpha _{s'K}^2C^2_{\mathrm{drift}} \le H^2_{\mathrm{drift}},
	\]
	where we define $H_{\mathrm{drift}} \coloneqq \sqrt{2H^2 + 8\alpha_{0}^2C^2_{\mathrm{drift}}}$.

	Then, for Lemma \ref{lem:mix}, we have
	\[
		\begin{aligned}
			    & \frac{1}{N}\sumn\EE\left<\bar{\theta}_t-\theta_{*},{g}\ita\left( \theta\it \right)-\bar{g}\ita\left( \theta\it \right) \right>                                                                                                          \\
			=   & \frac{1}{N}\sumn\EE\left<\bar{\theta}_t-\bar{\theta}\tta,{g}\ita\left( \theta\it \right)-\bar{g}\ita\left( \theta\it \right) \right>
			+ \frac{1}{N}\sumn\EE\left<\bar{\theta}\tta-\theta_{*},{g}\ita\left( \theta\it \right)-\bar{g}\ita\left( \theta\it \right) \right>                                                                                                            \\
			\le & \underbrace{\EE\left[\frac{1}{N}\sumn\EE\left[ \left< \bar{\theta}_t-\bar{\theta}\tta,  {g}\ita\left( \theta\it \right)-\bar{g}\ita\left( \theta\it \right) \right>  \given \mathcal{F}\tta \right]\right]}_{H_1}                       \\
			    & + \underbrace{\EE\left[ \frac{1}{N}\sumn\left\| \bar{\theta}\tta-\theta_{*} \right\|\left\|\EE\left[  {g}\ita\left( \theta\it \right)-\bar{g}\ita\left( \theta\it \right)  \given \mathcal{F}_{t-\tau} \right]\right\|  \right]}_{H_2}.
		\end{aligned}
	\]
	For $H_1$, since both $g\ita$ and $\bar{g}\ita$ are independent of $\theta\it$ conditioned on $\mathcal{F}\tta$, Lemma \ref{lem:v-para-prog} and \ref{lem:mix} give
	\begin{align}\notag
		H_1 = & \EE\left[ \frac{1}{N}\sumn \left<\EE[\bar{\theta}_t - \bar{\theta}\tta \mid \mathcal{F}\tta], \EE\left[ g\ita\left( \theta\it\right) - \bar{g}\ita\left( \theta\it \right)  \given \mathcal{F}\tta \right] \right> \right] \\\notag
		\le   & \EE\left[ \frac{1}{N}\sumn\left\| \EE[\bar{\theta}_t - \bar{\theta}\tta\mid \mathcal{F}\tta ]\right\| \left\| \EE \left[Z\ita\left( \theta\it \right) \given \mathcal{F}\tta\right] \right\|  \right]                      \\\label{eq:per-3.5}
		\le   & \alpha _{sK}C_{\mathrm{prog}} \cdot m \rho^{\tau} \EE h(\bm{\theta}_t)                                                                                                                                                       \\\notag
		\le   & \alpha _{sK}m\rho^{\tau}C_{\mathrm{prog}}H_{\mathrm{drift}}.
	\end{align}
	% {\cor where \eqref{eq:per-3.5} uses the fact that given $\mathcal{F}\tta$, $Z\ita$ is independent of $\theta\it$ and thus \cref{lem:mix} applies here.}
	Similarly, for $H_2$, we have
	\[
		\begin{aligned}
			H_2 = & \EE\left[ \left\| \bar{\theta}\tta - \theta_{*} \right\|\frac{1}{N}\sumn \left\| \EE\left[ {g}\ita\left( \theta\it \right)-\bar{g}\ita\left( \theta\it \right) \given \mathcal{F}\tta \right] \right\|\right] \\
			\le   & \EE\left[m\rho^{\tau}\EE h\left( \bm{\theta}_t \right)\left( \left\| \bar{\theta}\tta - \bar{\theta}_t \right\|+\left\| \bar{\theta}_t - \theta_{*} \right\| \right)  \right]                                 \\
			\le   & m\rho^{\tau}H_{\mathrm{drift}}\left( \EE\left\| \bar{\theta}_t - \theta_{*} \right\| + \alpha _{sK}C_{\mathrm{prog}} \right)                                                                                  \\
			\le   & \frac{1}{2}\left( \beta\EE\left\| \bar{\theta}_t - \theta_{*} \right\|^2
			+ \frac{1}{\beta} m^2\rho^{2\tau}H_{\mathrm{drift}}^2 \right)
			+ \alpha _{sK}m\rho^{\tau}C_{\mathrm{prog}}H_{\mathrm{drift}}.
		\end{aligned}
	\]
	Substituting $H_1$ and $H_2$ with the above bounds gives
	\begin{equation}\label{eq:per-4}
		\frac{1}{N}\sumn 2\EE\left<\bar{\theta}_t-\theta_{*},{g}\ita\left( \theta\it \right)-\bar{g}\ita\left( \theta\it \right) \right>
		\le \beta\EE\left\| \bar{\theta}_t - \theta_{*} \right\|^2
		+ m\rho^{\tau}H_{\mathrm{drift}}\left( \frac{1}{\beta} m\rho^{\tau}H_{\mathrm{drift}} + 4\alpha _{sK}C_{\mathrm{prog}}\right).
	\end{equation}

	For \cref{lem:stat}, the trick we applied in \cref{eq:per-3.5} is no longer valid because $g\it$ and $\theta\it$ are correlated. Notice that $\theta\ita$ is deterministic given $\mathcal{F}\tta$, we first apply the following decomposition:
	$$
		\begin{aligned}
			  & \frac{1}{N}\sumn 2\EE\left<\bar{\theta}_t-\theta_{*},g\it(\theta\it, O\it) - g\i\tta(\theta\it, \too\it) \right>                                                                                                 \\
			= & \underbrace{\frac{1}{N}\sumn 2\EE\left<\bar{\theta}_t-\theta_{*},\left(g\it(\theta\it, O\it) - g\it(\theta\ita, O\it)\right) + \left(g\ita(\theta\ita, \too\it) - g\ita(\theta\it, \too\it)\right)\right>}_{H_3} \\
			  & + \underbrace{\frac{1}{N}\sumn 2\EE\left<\bar{\theta}\tta-\theta_{*},g\it(\theta\ita, O\it) - g\i\tta(\theta\ita, \too\it)  \right>}_{H_4}                                                                       \\
			  & + \underbrace{\frac{1}{N}\sumn 2\EE\left<\bar{\theta}_t-\bar{\theta}\tta,g\it(\theta\ita, O\it) - g\i\tta(\theta\ita, \too\it)  \right>}_{H_5}.
		\end{aligned}
	$$
	By the Lipschitzness of semi-gradient $g\it$ and $g\ita$ and \cref{cor:para-prog}, we have
	$$
		H_3 \le \frac{1}{N}\sumn 2\EE \left[\left\| \bar{\theta}_t-\theta_{*} \right\| \cdot 4\left\| \theta\it - \theta\ita \right\|  \right]
		\le \beta\EE\|\bar{\theta}_t-\theta_{*}\|^2 + \frac{4}{\beta}\alpha_{sK}^2C^2_{\mathrm{prog}}(\tau)
		.$$
	By \cref{lem:stat}, we have
	$$
		\begin{aligned}
			H_4 \le & \frac{2}{N}\sumn\EE\left[ \left\| \bar{\theta}\tta-\theta_{*} \right\|
			\EE\left[ \left\| g\it(\theta\ita, O\it) - g\i\tta(\theta\ita, \too\it) \right\|  \given \mathcal{F}_{t-\tau} \right] \right] \\\notag
			\le     & \beta\EE\left\| \bar{\theta}_t - \theta_{*} \right\|^2
			+ \frac{1}{\beta} \left( \alpha _{sK}C_{\mathrm{back}}\EE h(\bm{\theta}\tta) \right) ^2 + 2 \alpha _{sK}^{2} C_{\mathrm{prog}}C_{\mathrm{back}}\EE h(\bm{\theta}_{t-\tau})
			.
		\end{aligned}
	$$

	Finally, for $H_5$, by Young's inequality, we have
	$$
		\begin{aligned}
			H_5 \le & \frac{\tau+2K}{\alpha _{sK}(\tau+K)(\tau+3K)}\EE\|\bar{\theta}_t-\bar{\theta}\tta\|^2                                                                                                                  \\
			        & + \frac{\alpha _{sK}(\tau+K)(\tau+3K)}{\tau+2K}\EE\left[ \EE\left[ \|\bm{g}_{t}(\bm{\theta}\tta, \bm{O}_t) - \bm{g}\tta(\bm{\theta}_{t-\tau}, \bm{\too}_t)\|^2 \given \mathcal{F}\tta \right] \right].
		\end{aligned}
	$$
	Since $\bm{\theta}\tta$ is deterministic given $\mathcal{F}\tta$, we can apply a similar argument to \cref{eq:grad-norm-4} here, which gives
	$$
		H_5 \le \frac{(\tau+2K)}{\alpha _{sK}(\tau+K)(\tau+3K)}\EE\|\bar{\theta}_t-\bar{\theta}\tta\|^2
		+ \alpha _{sK}(\tau+2K)\left( \frac{4H^2}{N} + \alpha^2_{sK}C^2_{\mathrm{back}}(\tau)H^2 \right).
	$$
	By \cref{cor:msv,lem:v-para-prog}, we get
	$$
		\begin{aligned}
			H_5 \le & \alpha _{sK}(\tau+2K)\left( 256 \EE\left\| \bar{\theta}_{sK}-\theta_{*} \right\|^2 + 4V(\tau)+ \frac{4H^2}{N} + \alpha^2_{sK}C^2_{\mathrm{back}}(\tau)H^2 \right)                                             \\
			\le     & \alpha _{sK}(\tau+2K)\Bigl( 256 (1 + 1/32) \EE\left\| \bar{\theta}_{t}-\theta_{*} \right\|^2 + 256 (1+32) \EE\left\| \bar{\theta}_{sK}-\bar{\theta}_{t} \right\|^2                                     \\
			        & \phantom{\alpha _{sK}(\tau+2K)\Bigl(}+ 4V(\tau)+ \frac{4H^2}{N} + \alpha^2_{sK}C^2_{\mathrm{back}}(\tau)H^2 \Bigr)                                                                                                                         \\
			\le     & \alpha _{sK}(\tau+2K)\left( 264 \EE\left\| \bar{\theta}_{t}-\theta_{*} \right\|^2 + 8448 \alpha^2_{sK}C^2_{\mathrm{prog}}(\tau)+ 4V(\tau)+ \frac{4H^2}{N} + \alpha^2_{sK}C^2_{\mathrm{back}}(\tau)H^2 \right) \\
		\end{aligned}
	$$
	We further require that $132 \alpha_{sK} (\tau+2K)\le \beta$.
	Then, plugging $H_3,H_4,H_5$, and $V(\tau)$ back gives
	\begin{align}\notag
		    & \frac{1}{N}\sumn 2\EE\left<\bar{\theta}_t-\theta_{*},g\it(\theta\it, O\it) - g\i\tta(\theta\it, \too\it) \right>                                                                             \\\notag
		\le & 4\beta\EE\|\bar{\theta}_t-\theta_{*}\|^2
		+ \frac{1}{\beta}\alpha_{sK}^2\left( 4C^2_{\mathrm{prog}} + 2\beta C_{\mathrm{prog}}C_{\mathrm{back}}H_{\mathrm{drift}} + C^2_{\mathrm{back}}H^2_{\mathrm{drift}}\right) + \frac{2\beta\Lambda^2}{w^2}                                                           \\\label{eq:per-5}
		    & + \alpha _{sK}(\tau+2K)\left(8448 \alpha^2_{sK}C^2_{\mathrm{prog}} + 4\alpha^2_{sK}C_{\mathrm{var}} + 16m^2\rho^{2\tau}H^2 + \alpha^2_{sK}C^2_{\mathrm{back}}H^2 + \frac{132H^2}{N} \right).
	\end{align}

	% \begin{align}\notag
	% 	    & \frac{1}{N}\sumn 2\EE\left<\bar{\theta}_t-\theta_{*},g\it(\theta\it, O\it) - g\i\tta(\theta\it, \too\it) \right>           \\\notag
	% 	=   & \frac{1}{N} \sumn 2\EE\left<\bar{\theta}_t-\bar{\theta}\tta,g\it(\theta\it, O\it) - g\i\tta(\theta\it, \too\it) \right>    \\\notag
	% 	    & + \frac{1}{N}\sumn 2\EE\left<\bar{\theta}\tta-\theta_{*},g\it(\theta\it, O\it) - g\i\tta(\theta\it, \too\it) \right>       \\\notag
	% 	\le & \frac{2}{N}\sumn \EE\left[ \left\| \bar{\theta}_t-\bar{\theta}\tta \right\|
	% 	\left\| \EE \left[g\it(\theta\it, O\it) - g\i\tta(\theta\it, \too\it)\given \mathcal{F}\tta, \bm{\theta}_t\right] \right\| \right] \\\notag
	% 	    & + \frac{2}{N}\sumn\EE\left[ \left\| \bar{\theta}\tta-\theta_{*} \right\|
	% 	\EE\left[ \left\| g\it(\theta\it, O\it) - g\i\tta(\theta\it, \too\it) \right\|  \given \mathcal{F}_{t-\tau} \right] \right]        \\\notag
	% 	\le & 2\alpha _{sK}C_{\mathrm{prog}}\cdot \alpha _{sK}C_{\mathrm{back}} \EE h(\bm{\theta}_t)
	% 	+ \beta\EE\left\| \bar{\theta}_t - \theta_{*} \right\|^2
	% 	+ \frac{1}{\beta} \left( \alpha _{sK}C_{\mathrm{back}}\EE h(\bm{\theta}_t) \right) ^2
	% 	+ 2\alpha _{sK}C_{\mathrm{prog}}\cdot \alpha _{sK}C_{\mathrm{back}} \EE h(\bm{\theta}_t)                                         \\\label{eq:per-5}
	% 	\le & \beta\EE\left\| \bar{\theta}_t - \theta_{*} \right\|^2
	% 	+ \alpha^2_{sK}C_{\mathrm{back}}H_{\mathrm{drift}}\left( \frac{1}{\beta}C_{\mathrm{back}}H_{\mathrm{drift}} + 4C_{\mathrm{prog}} \right).
	% \end{align}

	Putting \cref{eq:per-1,eq:per-2,eq:per-3,eq:per-4,eq:per-5} and \cref{lem:des-dir,lem:grad-norm} back into \cref{lem:decomp}, we get
	\[
		\begin{aligned}
			    & \EE\left\| \bar{\theta}_{t+1} - \theta_{*} \right\|^2
			\le \EE\left\| \breve{\theta}_{t+1} - \theta_{*} \right\|^2                                                                                                                                                                  \\
			\le & (1 - 2\alpha _{t}w)\left\| \bar{\theta}_t  -\theta_{*} \right\|^2
			+ 8\alpha _{t}\beta\EE\left\| \bar{\theta}_t - \theta_{*} \right\|^2                                                                                                                                                         \\
			    & + \alpha _{t}\left(
			\frac{\Lambda^2(\epsilon_p,\epsilon_r)}{\beta}+\frac{2\beta\Lambda^2(\epsilon_{p},\epsilon_{r})}{w^2}
			+ \frac{1}{\beta}\alpha^2_{s'K}(1+\gamma+\sigma LH)^2C^2_{\mathrm{drift}}
			+ \frac{1}{\beta}\alpha^2_{sK}C_{\mathrm{prog}}^2L^2H^2_{\mathrm{drift}}\sigma^2
			\right.                                                                                                                                                                                                                      \\
			    & +
			m\rho^{\tau}H_{\mathrm{drift}}\left( \frac{1}{\beta} m\rho^{\tau}H_{\mathrm{drift}} + 4\alpha _{sK}C_{\mathrm{prog}} \right)
			+ \frac{1}{\beta}\alpha_{sK}^2\left( 4C^2_{\mathrm{prog}} + 2 \beta C_{\mathrm{prog}}C_{\mathrm{back}}H_{\mathrm{drift}}+ C^2_{\mathrm{back}}H^2_{\mathrm{drift}}\right)
			                                                                                                                                                                                                 \\
			    & + \alpha _{sK}(\tau+2K)\left(8448\alpha^2_{sK}C^2_{\mathrm{prog}}
			\left. + 4\alpha^2_{sK}C_{\mathrm{var}} + 16m^2\rho^{2\tau}H^2 + \alpha^2_{sK}C^2_{\mathrm{back}}H^2 + \frac{132H^2}{N} \right)
			\right)                                                                                                                                                                                                                      \\
			    & + \alpha _{t}^2\left(64\left( \EE\left\| \bar{\theta}_t - \theta_{*} \right\|^2 + \frac{\Lambda^2(\epsilon_p,\epsilon_r)}{w^2}\right) + \alpha _{sK}^2C_{\mathrm{var}} + 4m^2\rho^{2\tau}H^2 + \frac{32H^2}{N}\right).
		\end{aligned}
	\]

	Note that $\tau$ is a virtual time range that we backtrack, and we have not determined it yet. Now we require it to be large enough such that $m\rho^{\tau} \le \alpha _{t}$. We also do not want $\tau$ to be too large. Thus, we fix
	\begin{equation}\label{eq:tau}
		\tau = \left\lceil (\log \alpha _{t} - \log m) / \log \rho \right\rceil \asymp \log \alpha^{-1}_{t}.
	\end{equation}
	We also require that the decay rate of $\alpha_t$ is non-increasing and $\sum_{t=0}^{\infty}\alpha _{t} = +\infty$. Then, there exists $T_1 > 0$ such that for any $t \ge T_1$, it holds that
	\[
		sK \ge t - \tau - K = t - \left\lceil  \frac{\log\alpha _{t} - \log m}{\log \rho} \right\rceil - K \ge \frac{t}{2}.
	\]
	The requirement on the step-size also gives $\limsup_{t \to \infty} \alpha_{t/2} / \alpha_{t} < +\infty$.
	Then, there exists $C'_{\alpha}, C_\alpha > 0$ such that for any $t \ge 0$, we have
	\[
		\frac{\alpha_{sK}}{\alpha_{t}}
		\le C'_\alpha\cdot \limsup_{t \to \infty}\frac{\alpha _{t /2}}{\alpha_{t}}
		= C_\alpha
		.\]
	Thus, after some rearrangement, we get
	\[
		\begin{aligned}
			    & \EE\left\| \bar{\theta}_{t+1} -\theta_{*} \right\|^2                                                           \\
			\le & (1 - 2\alpha_{t}w + 8\alpha_{t}\beta + 64\alpha^2_{t}) \EE\left\| \bar{\theta}_{t} -\theta_{*} \right\|^2
			+ 4\alpha^2_{t}(33C_{\alpha}(\tau+2K) + 8)\frac{H^2}{N}                                                              \\
			    & + \alpha^3_{t}C^2_{\alpha}\left(
			\frac{1}{\beta}\left(
			(1+\gamma+\sigma LH)^2C^2_{\mathrm{drift}}
			+ C_{\mathrm{prog}}^2L^2H^2_{\mathrm{drift}}\sigma^2
			+ H^2_{\mathrm{drift}}
			+ 4C^2_{\mathrm{prog}}
			+ 2 \beta C_{\mathrm{prog}}C_{\mathrm{back}}H_{\mathrm{drift}}
			+ C_{\mathrm{back}}^2H^2_{\mathrm{drift}}
			\right)\right.                                                                                                       \\
			    & +4C_{\mathrm{prog}}H_{\mathrm{drift}}\Bigr)
			+ \alpha^4_{t}C^3_{\alpha}(
			(\tau+2K)(8448C^2_{\mathrm{prog}} + 4C_{\mathrm{var}} + 16H^2 + C^2_{\mathrm{back}}H^2)
			+ C_{\mathrm{var}} + 4H^2
			)                                                                                                                    \\
			    & + \alpha _{t}\left( \frac{1}{\beta} + \frac{2\beta+64\alpha _{t}}{w^2}\right)\Lambda^2(\epsilon_p,\epsilon_r).
		\end{aligned}
	\]
	Now we let $\beta$ and $\alpha_0$ small enough such that
	\[
		8\beta + 64\alpha_0 \le w.
	\]
	Then we get the final form
	\[
		\EE\left\| \bar{\theta}_{t+1}-\theta_{*} \right\|^2
		\le (1 - \alpha _{t}w)\EE\left\| \bar{\theta}_t-\theta_{*} \right\|^2
		+ \alpha  _t C_1\Lambda^2(\epsilon_p,\epsilon_r)
		+ \alpha^2_t \frac{C_2}{N}
		+ \alpha^3_t C_3
		+ \alpha^4_t C_4,
	\]
	where
	\begin{equation}\label{eq:per-const}
		\begin{aligned}
			C_1 = & \beta^{-1} + (2\beta+64\alpha_0)w^{-2},                                                                                           \\
			C_2 = & 4(33C_{\alpha}(\tau+2K)+8)H^2,                                                                                                    \\
			C_3 = & C^2_{\alpha}\left(\frac{1}{\beta}\left( (1\!+\!\gamma\!
			+\!\sigma LH)^2C^2_{\mathrm{drift}}\! + C_{\mathrm{prog}}^2L^2H^2_{\mathrm{drift}}\sigma^2\!
			+ H^2_{\mathrm{drift}} + 5C^2_{\mathrm{prog}}
	+ 2C_{\mathrm{back}}^2H^2_{\mathrm{drift}}\right)\right.                                                                                   \\
			      & + 4C_{\mathrm{prog}}H_{\mathrm{drift}}\Bigr)                                                                                      \\
			C_4 = & C_{\alpha}^3((\tau+2K)(8448 C^2_{\mathrm{prog}} + 4C_{\mathrm{var}} + 16H^2 + C^2_{\mathrm{back}}H^2) + C_{\mathrm{var}} + 4H^2).
		\end{aligned}
	\end{equation}
\end{proof}

\section{Proof of Corrolaries~\ref{cor:err-con} and \ref{cor:err-dec}} \label{sec:cor-pf}

In this section, we provide the proofs of \cref{cor:err-con,cor:err-dec}. Combining with the constant dependencies discussed in \cref{sec:const-dep}, we get the final results presented in \cref{sec:anlys}.

\begin{repcorollary}{cor:err-con}
  With a constant step-size
  % $\alpha_{t}= \alpha_0 \le \min \{ 1 /(8K), w /64 \}$,
  $\alpha_t\equiv \alpha_0\le w/(2120(2K+8+\ln (m/(\rho w))))$,
  for any $T\in\mathbb{N}$, we have
  \[
    \EE\left\| \bar{\theta}_T - \theta_{*}\i \right\|^2 \le 4e^{-\alpha_{0}wT}\left\| \theta_0-\theta_{*}\i \right\|^2 + B,
  \]
  where $B$ is the squared convergence region radius defined by
  $$
    B \coloneqq \frac{1}{w}\left( \left(C_1 + \frac{6}{w}\right)\Lambda^2(\epsilon_p,\epsilon_r) + \alpha_0\frac{C_2}{N} + \alpha_{0}^2C_3 + \alpha_{0}^{3}C_4 \right).
  $$
  % That is, the parameter will converge exponentially to a region around the optimal parameter in expectation.
\end{repcorollary}
\begin{proof}
  Let $\theta_{*}$ be the central optimal parameter.
  By \cref{thm}, for any $T\in\mathbb{N}$, we have
  \[
    \EE\|\bar{\theta}_{T} - \theta_{*}\|^2 \le (1 - \alpha_0 w)^{T}\EE\|\theta_0 - \theta_{*}\|^2 + \alpha_{0}w\left(B - \frac{6\Lambda^2}{w^2}\right) \sum_{t=0}^{T-1}(1 - \alpha_0 w)^{t}
    \le e^{-\alpha_0 wT}\|\theta_0 - \theta_{*}\|^2 + B - \frac{6\Lambda^2}{w^2},
  \]
  where the last inequality uses the fact that $(1 - \alpha_0 w) \le e^{-\alpha_0 w}$ and $\sum_{t=0}^{\infty}(1 - \alpha_0 w)^{t} = (\alpha_0 w)^{-1}$.
  Then by \cref{thm:fix-drift}, we get
  $$
  \EE\left\|\bar{\theta} - \theta_*\i\right\|^2 \le 2\EE\|\bar{\theta}-\theta_*\|^2 + 2 \frac{\Lambda^2}{w^2}  \le 4e^{-\alpha_0 wT}\|\theta_0-\theta_*\i\|^2 + B - \frac{6\Lambda^2}{w^2} + \frac{6\Lambda^2}{w^2}.
  $$
\end{proof}

\begin{repcorollary}{cor:err-dec}
  With a linearly decaying step-size $\alpha _{t} = 4/(w(1+t+a)),$ where $a>0$ is to guarantee that $\alpha_0\le \min \{ 1 /(8K), w /64\}$, there exists a convex combination $\widetilde{\theta}_{T}$ of $\{ \bar{\theta}_t \}_{t=0}^{T}$ such that
  \[
    \EE\bigl\| \widetilde{\theta}_T - \theta_{*}\i\bigr\|^2 \le \frac{1}{w}O\left(\frac{C_4}{w^3T^2} + \frac{C_3 \log T}{w^2T^2} + \frac{C_2}{wNT} + C_1\Lambda^2(\epsilon_{p},\epsilon_{r})\right).
  \]
  % where the asymptotic notation suppresses the logarithmic factors and holds as $T$ goes to infinity and $\epsilon_p,\epsilon_r$ go to zero.
\end{repcorollary}
\begin{proof}
  Let $c_t = a + t$ and $C = \sum_{t=0}^{T}c_t \ge (T+1)^2/2$. We define
  \[
    \widetilde{\theta}_{T} = \frac{1}{C}\sum_{t=0}^{T}c_t \bar{\theta}_t,
  \]
  which is a convex combination of $\left\{ \bar{\theta}_t \right\}_{t=0}^{T}$. Then, by Jensen's inequality, we have
  \begin{equation}\label{eq:cor-1}
    \EE\left\| \widetilde{\theta}_{T} - \theta_{*} \right\|^2
    \le \frac{1}{C} \sum_{t=0}^{T} c_{t} \EE \left\| \bar{\theta}_t - \theta_{*} \right\|^2.
  \end{equation}
  Let $\theta_{*}$ be the central optimal parameter.
  By \cref{thm}, we have
  \[
    \frac{1}{2}\EE\left\| \bar{\theta}_t - \theta_{*} \right\|^2
    \le \left(\frac{1}{\alpha_t w} - \frac{1}{2}\right)\EE\left\| \bar{\theta}_t - \theta_{*} \right\|^2
    - \frac{1}{\alpha_t w}\EE\left\| \bar{\theta}_{t+1} - \theta_{*} \right\|^2
    + B(\alpha _{t}),
  \]
  where $B(\alpha) = (C_1 \Lambda^2 + \alpha C_2/N + \alpha^2C_3 + \alpha^3C_4) / w$.
  Recall our choice of the step-size $\alpha_t = 4 /(w(a+t+1))$; then we have $1 /(\alpha_t w) = (a+t+1)/4$.
  Plugging this back into \eqref{eq:cor-1} gives
  \[
    \begin{aligned}
      \EE\left\| \widetilde{\theta}_{T}\!-\!\theta_{*} \right\|^2
      \le & \frac{1}{C} \sum_{t=0}^{T} c_t\left( \frac{a\!+\!t\!-\!1}{2}\EE\left\| \bar{\theta}_t\! -\! \theta_{*} \right\|^2
      - \frac{\!a\!+\!t\!+\!1}{2}\EE\left\| \bar{\theta}_{t+1}\!-\!\theta_{*} \right\|^2
      + 2B(\alpha_t)  \right)                                                                                                                                  \\
      =   & \frac{1}{2C}\sum_{t=0}^{T}\left( (a+t-1)(a+t)\EE\left\| \bar{\theta}_{t}-\theta_{*} \right\|^2
      - (a+t)(a+t+1)\EE\left\| \bar{\theta}_{t+1}-\theta_{*} \right\|^2  \right)                                                                               \\
          & + \frac{2C_1\Lambda^2}{w}
      + \frac{8C_2}{CNw^2}\sum_{t=0}^{T}\frac{a+t}{a+t+1}
      + \frac{32C_3}{Cw^3}\sum_{t=0}^{T}\frac{a+t}{(a+t+1)^2}
      + \frac{128C_4}{Cw^4}\sum_{t=0}^{T}\frac{a+t}{(a+t+1)^3}
      \\
      \le & \frac{1}{2C}\left( a(a-1)\left\| \bar{\theta}_{0} - \theta_{*}\right\|^2 - (a+T)(a+T+1)\EE\left\| \bar{\theta}_{T+1}-\theta_{*} \right\|^2 \right) \\
          & + \frac{2C_1\Lambda^2}{w}
      + \frac{8C_2(T+1)}{CNw^2}
      + \frac{32C_3}{Cw^3}\sum_{t=0}^{T}\frac{1}{t+1}
      + \frac{128C_4}{Cw^4}\sum_{t=0}^{T}\frac{1}{(t+1)^2}                                                                                                     \\
      \le & \frac{a^2\left\| \theta_{0} - \theta_{*} \right\|^2}{T^2}
      + \frac{2C_1\Lambda^2}{w} + \frac{8C_2}{w^2NT} + \frac{32C_3}{w^3T^2}O(\log(T)) + \frac{256C_4}{w^4 T^2}                                                 \\[1ex]
      =   & O\left(\frac{a^2}{T^2} + \frac{C_4}{w^4T^2} + \frac{C_3 \log T}{w^3T^2} + \frac{C_2}{w^2NT} + \frac{C_1\Lambda^2}{w}\right).
    \end{aligned}
  \]
  Then by \cref{thm:fix-drift} and the fact that $1/w\lesssim C_1$ (see \cref{sec:const-dep}) and $a \lesssim K / w^2$, we get
  $$
  \EE\left\| \widetilde{\theta}_{T}\!-\!\theta_{*}\i \right\|^2
  \le 2\EE\left\| \widetilde{\theta}_{T}\!-\!\theta_{*} \right\|^2 + 2\frac{\Lambda^2}{w^2}
  =  O\left(\frac{K^2 + C_4}{w^4T^2} + \frac{C_3 \log T}{w^3T^2} + \frac{C_2}{w^2NT} + \frac{C_1\Lambda^2}{w}\right).
  $$

\end{proof}

\section{Constant Dependencies} \label{sec:const-dep}

In this section, we establish explicit dependencies between the constants. We begin by introducing problem constants that are independent of other parameters: the reward cap $R > 0$, discount factor $\gamma \in (0,1)$, projection radius $\bar{G} > 0$,\footnotemark{} local update period $K$, and kernel-related constants, $m \ge 1,\rho \in (0,1)$, and $\lambda\coloneqq \min_{i\in[\bar{N}]}\lambda\i \in (0,1]$.
Throughout this paper, we use asymptotic notation as $R,\bar{G},K,m \to \infty$ and $\gamma,\rho \to 1$. We also use the nonasymptotic notation $a \lesssim b$ and $b \gtrsim a$ to indicate that there exists $C \ge 0$ such that $a \le Cb$, and $a\asymp b$ to indicate that both $a\lesssim b$ and $b\gtrsim a$ hold.

\footnotetext{One can choose $\bar{G} = R /w$ as suggested in \cite{zou2019Finitesampleanalysis}. Here, we make it a pre-defined algorithm constant.}

We first give the dependencies of $\sigma'$ defined in \eqref{eq:sigma}. By its definition, we have $\sigma' \ge 0$, and
$$\sigma' \le \frac{\log m}{-\log \rho} + \frac{1}{1-\rho} \le \frac{\log m+1}{1-\rho} = O\left( \frac{\log m}{1-\rho} \right),$$
where the asymptotic notation holds as $\rho\to 1$ and $m\to\infty$. 
We also get $\sigma = \sigma' + 2 = O(\log m / (1-\rho))$.
We will now use $\sigma$ as a base constant.

$w$ is an important MDP constant and plays a critical role in the convergence rate. By its definition~\eqref{eq:w}, we get $w \le 1 /2$ and
$$
w = \min_{i\in[\bar{N}]}w_i \ge \frac{1-\gamma}{2}\min_{i\in[\bar{N}]}\lambda\i = \frac{1-\gamma}{2}\lambda,
$$
which gives
$$
w^{-1} = O((1-\gamma)^{-1}).
$$

We then consider $G$ and $H$. By \cref{cor:G}, we get
$$
G = \frac{2(2\bar{G} + R)}{1-16\alpha_0 ^2K^2\gamma} = O\left( \frac{\bar{G} + R}{1-\gamma} \right).
$$
When $\gamma$ is near 1, the above bound is undesirable. Thus, when $\gamma$ is large, we can further require $4\alpha_0K < \sqrt{0.5}$, which gives $G \le 4(2\bar{G} + R)$. Without loss of generality, we have
$$
G \asymp \bar{G} + R
$$
And by the definition of $H$, we get
$$
H = R + (1+\gamma)G \asymp \bar{G} + R.
$$
We now use $H$ as a base constant and replace $\bar{G} + R$ with $H$ for simplicity.
$H$ can be viewed as the scale of the problem. If we choose $\bar{G}$ according to \citep{zou2019Finitesampleanalysis}, then $H = O(R /(1-\gamma))$.

By \eqref{eq:L}, we get the dependencies of the policy improvement operator's Lipschitz constant $L$:
\[
	L \le \frac{w}{\sigma H}.
\]

We now address the constants in \cref{sec:use-lem}.
By \cref{lem:drift}, we directly have
\[
	C_{\mathrm{drift}} = O\left(KH\right).
\]
We now consider $\alpha_0$. There are two requirements on $\alpha_0$ throughout the proof: $4K\alpha_0 < 1$ in \cref{lem:drift,lem:v-para-prog}, and $64\alpha_0 \le w$ in \cref{sec:thm-pf}. Combining these conditions gives
\[
	\alpha_0 \le \min\left\{ \frac{1}{4K}, \frac{w}{64}\right\} \lesssim \min \left\{K^{-1}, w\right\}.
\]
Therefore, $C_{\mathrm{prog}}$ in \cref{lem:v-para-prog} has the following dependencies:
\[
	C_{\mathrm{prog}} = O((\tau + K)(H + K^{-1}\cdot KH)) = O((\tau + K) H).
\]
And $C_{\mathrm{back}}$ in \cref{lem:stat} has the following dependencies:
\[
	C_{\mathrm{back}} = O(\tau^2 LH) = O(\tau^2w).
\]
Then, $C_{\mathrm{var}}$ in \cref{lem:grad-norm} is controlled by
\[
	C_{\mathrm{var}} = O(C_{\mathrm{drift}}^2 + w^2 C^2_{\mathrm{prog}} + H^2C^2_{\mathrm{back}})
	= O(H^2(K^2 + w^2 \tau^4)).
\]

Next, we give the dependencies of constants in \cref{sec:thm-pf}. By definition, we have
$$
H_{\mathrm{drift}} = O(H + \alpha_0 C_{\mathrm{drift}}) = O(H).
$$
By the requirement of $\beta$, we have
$$
\beta \asymp w.
$$
And we have $C_{\alpha} = O(1)$.
Therefore, we get
\[
	\begin{aligned}
		C_1 & = O(w^{-1}) = O((1-\gamma)^{-1}),\\
		C_2 & = O(H^2(\tau+K)),\\
		C_3 & = O\left( H^2(w^{-1}(\tau^2 + K^2) + w\tau^{4}) \right), \\
		C_4 &= O(H^2(\tau+K)(\tau^2 + K^2 + w^2 \tau^4)).
	\end{aligned}
\]

Finally, we give the dependencies of constants in \cref{cor:err-con,cor:err-dec}. In \cref{cor:err-con}, we choose a constant step-size $\alpha _{t} = \alpha_0$.
There are two requirements on $\alpha_t$ throughout the proof: $132(\tau+K) \alpha_{sK} \le 1$ in \cref{cor:msv} and $132\alpha _{sK}(\tau+2K)\le \beta$ in \cref{sec:thm-pf}.
% To make the dependencies cleaner, we further require that $\alpha_0 \lesssim e^{-\max\{K, 1/w\}}$, which gives $\tau \gtrsim \max\{K,w^{-1}\}$. % This requirement on $\alpha_0$ is the strongest.
A concrete condition satisfying these requirements is $
\alpha_0 \le w/(2120\left(2K + \ln  (2120 m)/(\rho w)\right)).$
% \alpha_0 \le \min \left\{\frac{w}{2120\left(2K + \ln  \frac{2120 m}{\rho w}\right)}, \exp \left( -\max \left\{ K, \frac{1}{w} \right\}\right)\right\}.
Furthermore, if we choose a small enough initial step size such that $\alpha _{0}^{-1} \asymp \tau \gtrsim \max \left\{ K, w^{-1}  \right\}$,
% These additional requirements and \eqref{eq:tau} gives
% $$
% \alpha_0 \lesssim \frac{w}{\max\{ K, \log w^{-1} \}}.
% $$
% To have a fast convergence, we choose a step-size making the above bound tight. Then,
% the virtual backtracking period satisfies
% $$
% \tau \asymp \log \alpha_{0}^{-1} \asymp \log (w^{-1} \cdot \max\{ K,\log w^{-1} \}).
% $$
then $C_2,C_3$, and $C_4$ becomes
\begin{equation}\label{eq:C}
C_2 = O(H^2\tau) = \widetilde{O}(H^2),\ 
C_3 = O(H^2w\tau^{4}) = \widetilde{O}(H^2w),\ 
C_4 = O(H^2w^{2}\tau^{5}) = \widetilde{O}(H^2w^2),
\end{equation}
where $\widetilde{O}$ omits the logarithmic dependencies on $\tau$.
Then the convergence region radius in \cref{cor:err-con} becomes
$$
\begin{aligned}
B = O\left(\alpha_{0}^2 H^2 \tau^{4} + \frac{\alpha_{0}H^2\tau}{N(1-\gamma)} + \frac{\Lambda^2}{(1-\gamma)^2}\right)
= \widetilde{O}\left(\alpha_{0}^2 H^2 + \frac{\alpha_{0}H^2}{N(1-\gamma)} + \frac{\Lambda^2}{(1-\gamma)^2}\right)
\end{aligned}
$$

With the linearly decaying step-size in \cref{cor:err-dec}, \eqref{eq:tau} gives
$$
\tau \asymp\log T
$$
as the total number of iterations $T \to \infty$. And the requirements on $\alpha _{t}$ in previous discussion %\cref{cor:msv,sec:thm-pf} 
automatically hold for large enough $t$.
Omitting the logarithmic dependencies on $T$, $C_2, C_3$, and $C_4$ in this case are the same as \cref{eq:C}.
Therefore, the finite-time error bound in \cref{cor:err-dec} becomes
$$
\begin{aligned}
\EE\left\| \widetilde{\theta}_{T} - \theta_{*} \right\|^2 
= \frac{H^2}{(1-\gamma)^2} \cdot O\!\left( \frac{\tau^{5}}{T^2} + \frac{\tau}{NT} + \frac{\Lambda^2(\epsilon_{p},\epsilon_{r})}{H^2}\right) = \frac{H^2}{(1-\gamma)^2} \cdot \widetilde{O}\!\left(\frac{1}{NT} + \frac{\Lambda^2(\epsilon_{p},\epsilon_{r})}{H^2} \right).
\end{aligned}
$$

\section{Tabular \fedsarsa} \label{sec:tab}

In this section, we reduce our algorithm and analysis to the tabular setting.
Recall that $S$ and $A$ are the measures of the state space $\S$ and action space $\A$, respectively. For the tabular setting, $S$ and $A$ are the numbers of states and actions.
Then, we choose the feature map to be an indicator vector function, i.e., 
$$
	\phi: \S \times \A \to \R^{SA},\quad
	[\phi(s,a)]_{(s',a')} \mapsto \mathbbm{1}\{ (s',a') = (s,a) \},
$$
where we treat $\phi(s,a)$ as a vector and use a two-dimensional index such that $[\phi(s,a)]_{(s',a')}$ is the $(s',a')$-th element of $\phi(s,a)$; $\mathbbm{1}$ is the indicator function.
Using this feature map, the parameter $\theta$ is indeed the estimated value function table:
$$
Q_{\theta}(s,a) = \phi^T(s,a)\theta = [\theta]_{(s,a)}
$$
Therefore, the local update rule in Algorithm~\ref{alg} reduces to the tabular SARSA update rule \cref{eq:sarsa}.

We now show a natural bound $G$ for $\|\theta\|_{2}$ without an explicit projection. First, the true value function \eqref{eq:q} is bounded by
$$
|q_{\pi}(s,a)| \le \sum_{t=0}^{\infty}\gamma^{t}R = \frac{R}{1-\gamma} \eqqcolon G_{\infty}.
$$
Suppose current estimated value function satisfies that $|Q_t(s,a)| \le G_{\infty}$ for any state-action pair, then we have
$$
\begin{aligned}
|Q_{t+1}(s,a)| =& |Q_{t}(s,a) + \alpha (r(s,a) + \gamma Q_{t}(s',a') - Q_{s,a})|\\
=& |(1 - \alpha)Q_{t}(s,a) + \alpha\gamma Q_{t}(s',a') + \alpha r(s,a)|\\
\le& (1-\alpha)G_{\infty} + \alpha\gamma G_{\infty} + \alpha R\\
=& (1-\alpha+\alpha\gamma) \frac{R}{1-\gamma} + \alpha R\\
=& \frac{R}{1-\gamma} = G_{\infty}.
\end{aligned}
$$
Therefore, if the bound holds for the initial estimated value function, it holds for all sequential, local or central, estimated value functions.
However, $G_{\infty}$ is a upper bound for $\|\theta\|_{\infty}$. For 2-norm, we have
$$
\|\theta\|_{2} \le \sqrt{SA} \|\theta\|_{\infty} \le \frac{\sqrt{SA}R}{1-\gamma} \eqqcolon G,
$$
which further gives
$$
H = O\left( \frac{\sqrt{SA}R}{1-\gamma} \right).
$$
Also, for tabular \fedsarsa, \cref{rmk:lam} tells us that
$$
w^{-1} = O\left(\frac{1}{\lambda(1-\gamma)}\right),
$$
 $\lambda$ is the probability of visiting the least probable state-action pair under the steady distribution of the optimal policy across all agents.
Then, \cref{cor:err-dec} can be translated into the following corollary.

\begin{corollary}[Finite-time error bound for tabular \fedsarsa with decaying step-size]%\label{cor:tab}
	With a linearly decaying step-size $\alpha _{t} = 4/(w(1+t+a)),$ where $a>0$ is to guarantee that $\alpha_0\le\min\{ 1 /(8K), w /64 \}$, there exists a convex combination $\widetilde{\theta}_{T}$ of $\{ \bar{\theta}_t \}_{t=0}^{T}$ such that
	\[
	\EE\bigl\| \widetilde{\theta}_T - \theta\i_{*}\bigr\|_{2}^2 \le
	\frac{1}{\lambda^2(1-\gamma)^2} \cdot \widetilde{O}\left(\frac{SAR^2}{\lambda^2(1-\gamma)^4 T^2}+\frac{SAR^2}{(1-\gamma)^2N T}+\Lambda^2\left(\epsilon_p, \epsilon_r\right)\right).
	\]
	where the asymptotic notation suppresses the logarithmic factors.
	Since $\|\theta\|_{\infty} \le \|\theta\|_{2}$, we also get the finite-time error bound under the infinity norm.
\end{corollary}

\end{document}